\documentclass[accepted]{uai2025} % after acceptance, for a revised version; 
\usepackage[american]{babel}
\usepackage{natbib} % has a nice set of citation styles and commands
\usepackage{mathtools} % amsmath with fixes and additions
\usepackage{booktabs} % commands to create good-looking tables
\usepackage{tikz} % nice language for creating drawings and diagrams

\usepackage{amsmath}
\usepackage{amsfonts}
\usepackage{amssymb}
\usepackage{amsthm}
\usepackage[capitalise]{cleveref}
\usepackage[makeroom]{cancel}
\usepackage[thinc]{esdiff}
\usepackage{subcaption}
\usepackage{caption}
\usepackage{multirow}
\usepackage{adjustbox}
\usepackage{tabularx}
\usepackage{algorithm}
\usepackage{algpseudocode}

\newtheorem{proposition}{Proposition}[section]

\newcommand{\x}{\mathbf{x}}

\renewcommand{\u}{u_{\theta}}
\newcommand{\udet}{u_{\theta_{\text{det}}}}
\newcommand{\udetopt}{u_{\theta_{\text{det}}^*}}
\newcommand{\mse}[1]{\lVert#1\rVert^2_2}
\newcommand{\slike}{\sigma_{\text{Like}}}

\newcolumntype{Y}{>{\centering\arraybackslash}X}

\title{Improved Uncertainty Quantification in Physics-Informed Neural Networks Using Error Bounds and Solution Bundles}

\author[1]{\href{mailto:ptflores1@uc.cl?Subject=Improved Uncertainty Quantification in Physics-Informed Neural Networks Using Error Bounds and Solution Bundles}{Pablo Flores}{}}
\author[2]{Olga Graf}
\author[3]{Pavlos Protopapas}
\author[1]{Karim Pichara}
\affil[1]{%
    Departamento de Ciencia de la Computación, Pontificia Universidad Católica de Chile, Santiago, Chile
}
\affil[2]{%
    Department of Computer Science, University of Tübingen, Germany
}
\affil[3]{%
    John A. Paulson School of Engineering and Applied Sciences, Harvard University, Cambridge, Massachusetts 02138, USA
  }
  
\begin{document}
\maketitle

\begin{abstract}
  Physics-Informed Neural Networks (PINNs) have been widely used to obtain solutions to various physical phenomena modeled as Differential Equations. As PINNs are not naturally equipped with mechanisms for Uncertainty Quantification, some work has been done to quantify the different uncertainties that arise when dealing with PINNs. In this paper, we use a two-step procedure to train Bayesian Neural Networks that provide uncertainties over the solutions to differential equation systems provided by PINNs. We use available error bounds over PINNs to formulate a heteroscedastic variance that improves the uncertainty estimation. Furthermore, we solve forward problems and utilize the obtained uncertainties when doing parameter estimation in inverse problems in cosmology.
\end{abstract}

\section{Introduction}

\textit{Physics-Informed Neural Networks} (PINNs), first proposed by \cite{lagaris_artificial_1997}, solve differential equations (DEs) by embedding the physics of the problem into the network, eliminating the need for extra data. PINNs offer advantages over traditional solvers: they are continuous, differentiable, and parallelizable, allowing them to bypass the need for previous time steps. However, challenges remain in computing solution errors or error bounds, with ongoing research \citep{liu_evaluating_2022, liu_residual-based_2023, de_ryck_generic_2022}. Specific equations like Navier-Stokes and Elasticity are also under study \citep{de_ryck_error_2022, de_ryck_error_2022-1, guo_energy-based_2022}.

Since their introduction, PINNs have rapidly gained attention, being applied to heat transfer \citep{cai_physics-informed_2021}, wave equations \citep{rasht-behesht_physics-informed_2022}, and fluid mechanics \citep{jin_nsfnets_2021, cai_physics-informed_2021-1, mao_physics-informed_2020}. Efforts to address failure modes and optimization issues are ongoing \citep{nabian_efficient_2021, steger_how_2022, krishnapriyan_characterizing_2021, daw_mitigating_2022}, with interest in Bayesian PINNs for uncertainty quantification (UQ) \citep{yang_b-pinns_2020, linka_bayesian_2022, graf_uncertainty_2021, psaros_uncertainty_2023}.

The use of NNs for solving ODEs and PDEs was pioneered by \cite{lagaris_artificial_1997} and later advanced by \cite{raissi_physics-informed_2019}, who introduced PINNs for forward and inverse problems. Forward problems solve for the solution given the equation and boundary conditions, while inverse problems estimate unknown parameters. Inverse problems were implemented by \cite{raissi_physics-informed_2019} by treating these parameters as trainable variables.

Although effective, this approach lacks a robust mechanism for uncertainty quantification, which Bayesian Neural Networks (BNN) address by introducing distributions over network weights. However, simply applying Bayesian methods does not fully leverage the available information about solution accuracy. This is where Solution Bundles \citep{flamant_solving_2020} come into play. Solution Bundles enable statistical analysis over multiple possible solutions, providing a more comprehensive view of uncertainty in both the solution and the equation parameters.

In cosmology, DEs aim to explain the universe's expansion. Testing new models against observations typically involves statistical analysis to determine parameter bounds. By combining Solution Bundles with BNNs, we enhance uncertainty quantification, offering a more reliable method for parameter estimation. Our contributions include:

\begin{itemize}
    \item Introducing error-bound-based heteroscedastic variance for better uncertainty quality.
    \item Solving forward problems for cosmological equations while quantifying uncertainties.
    \item Applying Solution Bundles to solve inverse problems for parameter estimation.
\end{itemize}
%%%%% BACKGROUND

\section{Background}
\subsection{Problem Formulation}\label{sec:prob_form}
We adopt a slightly different formulation from \cite{psaros_uncertainty_2022}. The DEs we will work with can be defined as follows:
\begin{align}
    \mathcal{F}_{\lambda}[u(\x)] &= f(\x), \quad \x\in \Omega, \label{de1} \\
    \mathcal{B}_{\lambda}[u(\x)] &= b(\x), \quad \x\in \Gamma,
\end{align}
where $\x$ is the space-time coordinate, $\Omega$ is a bounded domain with boundary $\Gamma$, $f(\x)$ is the source term, $u$ is the solution of the system, $\mathcal{F}_{\lambda}$ is a differential operator, $\mathcal{B}_{\lambda}$ and $b(\x)$ are the boundary conditions (BCs) operator and term, respectively, and $\lambda$ denotes the parameters of the system.

In this paper, we will focus on problems where the operators $\mathcal{F}_{\lambda}$, $\mathcal{B}_{\lambda}$ and terms $f(\x)$, $b(\x)$ are known. If $\lambda$ is assumed to be known, the goal is to find the solution $u$, referred to as the \textit{forward problem}. Conversely, if $u$ is known and the aim is to estimate $\lambda$, then this is known as the \textit{inverse problem}.

\subsection{Physics-Informed Neural Networks}\label{pinns}
A \textit{Physics-Informed Neural Network} uses a neural network, $u_{\theta}(\x)$, to approximate the true solution $u(\x)$ of a differential equations system. As discussed by \cite{psaros_uncertainty_2023}, PINNs can be trained by minimizing a fitting dataset's Mean Squared Error (MSE). A fitting dataset $\mathcal{D} = \{\mathcal{D}_f, \mathcal{D}_b\}$ is composed of noisy observations of $f$, $\mathcal{D}_f = \{\x_i, f_i\}_{i=0}^{N_f}$, and noisy BCs data $\mathcal{D}_b = \{\x_i, b_i\}_{i=0}^{N_b}$
\begin{equation}
    \label{eq:loss}
    \mathcal{L}(\theta)=\frac{w_f}{N_f}\sum_{i=0}^{N_f}\mse{f_{\theta}(\x_i) - f_i} + \frac{w_b}{N_b}\sum_{i=0}^{N_b}\mse{b_{\theta}(\x_i) - b_i}
\end{equation}

Where $w_f$ and $w_b$ are weighting constants. \cite{psaros_uncertainty_2023} call this setup the \textit{Forward Deterministic PDE Problem}, they also describe \textit{Mixed Deterministic PDE Problem} where the objective is to obtain solutions for $u$ and $\lambda$, and the \textit{Mixed Stochastic PDE Problem} which deals with stochastic PDEs.

In this work, we adopt a different optimization problem to solve the forward deterministic problem and address the inverse problem separately rather than solving both forward and inverse problems simultaneously.

\subsection{Solving Forward Problems} \label{sec:forward-problems}
We begin by defining the \textit{residual} of a differential equation as
\begin{equation}\label{eq:res_deff}
    r_{\theta}(\x) = \mathcal{F}_{\lambda}[u_{\theta}(\x)] - f(\x)
\end{equation}

The computation of $\mathcal{F}_{\lambda}[u_{\theta}(\x)]$ is easy to implement thanks to automatic differentiation provided by Deep Learning frameworks such as PyTorch \citep{paszke_pytorch_2019}. This formulation allows PINNs to be trained as a self-supervised network.
%by minimizing the module of residuals.
Since any solution $u^*$ to the differential equation satisfies $\mathcal{F}_{\lambda}[u^*(\x)] - f(\x) = 0$, we train the network to minimize the square of the residual
\begin{equation}
    \label{eq:res-loss}
    \min_{\theta} \frac{1}{N_r}\sum_i^{N_r}r_{\theta}^2(\x_i).
\end{equation}

Usually, $\x_i$ are sampled from $\Omega$ with a uniform distribution or by taking an equally-spaced subset. It is important to note that we do not deal with noisy data, unlike the work by \cite{raissi_physics-informed_2019}. 

\subsection{Enforcing Boundary Conditions}\label{sec:enforcing}
While adding a term for BCs in the loss (see \cref{eq:loss}) when dealing with data is a good way to incorporate such knowledge, there is no guarantee that the conditions will be satisfied. When BCs are known rather than observed in the data, we can use the transformation introduced by \cite{lagaris_artificial_1997} that enforces the BCs to be always satisfied. This is achieved by writing the approximate solution as a sum of two terms: 
\begin{equation}
    \label{eq:reparam}
    \tilde u_{\theta}(\x):= A(\x) + F(\x, \u(\x))
\end{equation}
where $A$ does not depend on the network parameters $\theta$ and it satisfies the BCs. Since we need $\tilde u_{\theta}$ to satisfy BCs, $F$ is constructed so it does not contribute to them. This transformation is also used by \cite{graf_uncertainty_2021, chen_neurodiffeq_2020}.

\paragraph{One-dimensional Initial Value Problem}
Given an initial condition $u_0=u(t_0)$, we consider a transformation
\begin{equation}
    \tilde u_{\theta}(t):=u_0 + (1-e^{-(t-t_0)}) \u(t)
\end{equation}
In \cref{sup:enforce} of the Supplementary Material we show the enforcing of two-dimensional Dirichlet BCs. Similar transformations can be defined for Neumann and mixed BCs. 
%The Python package Neurodiffeq \cite{chen_neurodiffeq_2020} implements various of these.

\subsection{Solution Bundles}\label{sec:bundles}

Solution Bundles \citep{flamant_solving_2020} extend PINNs by allowing the network to take equation parameters $\lambda \in \Lambda$ as inputs. This modification allows the network to approximate a variety of solutions to a parameterized differential equation without the need to retrain for each value of $\lambda$.

In \cref{sec:prob_form} we considered a unique value for $\lambda$ and one BC term $b(\x)$. When working with Solution Bundles instead, we have subsets for the equation parameters and BCs.
\begin{gather*}
    \lambda \in \Lambda \subset \mathbb{R}^p\\
    b(\x) \in B(\x) \subset \mathbb{R}^n
\end{gather*}
$\Lambda$, $B(\x)$ are such subsets that will be used to train the NN. $p$ and $n$ are the dimensionality of the equation parameters and the system's state variable, respectively. The transformation described in \cref{sec:enforcing} can also be used for Solution Bundles. \cref{eq:reparam} turns into:
\begin{equation}
    \tilde u_{\theta}(\x,\lambda):= A(\x,\lambda) + F(\x, \u(\x,\lambda))
\end{equation}

\subsubsection{Training Solution Bundles}
\cite{flamant_solving_2020} proposed a weighting function for the residual loss (\cref{eq:res-loss}) to influence how the approximation error is distributed across the training region. 
%They show a bound for the error $|\u(\x) - u(\x)|$ that depends on the Lipschitz constant of the forcing function $f$. This bound grows exponentially as the training points move away from the initial conditions. Because of this, they propose a weighting function that captures this behavior and that also grows exponentially. 
However, in this work, for simplicity, we stick to the unweighted residual loss
\begin{align}
&\tilde{r}_{\theta} := \mathcal{F}_{\lambda}[u_{\theta}(\x, \lambda)] - f(\x)\\
\label{eq:bundle_loss}&\min_{\theta} \sum_i^N\sum_j^M\tilde{r}_{\theta}^2(\x_i, \lambda_j).
\end{align}
where we have redefined $r_{\theta}$ from \cref{eq:res_deff} for the Solution Bundle case as $\tilde{r}_{\theta}$.

\subsubsection{Solution Bundles for Solving Inverse Problems}
So far we have explained how to solve forward problems, that is, finding a solution to a differential equations system. As we described, for the case of PINNs, the solving step is an optimization problem.

On the other hand, inverse problems aim to find a differential system that best describes some collected data. For this, we assume the differential system can explain the phenomena observed and we seek to estimate the system's parameters for a given dataset.

An effective approach to addressing the parameter estimation problem involves statistical analysis, specifically through the application of Bayesian methods. This approach necessitates multiple computations of the system's solution, corresponding to each parameter value. Traditional numerical methods require the discretization and integration process to be performed for each of these solutions. In contrast, Solution Bundles eliminate the need for retraining for each parameter value, thereby expediting the computational process. A comprehensive explanation of the probabilistic setup for Bayesian parameter estimation is provided in \cref{sec:prob_inverse}.

\subsection{Error Bounds for PINNs}\label{sec:eb}
Good quality uncertainties should correlate with the true error of a solution. Since the true error is not accessible, we use error bounds in \cref{sec:interpretation} to improve the uncertainty in the Bayesian NNs.

In \citep{liu_evaluating_2022, liu_residual-based_2023}, the authors present algorithms for computing error bounds on PINNs. These bounds apply to linear ODEs, systems of linear ODEs, non-linear ODEs in the form \( \epsilon v^k \)\footnote{Here, \( v \) is a variable and \( |\epsilon| \ll 1. \)}, as well as certain types of PDEs. These algorithms are independent of the NN architecture and depend solely on the structure of the equation as defined in \cref{de1} and the residuals of the DE.

For the Solution Bundle setup, the network error is denoted as $\eta(\x, \lambda) := u(\x, \lambda) - \tilde{u}_{\theta}(\x, \lambda) $, and the error bound is represented by a scalar function $\mathbb{B}$ such that
\begin{equation}
    \lVert\eta(\x, \lambda)\rVert \leq \mathbb{B}(\x, \lambda)
\end{equation}

In this work we use the error bounds developed for first-order linear ODEs with nonconstant coefficient. In \cref{sup:errorbounds} of the Supplementary Material provide the algorithm for its computation and for first-order linear ODEs with constant coefficient. \cref{sup:tight_bounds} and \cref{alg:eb} describe how obtain tight bounds.

\subsection{Bayesian Neural Networks} \label{sec:bnn_for_uq}
To quantify uncertainty, we adopt a Bayesian perspective on the neural networks. We do this by viewing the neural network as a probabilistic model $p(y | x, \theta)$ and placing a prior distribution over its parameters $p(\theta)$. Using Bayes' theorem, we can get the posterior distribution of $\theta$:
\begin{equation}
    p(\theta|\mathcal{D})=\frac{p(\mathcal{D}|\theta) p(\theta)}{p(\mathcal{D})}
\end{equation}
$p(\mathcal{D}|\theta)$ is the likelihood distribution over a dataset $\mathcal{D}=\{(\mathbf{x}_i, y_i)\}_{i=1}^{N}$.

The posterior distribution allows us to make predictions about unseen data by taking expectations. Consider a new data point $\hat{x}$, we can obtain the probability of the output $y$ being $\hat{y}$ as:
\begin{equation}
   \label{posterior_predictive}
    p(\hat{y}|\hat{x}, \mathcal{D})=\mathbb{E}_{p(\theta|\mathcal{D})}[p(\hat{y}|\hat{x}, \theta)]=\int_{\Theta}p(\hat{y}|\hat{x}, \theta)p(\theta|\mathcal{D})d\theta
\end{equation}
Usually, the integral in \cref{posterior_predictive} is analytically intractable, and we have to resort to Monte Carlo (MC) approximations that can be computed as:
\begin{equation}
    \label{mcposterior}
    p(\hat{y}|\hat{x}, \mathcal{D})\approx \sum_{i=1}^M p(\hat{y}|\hat{x}, \theta_i), \text{ where }\theta_i\sim p(\theta|\mathcal{D})
\end{equation}

\section{Shortcomings of Residual-Based UQ Methods in PINNs}

A direct application of Bayesian Neural Networks to PINNs is straightforward by placing the likelihood over the residuals $r$ of the PINN solution given a coordinate point $\x$ i.e. $p(r | \x, \theta)$. For more details see \cref{sec:baseline}.

However, evaluating a PINN solution based only on the residual loss \cref{eq:res-loss} can be misleading. A low residual at $\x$ does not guarantee a low solution error at $\x$. In Control Theory, for instance, an Integral Controller drives steady-state error to zero, yet errors may still exist:
\begin{equation}
    u(t) = k_i\int_{0}^te(\tau)d\tau
\end{equation}
where $u$ is the control input, $e(\tau):=r(\tau) - y(\tau)$ the error, $r$ the reference, $y$ the system output, and $k_i$ the integral gain \citep{astrom_feedback_2008}. For a DE $\diff{u}{t} = f(t)$ with IC $u(t_0) = u_0$, the approximation error $\hat u(\tau) - u(\tau)$ can be expressed as:
\begin{equation}
\label{eq:mybound}
    \hat u(\tau) - u(\tau) = \int_{t_0}^{\tau}R(t)dt
\end{equation}
where $R(t) := \diff{\hat u}{t} - f(t)$, given the approximation satisfies the ICs. Proof of \cref{eq:mybound} is in \cref{prop:bounds} of the Supplementary Material. \cref{eq:mybound} suggests that the residual at some point $\x$ is not enough to characterize the solution error at the same point $\x$.

These challenges, along with other \textit{Failure Modes} of PINNs, remain an active area of research \citep{krishnapriyan_characterizing_2021, wang_respecting_2022, penwarden_unified_2023}. Notably, \cite{wang_respecting_2022} identify an implicit bias in the PINN framework that can severely violate the temporal causality of dynamical systems. As a result, residual minimization at a given time $t_i$ may occur even when predictions at earlier times are inaccurate.

In \cref{fig:shortcomings} we see how a PINN solution has near maximum and minimum error for the same residual value. This illustrates the decoupling of residuals from solution errors, which can be attributed to a violation of temporal causality. Furthermore, \cref{fig:err_vs_res_lcdm,fig:err_vs_res_cpl,fig:err_vs_res_quint,fig:err_vs_res_hs} show how there is no clear correlation between residuals and solution errors. These figures were obtained by training a deterministic PINN on a cosmological model, the construction details are provided in \cref{sm:sol_dist_exp}.

Our experiments demonstrate that the baseline method produces uncalibrated predictive distributions and, in some cases, fails to approximate the true solutions. To address these limitations, we propose a two-step approach that incorporates error bounds into the solutions, leading to predictive distributions with improved calibration.

\begin{figure}
    \centering
    \includegraphics[width=\linewidth]{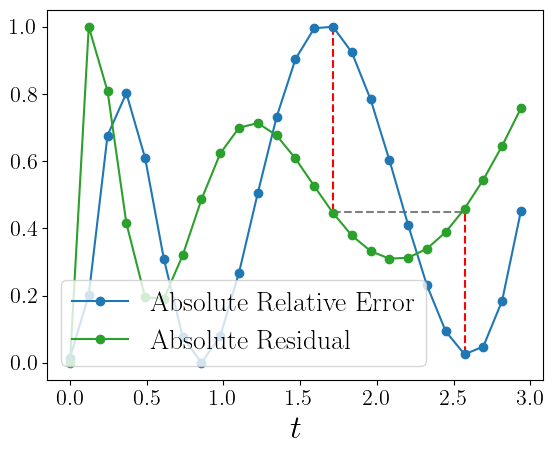}
    \caption{Normalized Absolute Residual vs Normalized Absolute Relative Error of a Deterministic NN Solution for CPL Model.}
    \label{fig:shortcomings}
\end{figure}

%\paragraph{Residuals Scale}
%By placing the likelihood over the residuals, we associate uncertainty with the variance of a random variable representing these residuals. This variance characterizes the dispersion in the magnitude of residuals rather than the solution to the DE. While this approach to uncertainty quantification may prove beneficial in certain applications, its utility may not extend universally. Notably, it precludes the construction of confidence intervals directly over the solution. 

%\paragraph{Error Bounds Leverage}
%In \cref{sec:eb}, we introduced the work done by \cite{liu_evaluating_2022, liu_residual-based_2023}, which provides algorithms to compute error bounds on PINN solutions. In addition to the flaws described, we can leverage these error bounds in our two-step approach to capture this information and enhance the uncertainty quality. In \cref{sec:improve-uq} we explain how we use the error bounds.

\section{Two-Step Bayesian PINNs}
We use a two-step approach to obtain uncertainties in the solutions of equations. In the first step, we train a PINN as a Solution Bundle that we refer to as the \textit{deterministic network}. We denote this network as $u_{\theta_{\text{det}}}: \Omega, \Lambda \rightarrow \mathbb{R}$ and $\theta_{\text{det}}^*$ as the parameters resulting after training.

Note that as shown in \cref{eq:bundle_loss}, this training step is carried out without the use of any data, but rather sampling from the network domain ($\Omega \times \Lambda$) and minimizing the network's residuals. \cref{fig:method} shows a diagram of the entire training process.

\begin{figure}
    \centering
    \includegraphics[width=\linewidth]{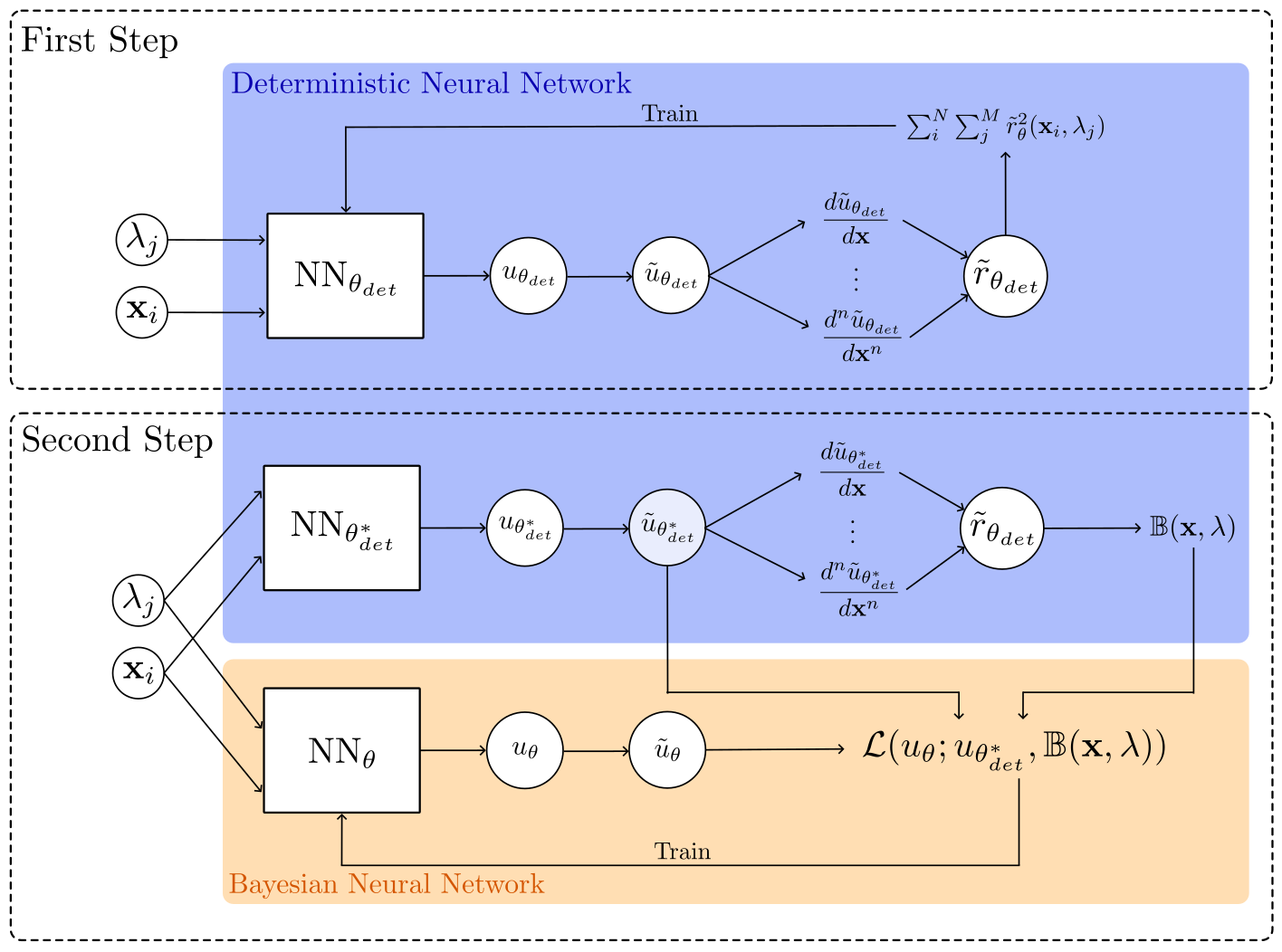}
    \caption{Diagram of Both Training Steps. In the Second Step $\theta_{\text{det}}^*$ are The Resulting Network Parameters From the First Step and $\mathcal{L}$ is The Likelihood Function.}
    \label{fig:method}
\end{figure}

\subsection{Bayesian Neural Network Training}\label{sec:bayes_training}
In the second step, we use the outputs of $u_{\theta_{\text{det}}^*}$ as targets to train a Bayesian Neural Network.

We construct a dataset by taking the space-time coordinates and equation parameters as independent variables, and the outputs of the deterministic net as dependent variables:\scriptsize
\begin{equation}
    \mathcal{D} = \{(\mathbf{x}_i, \lambda_i, u_{\theta_{\text{det}}^*}(\mathbf{x}_i, \lambda_i))\;|\; \mathbf{x}_i\in \Omega, \lambda_i \in \Lambda, u_{\theta_{\text{det}}^*}(\mathbf{x}_i, \lambda_i)\in\mathbb{R}\}_{i=1}^{N'}\nonumber
\end{equation}
\normalsize

%\subsubsection{Distributions}
We define a BNN $u_{\theta}$ by assigning distributions to the network parameters and dataset as it is done in \citep{graf_uncertainty_2021}. For all three methods described in \cref{sec:bnn_for_uq} we use a Gaussian prior:
\begin{equation*}
    \theta \sim \mathcal{N}(0, \sigma_{\text{prior}})
\end{equation*}
To formulate the likelihood, we assume a Gaussian additive noise on the \textit{observations}\footnote{Here we call $\udetopt(\x, \lambda)$ \textit{observations} even though they are not actual experimental observations.}
\begin{align}
    \label{eq:lik_noise}
    \udetopt(\x, \lambda) &= \u(\x, \lambda) + \eta(\x,\lambda),\\
    &\text{ where } \eta(\x,\lambda) \sim \mathcal{N}(0, \slike(\x,\lambda))\nonumber
\end{align}
This setup is often used in machine learning, but we employ it with a different interpretation; see \cref{sec:interpretation}.
In this paper, we model the standard deviation $\slike(\x, \lambda)$ as known, but it can also be modeled as unknown. Given \cref{eq:lik_noise}, the resulting likelihood is:
% \begin{align}
%     &p(\udetopt(\x, \lambda)\;|\;\x,\lambda, \theta) = \mathcal{N}(\u(\x, \lambda),\slike(\x,\lambda))\label{eq:sample_dist}\nonumber\\
%     &p(\mathcal{D}|\theta) = \prod_{i=1}^{N'} p(\udetopt(\mathbf{x}_i, \lambda_i)\;|\; \mathbf{x}_i, \lambda_i, \theta)
% \end{align}
\begin{equation}
    p(\mathcal{D}|\theta) = \prod_{i=1}^{N'} \mathcal{N}(\u(\x, \lambda),\slike(\x,\lambda))\label{eq:sample_dist}
\end{equation}

We assume Gaussian noise in \cref{eq:lik_noise} due to its simplicity, tractability (see \cref{sec:post_aprox}), and suitability as a starting point for evaluating our methodology. To assess this assumption, we analyze the distributional behavior of $\udet$, with details provided in \cref{sm:sol_dist_exp}. \cref{fig:lcdm_sol_dist,fig:cpl_sol_dist,fig:quint_sol_dist_1,fig:quint_sol_dist_2,fig:hs_sol_dist_1,fig:hs_sol_dist_2,fig:hs_sol_dist_3,fig:hs_sol_dist_4,fig:hs_sol_dist_5} illustrate the solution distributions obtained by a PINN for various cosmological models. Although these distributions deviate from a perfect Gaussian, they are not significantly distant, suggesting that a Gaussian approximation is a reasonable initial choice.

\subsubsection{Posterior Distribution Approximation Methods}\label{sec:post_aprox}
To apply the MC approximation in \cref{posterior_predictive}, we must generate samples from the posterior distribution \( p(\theta|\mathcal{D}) \). Given the complexity of neural networks and the need for computational efficiency, we compare three methods: Neural Linear Models (NLMs), Bayes By Backpropagation (BBB), and Hamiltonian Monte Carlo (HMC). NLMs offer a lightweight approach by approximating the posterior in a linearized feature space, making them suitable for speed-critical applications. BBB uses variational inference to achieve greater accuracy while maintaining tractability for larger networks. HMC is the most advanced method, producing highly accurate samples at the expense of increased computational complexity. Comparing these techniques allows for an evaluation of the trade-offs between accuracy and computational feasibility.

\paragraph{Neural Linear Models (NLMs)} NLM represent Bayesian linear regression using a neural network for the feature basis, with only the final layer parameters treated as stochastic. Using Gaussian prior and likelihood allows for tractable inference, leading to the posterior and predictive distributions given by:
\begin{align}
    &p(u \;|\; \mathbf{x}, \lambda, \mathcal{D}) = \mathcal{N}(\mu_{\text{NLM}}(\x, \lambda), \sigma_{\text{NLM}}(\x, \lambda))\label{eq:nlm}\\
    &\mu_{\text{NLM}}(\x, \lambda)=\Phi_{\theta}(\mathbf{x},\lambda) \mu_{\text{post}}\nonumber\\
    &\sigma_{\text{NLM}}(\x, \lambda) = \slike^2(\mathbf{x},\lambda)+\Phi_{\theta}(\mathbf{x},\lambda)\Sigma_{\text{post}}\Phi_{\theta}^T(\mathbf{x},\lambda)\nonumber
\end{align}
where \( \mu_{\text{post}} \) and \( \Sigma_{\text{post}} \) are posterior parameters, and \( \Phi_{\theta}(\mathbf{x}, \lambda) \) is the learned feature map. The details of $\mu_{\text{post}}$ and $\Sigma_{\text{post}}$ computation can be found in \cref{sup:nlm_details} of the Supplementary Material.

\paragraph{Bayes By Backpropagation (BBB)} BBB approximates the posterior of network parameters using variational inference (VI) and minimizes the KL divergence between the true posterior and a variational distribution \( q(\theta | \rho) \). This results in an approximation \( q(\theta |\rho) \) that is used to generate Monte Carlo samples for \( p(\theta | \mathcal{D}) \). We assume independent parameters under a mean-field approximation\citep{bishop_pattern_2006}, defining \( q(\theta | \rho) = \prod_{i=1}^N q(\theta_i|\rho_i) \), using a Gaussian distribution for the variational posterior \citep{blundell_weight_2015}.

\paragraph{Hamiltonian Monte Carlo (HMC)} HMC employs a Metropolis-Hastings algorithm \citep{metropolis_equation_2004, hastings_monte_1970} to draw samples from \( p(\theta | \mathcal{D}) \) by simulating particle movement through Hamiltonian dynamics. To enhance sampling efficiency, we utilize the No-U-Turn Sampler \citep{homan_no-u-turn_nodate}.

\begin{figure*}[t]
    \begin{subfigure}{\linewidth}
         \centering
         \includegraphics[width=.63\linewidth]{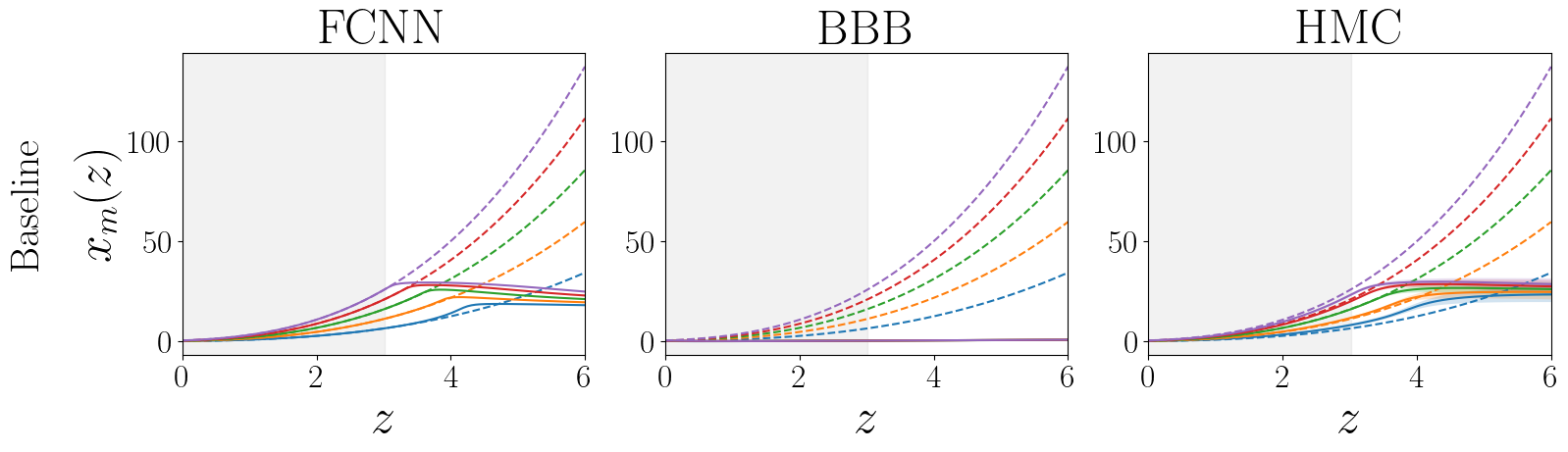}
         \label{fig:lambda_res}
     \end{subfigure}
    \begin{subfigure}{\linewidth}
         \centering
         \includegraphics[width=.8\linewidth]{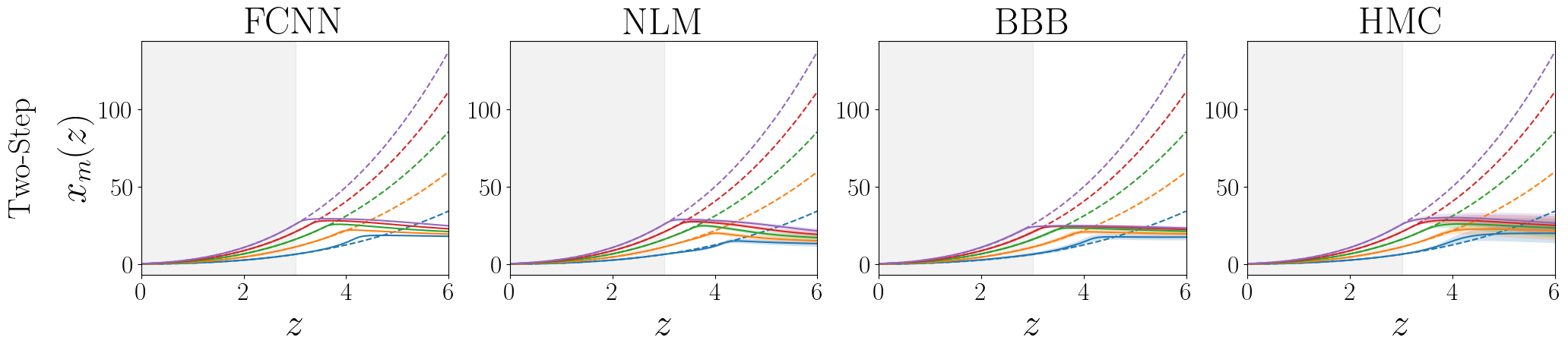}
         \label{fig:lambda}
     \end{subfigure}
     \begin{subfigure}{\linewidth}
         \centering
         \hspace{1.5em}
         \includegraphics[width=.93\linewidth]{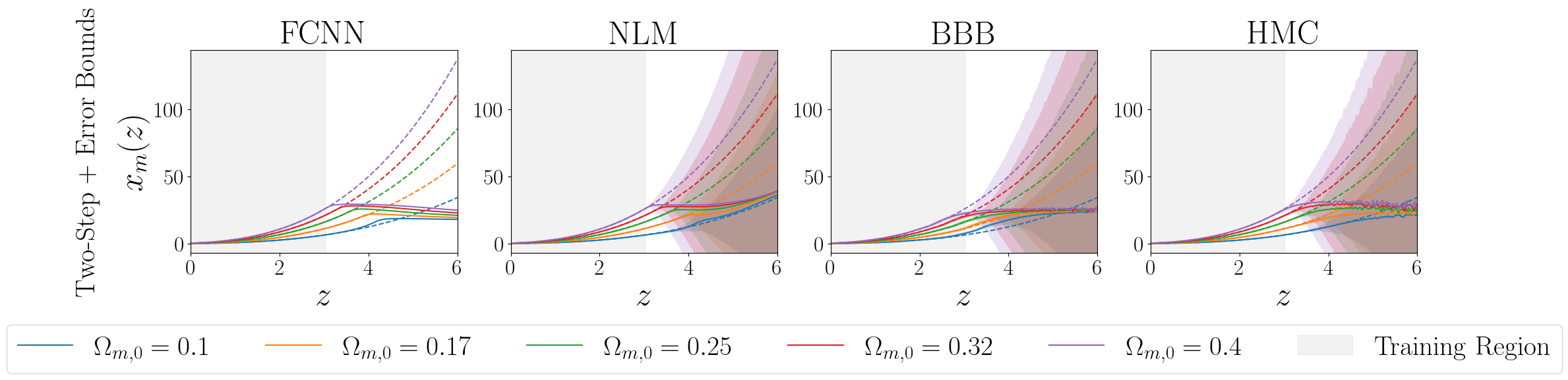}
         \label{fig:lambda_eb}
     \end{subfigure}
    %\vspace{1em}
    \caption{Examples of $\Lambda$CDM Bayesian Solutions Obtained Using the Bundle Solver. Analytic Solutions are Presented in Dotted Lines.}
    \label{fig:lambda}
\end{figure*}

\subsection{Improving Predictive Uncertainty with Error Bounds}
\label{sec:interpretation}
\label{sec:improve-uq}
When learning from observations \(\hat{y}\), the standard approach assumes these observations have an error \(\epsilon\):
\begin{equation}
\label{eq:obs_error}
    \hat{y} = y_w(\mathbf{x}) + \epsilon
\end{equation}

where \(y_w(\mathbf{x})\) is a deterministic function parameterized by \(w\). This formulation attributes all errors to the observations. However, we can introduce an additional error term \(\eta\) for model error:
\begin{equation}
    \label{eq:model_error}
    \hat{y} = y_w(\mathbf{x}) + \eta + \epsilon
\end{equation}

Separating model error \(\eta\) from observational error \(\epsilon\) is typically challenging due to their additive nature. If we know there is no observational error, we can use \cref{eq:model_error} similarly to \cref{eq:obs_error}. This is the case in \cref{sec:bayes_training}, where the dataset has \textit{observations} \(\udetopt(\mathbf{x}, \lambda)\) free of noise:
\begin{equation}
    \udetopt(\mathbf{x}, \lambda) = \u(\mathbf{x}, \lambda) + \cancel{\epsilon} + \eta
\end{equation}

However, the approximate solution \(\udetopt\) may still have prediction errors. We model this error with a Gaussian distribution with standard deviation \(\slike(\mathbf{x}, \lambda)\), which we treat as known in two ways:

\begin{enumerate}
    \item \textbf{Homoscedastic:} \\
    \begin{equation*}
        \sigma_{\text{Like}}(\mathbf{x}, \lambda)=\begin{cases}
              0  & \mathbf{x} \in \Gamma, \\
              \text{const} & \mathbf{x} \not\in \Gamma
                \end{cases} 
    \end{equation*}
    We set $\sigma_{\text{Like}}$ to be zero on the boundaries because we enforce BC always to be met.
    \item \textbf{Error Bounds Based Heteroscedastic:} \\
    \begin{equation*}
        \slike(\x, \lambda) = \mathbb{B}(\x, \lambda)
    \end{equation*}
    This choice ensures we are taking into account the error made by $\udetopt$ in the predictive uncertainty.
\end{enumerate}

\subsection{Baseline}\label{sec:baseline}
In the literature, the standard approach to formulating the \textit{forward} problem with Bayesian PINNs relies on noisy observations of $u(\x)$, $f(\x)$ and $b(\x)$ \citep{psaros_uncertainty_2022,zou_uncertainty_2023,yang_b-pinns_2020}. However, since our setup does not incorporate observations, we adopt as a baseline a Bayesian neural network with the likelihood defined over the residuals:
\begin{equation}
    p(\mathcal{D}_{r}|\theta) = \prod_{i=1}^{N^r} \mathcal{N}(r_{\theta}(\x, \lambda),\sigma_{r})
\end{equation}

The training procedure follows the approach described in \cref{sec:bayes_training}, with the exception that we could not implement NLM in this setting. This limitation arises because the residuals depend on the solution, thereby breaking the analytical tractability required for NLM.

\section{Probabilistic Formulation of Inverse Problems}
\label{sec:prob_inverse}
The Bayesian framework can be used to estimate the parameters $\lambda$ of a DE system. For a given set of observations $\mathcal{O} = \{(\mathbf{x}_i, \mu_i, \sigma_i) \mid \mathbf{x}_i \in \Omega, \mu_i \in \mathbb{R}, \sigma_i \in \mathbb{R}^+ \}_{i=1}^O$, we seek to find the posterior distribution $p(\lambda| \mathcal{O})$ of the parameters. Here $\mu_i$ and $\sigma_i$ are the mean and standard deviation of the observations at some point $\x_i$, respectively.

Assume we have a probability distribution over the solutions $p(u_{\lambda}(\x)|\x, \lambda)$, then using Bayes' Theorem, the posterior can be computed as $p(\lambda|\mathcal{O})=p(O|\lambda)/ p(\lambda)$.
The likelihood $p(\mathcal{O}|\lambda)$ is obtained by marginalizing over the solutions as:
\begin{align}
        \label{eq:marginalization}
        p(\x_i, \mu_i, \sigma_i|\lambda)
        &=\int_{\mathcal{U}}p(\x_i, \mu_i, \sigma_i|u_{\lambda}(\x_i))\cdot p(u_{\lambda}(\x_i)|\x_i, \lambda)du\\
        &\approx \frac{1}{M}\sum_{j=1}^Mp(\mu_i, \sigma_i|u_{\lambda}^{(j)}(\x_i))\\
        p(\mathcal{O}|\lambda)&=\prod_{i=1}^O p(\x_i, \mu_i, \sigma_i|\lambda)
\end{align}
where $u_{\lambda}^{(j)}\sim p(u_{\lambda}(\x_i)|\x_i, \lambda)$.

The distribution induced by FCNNs can be interpreted as a Delta distribution, where the function learned by the network outputs a specific value \( \udet(\mathbf{x}) \) for a given input $\x$. In contrast, BNNs naturally provide a distribution over solutions through their posterior predictive, thus \cref{eq:marginalization} results in a marginalization over the posterior predictive distribution \( p(u(\mathbf{x}_i, \lambda)|\mathbf{x}_i, \lambda, \mathcal{D}) \).

%\subsection{Choosing Distributions}
We use a uniform distribution to aim for an uninformative prior for $\lambda$. In the same way \cite{chantada_cosmological_2022} did, to define the likelihood $p(\mathcal{O}|\lambda)$, we assume the observations are normally distributed around the true solution, i.e., $\mu_i \sim \mathcal{N}(u(\mathbf{x}_i), \sigma_i)$.
% Thus, \cref{eq:bayes_marginalization} becomes
% \begin{equation*}
%     p(x_i, \mu_i, \sigma_i|\lambda)=\int_{\mathcal{U}}\mathcal{N}(\mu_i;u(x_i, \lambda), \sigma_i)\cdot p(u(x_i,\lambda)|x_i, \lambda, \mathcal{D})du
% \end{equation*}

Having defined the prior and likelihood, we can apply a sampling algorithm to approximate $p(\lambda|\mathcal{O})$. We use the \texttt{emcee} Python package \citep{foreman-mackey_emcee_2013} which implements the samplers introduced by \cite{goodman_ensemble_2010}.

\begin{table}
    \centering
    \caption{Cosmology Equations Variables and Parameters. Parameters Marked With * are Inputs in the Bundle Solution, The Remainder are Needed to Compute the Hubble Parameter.}
    \begin{adjustbox}{max width=\columnwidth}
    \begin{tabular}{c c c c}
         \hline \hline
         Equation & Type & Variables & Parameters \\ \hline
         $\Lambda$CDM & Linear Equation & $x_m$ & $\Omega_{m,0}^*, H_0$ \\
         CPL & Linear Equation & $x_{DE}$ & $w_0^*,w_1^*,\Omega_{m,0}, H_0$ \\
         Quintessence & Non-linear System & $x,y$ & $\lambda^*, \Omega_{m,0}^*, H_0$ \\
         HS & Non-linear System & $x,y,v,\Omega,r$ & $b^*, \Omega_{m,0}^*, H_0$ \\
         \hline \hline
    \end{tabular}
    \end{adjustbox}
    \label{tab:cosmo_eqs}
\end{table}

%%%%%% EXPERIMENTS

% \begin{figure}[t]
%     \begin{subfigure}{\linewidth}
%          \centering
%          \includegraphics[width=\linewidth]{figures/lcdm/calibration_lcdm_bundle_all.png}
%      \end{subfigure}
%      \begin{subfigure}{\linewidth}
%          \centering
%          \includegraphics[width=\linewidth]{figures/lcdm/calibration_lcdm_bundle_eb_all.png}
%      \end{subfigure}
%     %\vspace{1em}
%     \caption{Calibration Curves Of Bundle Solutions For $\Lambda$CDM.}
%     \label{fig:lambda_calibration}
% \end{figure}

\section{Experiments on Cosmology Models}\label{sec:experiments}
This section specifies the cosmological equations we used to test our methodology. These equations were solved with PINNs by \cite{chantada_cosmological_2022}. For brevity, we describe the $\Lambda$CDM and Parametric Dark Energy models. The details of Quintessence and $f(R)$ gravity (or HS for the name of the authors Hu and Sawicki) are provided in \cref{sup:cosmology}. However, \cref{tab:cosmo_eqs} lists the variables and parameters of each cosmological model.

\paragraph{$\Lambda$CDM}
The equation and initial conditions are
\begin{align*}
    \frac{dx_m}{dz}&=\frac{3x_m}{1+z} \\
    x_m(z=0)&=\Omega_{m,0}
\end{align*}
where $\Omega_{m,0}$ is a parameter. Having a solution to this equation, the Hubble parameter can be obtained as
\begin{equation}
    H(z)=H_0\sqrt{x_m(z)+1-\Omega_{m,0}}
\end{equation}
here $H_0$ is also a parameter.

\paragraph{Parametric Dark Energy} 
We refer to this model as CPL for the names of the authors Chevallier, Polarski and Linder \citep{linder_exploring_2003,chevallier_accelerating_2001}. The equation and initial conditions are
\begin{align*}
    \frac{dx_{DE}}{dz}&=\frac{3x_{DE}}{1+z}\left(1+w_0+\frac{w_1z}{1+z}\right)\\
    x_{DE}(z=0) &= 1-\Omega_{m,0}
\end{align*}
where $w_0, w_1$ and $\Omega_{m,0}$ are parameters. Having a solution to this equation, the Hubble parameter can be obtained as
\begin{equation}
    H(z) = H_0\sqrt{\Omega_{m,0}(1+z)^3 +x_{DE}(z)}
\end{equation}
here $H_0$ is also a parameter.

For brevity, we provide the error bounds computation details in \cref{sup:errorbounds} of the Supplementary Material.

\paragraph{Forward Problems}\label{sec:exp_forward} 
We trained \(\udetopt\) using methods from \citep{chantada_cosmological_2022}, then built a dataset \(\mathcal{D}\) to train the BNNs. We compared homoscedastic and error-bounds-based heteroscedastic variance for \(\Lambda\)CDM and CPL equations, while providing homoscedastic variance results for Quintessence and HS. The networks were trained using regular PINNs and Solution Bundles, referred to as \textit{Forward} and \textit{Bundle}, respectively, with implementation details in \cref{sec:imp_details} of the Supplementary Material. We provide the complete implementation in a code repository.\footnote{\url{https://github.com/ptflores1/improved-pinn-uq}}

We evaluated the NN solutions by calculating the Median Relative Error (MRE) against analytical solutions for \(\Lambda\)CDM and CPL, and numerical solutions for Quintessence and HS, using a Runge-Kutta method. We also computed the median residual, and miscalibration area (MA) \citep{chung_uncertainty_2021} to assess uncertainty quality. Results are summarized in \cref{tab:results}, with additional metrics in \cref{tab:test_quantiles,tab:metrics_forward,tab:metrics_bundle}.

For the CPL model, we utilized a reparameterization from \citep{chantada_cosmological_2022} that separates equation parameters from Bundle Network training, although this method is incompatible with Neural Linear Models (NLM).

\begin{table}
    \centering
    \caption{Evaluation Metrics of All Bundle Networks and Equations in the Testing Region. Here we Use 2S for Two-step and EB for Error Bounds.}
    \label{tab:results}
    \begin{adjustbox}{width=\linewidth}
        \begin{tabular}{c c | c c c }
    \hline
    \hline
    
    Equation & Method & Median RE & Median Residual & Miscal. Area\\
     \hline
     \multirow{9}{*}{\rotatebox[origin=c]{90}{$\Lambda$CDM}} & FCNN & $\mathbf{0.001}$ & 0.0 & -  \\
     \cline{2-5}
     & BBB & 0.978 & 0.147 & 0.491 \\
     & HMC & 0.084 & 0.198 & 0.4 \\
     \cline{2-5}
     & NLM + 2S & 0.018 & 8.157 & 0.183 \\
     & BBB + 2S & 0.035 & 0.963 & 0.143 \\
     & HMC + 2S & 0.004 & 0.197 & 0.123 \\
     \cline{2-5}
     & NLM + 2S + EB & 0.002 & 10.783 & 0.063 \\
     & BBB + 2S + EB & 0.047 & 2.476 & $\mathbf{0.05}$ \\
     & HMC + 2S + EB & 0.003 & 0.272 & 0.098 \\
     \hline
    
     \multirow{7}{*}{\rotatebox[origin=c]{90}{CPL}} & FCNN & $\mathbf{0.0}$ & 0.001 & - \\
     \cline{2-5}
     & BBB & 0.128 & 0.014 & 0.199 \\
     & HMC & $\mathbf{0.0}$ & 0.01 & 0.317 \\
     \cline{2-5}
     & BBB + 2S & 0.063 & 0.006 & 0.15 \\
     & HMC + 2S & 0.004 & 0.005 & 0.255 \\
     \cline{2-5}
     & BBB + 2S + EB & 0.033 & 0.003 & 0.177 \\
     & HMC + 2S + EB & 0.011 & 0.004 & $\mathbf{0.145}$ \\
     \hline
    
     \multirow{6}{*}{\rotatebox[origin=c]{90}{Quint.}} & FCNN & 0.007 & 0.0 & -  \\
     \cline{2-5}
     & BBB & 0.094 & 0.012 & 0.151 \\
     & HMC & $\mathbf{0.002}$ & 0.005 & 0.147 \\
     \cline{2-5}
     & NLM + 2S & 0.096 & 0.175 & 0.136 \\
     & BBB + 2S & 0.12 & 0.011 & 0.119 \\
     & HMC + 2S & 0.016 & 0.002 & $\mathbf{0.048}$ \\
     \hline
    
     \multirow{6}{*}{\rotatebox[origin=c]{90}{HS}} & FCNN & $\mathbf{0.001}$ & 0.0 & -  \\
     \cline{2-5}
     & BBB & 0.393 & 0.083 & 0.449 \\
     & HMC & 0.226 & 0.01 & 0.455 \\
     \cline{2-5}
     & NLM + 2S & 0.259 & 1.666 & $\mathbf{0.315}$ \\
     & BBB + 2S & 0.296 & 0.13 & 0.396 \\
     & HMC + 2S & 0.286 & 0.345 & 0.486 \\
    \end{tabular}
    \end{adjustbox}
\end{table}

\paragraph{Inverse Problems}
For the inverse problem, we used 30 measurements of the Hubble parameter \(H\) from the Cosmic Chronometers (CC) method \citep{simon_constraints_2005, stern_cosmic_2010, moresco_improved_2012, zhang_cong_four_2014, moresco_raising_2015, moresco_6_2016}. Each measurement includes a tuple \((z_i, H^{obs}, \sigma_{H^{obs}})\), indicating redshift, observed mean Hubble parameters, and their standard deviation. The CC dataset is available in \cref{tab:z_H_values} of the Supplementary Material, with \cref{eq:inverse2} used as the likelihood of observations.

We estimated equation parameters using the Solution Bundles from the forward step and performed inference with the \texttt{emcee} package \citep{foreman-mackey_emcee_2013}, running 32 chains for 10,000 steps, resulting in 320,000 samples per parameter. \cref{tab:inverse,tab:sigmas} show the results and their concordance with values found in the literature, respectively.

\section{Discussion}
The FCNN consistently achieves the lowest median RE and mean residual across all equations, outperforming Bayesian methods in terms of accuracy. This highlights a tradeoff between equipping PINNs with uncertainty quantification and maintaining accuracy. In this work, our primary focus is on uncertainty rather than pure accuracy, as a well-calibrated model with higher error is generally more desirable than an uncalibrated model with low error. Properly calibrated uncertainties enable a meaningful assessment of prediction reliability.

From \cref{tab:results}, we observe that the baseline Bayesian methods—BBB and HMC with a residual likelihood—exhibit significantly higher errors compared to the deterministic network, particularly for the CPL and HS equations. This effect is especially pronounced for BBB, which fails to approximate the $\Lambda$CDM solution. Additionally, both methods exhibit high miscalibration areas, indicating that their uncertainty estimates are poorly calibrated.

Introducing our two-step (2S) method substantially improves calibration, as evidenced by the reduction in miscalibration areas across all cases. For instance, applying 2S to BBB in the $\Lambda$CDM equation reduces the miscalibration area from 0.491 to 0.143.

Further incorporating error bounds into the two-step method enhances uncertainty calibration even further. Notably, BBB + 2S + EB achieves the lowest miscalibration area (0.05) for the $\Lambda$CDM equation, while HMC + 2S + EB provides the best calibration in CPL (0.145). However, certain cases exhibit excessively large mean residuals, such as in CPL, where BBB + 2S + EB results in a residual of $9.58 \times 10^8$. This suggests that while error bounds improve calibration, they may introduce numerical instability in the metrics due to some samples' extremely large errors.

For the Quintessence equation, HMC + 2S achieves the best calibration (0.048) while maintaining a low median RE (0.016), demonstrating an effective balance between accuracy and uncertainty quantification. In contrast, for the HS equation, while the two-step method improves calibration, overall performance remains suboptimal, with high median RE and residuals persisting across all Bayesian approaches.

\begin{table}
    \centering
    \caption{Evaluation Metrics of All Bundle Networks and Equations in the Training Region. Here we Use 2S for Two-step and EB for Error Bounds.}
    \label{tab:train}
    \begin{adjustbox}{width=\linewidth}
        \begin{tabular}{c c | c c c }
    \hline
    \hline
    
    Equation & Method & Median RE & Median Residual & Miscal. Area\\
     \hline
     \multirow{9}{*}{\rotatebox[origin=c]{90}{$\Lambda$CDM}} & FCNN & $\mathbf{0.0}$ & 0.0 & -  \\
     \cline{2-5}
     & BBB & 0.961 & 0.12 & 0.491 \\
     & HMC & 0.041 & 0.07 & 0.358 \\
     \cline{2-5}
     & NLM + 2S & 0.003 & 4.212 & 0.402 \\
     & BBB + 2S & 0.008 & 0.238 & 0.269 \\
     & HMC + 2S & 0.001 & 0.038 & 0.307 \\
     \cline{2-5}
     & NLM + 2S + EB & 0.001 & 4.211 & $\mathbf{0.016}$ \\
     & BBB + 2S + EB & 0.01 & 0.228 & 0.248 \\
     & HMC + 2S + EB & 0.001 & 0.014 & 0.251 \\
     \hline
    
     \multirow{7}{*}{\rotatebox[origin=c]{90}{CPL}} & FCNN & $\mathbf{0.0}$ & 0.001 & - \\
     \cline{2-5}
     & BBB & 0.052 & 0.016 & 0.161 \\
     & HMC & $\mathbf{0.0}$ & 0.009 & 0.482 \\
     \cline{2-5}
     & BBB + 2S & 0.024 & 0.009 & $\mathbf{0.085}$ \\
     & HMC + 2S & 0.001 & 0.008 & 0.436 \\
     \cline{2-5}
     & BBB + 2S + EB & 0.007 & 0.005 & 0.14 \\
     & HMC + 2S + EB & 0.001 & 0.008 & 0.159 \\
     \hline
    
     \multirow{6}{*}{\rotatebox[origin=c]{90}{Quint.}} & FCNN & 0.003 & 0.0 & -  \\
     \cline{2-5}
     & BBB & 0.059 & 0.005 & $\mathbf{0.09}$ \\
     & HMC & $\mathbf{0.001}$ & 0.003 & 0.279 \\
     \cline{2-5}
     & NLM + 2S & 0.063 & 0.097 & 0.326 \\
     & BBB + 2S & 0.047 & 0.002 & 0.165 \\
     & HMC + 2S & 0.006 & 0.0 & 0.179 \\
     \hline
    
     \multirow{6}{*}{\rotatebox[origin=c]{90}{HS}} & FCNN & $\mathbf{0.0}$ & 0.0 & -  \\
     \cline{2-5}
     & BBB & 0.373 & 0.162 & 0.426 \\
     & HMC & 0.257 & 0.121 & 0.43 \\
     \cline{2-5}
     & NLM + 2S & $\mathbf{0.244}$ & 0.201 & $\mathbf{0.252}$ \\
     & BBB + 2S & 0.284 & 0.094 & 0.387 \\
     & HMC + 2S & 0.279 & 0.111 & 0.427 \\
    \end{tabular}
    \end{adjustbox}
\end{table}

\begin{table}
    \centering
    \caption{Evaluation Metrics of All Bundle Networks and Equations in the OOD Region. Here we Use 2S for Two-step and EB for Error Bounds.}
    \label{tab:ood}
    \begin{adjustbox}{width=\linewidth}
        \begin{tabular}{c c | c c c }
    \hline
    \hline
    
    Equation & Method & Median RE & Median Residual & Miscal. Area\\
     \hline
     \multirow{9}{*}{\rotatebox[origin=c]{90}{$\Lambda$CDM}} & FCNN & 0.374 & 0.027 & -  \\
     \cline{2-5}
     & BBB & 0.986 & 0.168 & 0.495 \\
     & HMC & 0.316 & 11.979 & 0.444 \\
     \cline{2-5}
     & NLM + 2S & 0.449 & 10.569 & 0.43 \\
     & BBB + 2S & 0.363 & 12.672 & 0.495 \\
     & HMC + 2S & 0.31 & 12.054 & 0.424 \\
     \cline{2-5}
     & NLM + 2S + EB & $\mathbf{0.287}$ & 14.778 & $\mathbf{0.113}$ \\
     & BBB + 2S + EB & 0.375 & 10.695 & 0.155 \\
     & HMC + 2S + EB & 0.329 & 11.463 & 0.168 \\
     \hline
    
     \multirow{7}{*}{\rotatebox[origin=c]{90}{CPL}} & FCNN & $\mathbf{0.033}$ & 0.0 & - \\
     \cline{2-5}
     & BBB & 0.373 & 0.013 & 0.357 \\
     & HMC & 0.082 & 0.012 & 0.152 \\
     \cline{2-5}
     & BBB + 2S & 0.59 & 0.001 & 0.495 \\
     & HMC + 2S & 0.275 & 0.001 & $\mathbf{0.072}$ \\
     \cline{2-5}
     & BBB + 2S + EB & 0.958 & 0.001 & 0.216 \\
     & HMC + 2S + EB & 0.648 & 0.001 & 0.194 \\
     \hline
    
     \multirow{6}{*}{\rotatebox[origin=c]{90}{Quint.}} & FCNN & 0.022 & 0.016 & -  \\
     \cline{2-5}
     & BBB & 0.125 & 0.089 & 0.349 \\
     & HMC & $\mathbf{0.006}$ & 0.011 & $\mathbf{0.021}$ \\
     \cline{2-5}
     & NLM + 2S & 0.138 & 0.335 & 0.334 \\
     & BBB + 2S & 0.221 & 0.13 & 0.377 \\
     & HMC + 2S & 0.038 & 0.03 & 0.096 \\
     \hline
    
     \multirow{6}{*}{\rotatebox[origin=c]{90}{HS}} & FCNN & 0.374 & 0.027 & -  \\
     \cline{2-5}
     & BBB & 0.374 & 0.161 & 0.427 \\
     & HMC & 0.257 & 0.121 & 0.43 \\
     \cline{2-5}
     & NLM + 2S & 6.999 & 0.592 & $\mathbf{0.412}$ \\
     & BBB + 2S & 2.124 & 1.616 & 0.452 \\
     & HMC + 2S & 0.885 & 0.984 & 0.487 \\
    \end{tabular}
    \end{adjustbox}
\end{table}

In \cref{tab:train,tab:ood}, we present results separately for the training and OOD regions. In the training region, the two-step approach consistently improves both accuracy and miscalibration, and the addition of error bounds further enhances calibration across most equations. For example, in the $\Lambda$CDM case, combining two steps with error bounds yields the lowest miscalibration area among all methods, indicating effective uncertainty correction when the model operates within its training distribution.

In the OOD region, results are more mixed. For $\Lambda$CDM, the combination of two steps and error bounds significantly reduces miscalibration. However, in the CPL model, applying two steps increases miscalibration for BBB but reduces it for HMC; adding error bounds reverses this behavior—improving BBB but slightly worsening HMC. As shown in \cref{fig:cpl_examples_eb}, this is likely because FCNN closely tracks the ground truth in the OOD region, leading to narrow error bounds, while BBB and HMC tend to drift from the true function, which error bounds fail to capture accurately. For the more complex equations, Quintessence and HS, the benefits of the two-step method are less consistent, and no clear advantage is observed over the baseline approaches.

In \cref{tab:inverse} we can see how the Bayesian methods provide tighter intervals than deterministic nets. This is due to observations being more likely under the distributions generated by Bayesian methods. These uncertainties can be propagated through the approximated solution to obtain solution distributions that account for the uncertainties in the equation parameters.

\cref{tab:sigmas} shows the level of agreement of our results with the literature. Most of our results are in an agreement of $1\sigma$ or $2\sigma$. $\Lambda$CDM is an exception, where we observe higher disagreement due to the tight bounds obtained and discrepancy in the means.

Overall, our results demonstrate that leveraging error bounds within a two-step framework significantly enhances uncertainty quantification, particularly in calibration. However, the observed tradeoff between calibration and numerical stability suggests the need for further refinement, especially for highly complex equations such as HS.

Among Bayesian approaches, HMC generally achieves superior accuracy and calibration, but its computational cost is substantial. BBB, on the other hand, performs well when combined with our two-step procedure, offering a more computationally efficient and flexible alternative. There is considerable potential for improving BBB by adopting more expressive variational posterior distributions. Finally, NLM also demonstrates strong performance while being the least expensive method, though its applicability is limited by the underlying distributional assumptions.

% \begin{contributions} % will be removed in pdf for initial submission 
% 					  % (without ‘accepted’ option in \documentclass)
%                       % so you can already fill it to test with the
%                       % ‘accepted’ class option
%     Briefly list author contributions. 
%     This is a nice way of making clear who did what and to give proper credit.
%     This section is optional.

%     H.~Q.~Bovik conceived the idea and wrote the paper.
%     Coauthor One created the code.
%     Coauthor Two created the figures.
% \end{contributions}

% \begin{acknowledgements} % will be removed in pdf for initial submission,
% 						 % (without ‘accepted’ option in \documentclass)
%                          % so you can already fill it to test with the
%                          % ‘accepted’ class option
%     Briefly acknowledge people and organizations here.

%     \emph{All} acknowledgements go in this section.
% \end{acknowledgements}

% References
\bibliography{references}

\newpage

\onecolumn

\title{Supplementary Material}
\maketitle

\appendix
\section{Enforcing IC/BC}\label{sup:enforce}
\subsection{Two-dimensional Dirichlet Boundary Value Problem}
In the Dirichlet problem, also known as \textit{first boundary value problem}, the values of the solution function at the boundaries $\Gamma$ are known \citep{boyce_elementary_2008}. In this case the re-parameterization is
\begin{equation*}
    \tilde u_{\theta}(x,y):=A(x,y)
        +\tilde{x}\big(1-\tilde{x}\big)\tilde{y}\big(1-\tilde{y}\big)\u(x,y)
\end{equation*}
where
\begin{align*}
        A(x,y)&=\big(1-\tilde{x}\big)u(x_0, y)+\tilde{x}u(x_1, y) \\
        &+\big(1-\tilde{y}\big)\Big(u(x, y_0)-\big(1-\tilde{x}\big)u(x_0, y_0)+\tilde{x}u(x_1, y_0)\Big) \\
        &+\tilde{y}\Big(u(x, y_1)-\big(1-\tilde{x}\big)u(x_0, y_1)+\tilde{x}u(x_1, y_1)\Big)\\
        \tilde{x}&=\frac{x-x_0}{x_1-x_0}\\
        \tilde{y}&=\frac{y-y_0}{y_1-y_0}\\
        (x_0, y), &(x_1, y), (x, y_0), (x, y_1) \in \Gamma
\end{align*}

\section{Error Bounds Computation}\label{sup:errorbounds}

\subsection{First Order Linear ODE with Constant Coefficient}
A first-order linear ODE with constant coefficient has the general form
\begin{equation}
    \label{eq:folcc}
    u'(t)+(\lambda+i\omega)u(t)=f(t)
\end{equation}

As shown by \cite{liu_evaluating_2022}, the error of a PINN solution $\u$ to \cref{eq:folcc} can be bounded as
\begin{equation}
    |\u(t) - u(t)|\leq\varepsilon e^{-\lambda t}\int_{t_0}^{t}e^{\lambda\tau}d\tau
\end{equation}
when the initial conditions are satisfied i.e $\u(t_0)=u(t_0)$. Here, $\varepsilon$ is an upper bound on the residuals
\begin{equation}
    \label{eq:eb_res_bound}
    |u'(t)+(\lambda+i\omega)u(t)-f(t)|\leq \varepsilon \quad \forall t \in I
\end{equation}
where $I$ can be any of the forms $(t_0, t), (t_0, t], (t_0, \infty)$.

\begin{algorithm}
    \caption{Tight Error Bounds}\label{alg:eb}
    \begin{algorithmic}
        \Require Domain $I=[T_{min}, T_{max}]$, residual network $r(t)$, an expression for $e^{P(t)}$ we will call EP$(t)$, an expression for $\int_{T_{min}}^te^{P(\tau)}d\tau$ we will call IntEP$(t)$, number of partitions $N$, number of points in each partition $K$.
        \Ensure A set of times and error bounds at those times $\{t_i, b_i\}_{i=1}^N$.
        \State $\{t_i\}_{i=0}^N \gets \text{Linspace}(T_{min}, T_{max}, N + 1)$
        \State Initialize $b_0 := 0$
        \For{$i \gets 1\dots N$}
        \State $I_i \gets \text{Linspace}(t_{i-1}, t_{i}, K)$
        \State $\varepsilon_i \gets \max_{\tau\in I_{i}} |r(\tau)|$
        \State $b_i \gets b_{i-1} + \varepsilon_{i} \frac{\text{IntEP}(t_{i}) - \text{IntEP}(t_{i-1})}{\text{EP}(t_{i})}$ \Comment{Implementation of \cref{eq:bound_computation}}
        \EndFor
        \State \hspace{-1.2em} \textbf{Note:} Linspace($a,b,n$) returns an array of $n$ points equally spaced in the range $[a, b]$.
    \end{algorithmic}
\end{algorithm}

\subsection{First Order Linear ODE with Nonconstant Coefficient}
A first-order linear ODE with nonconstant coefficient has the general form
\begin{equation}
    \label{eq:folncc}
    u'(t)+(p(t)+iq(t))u(t)=f(t)
\end{equation}
As shown by \cite{liu_evaluating_2022}, when the initial conditions are satisfied, the error of a PINN solution $\u$ to \cref{eq:folncc} can be bounded as
\begin{equation}
    |\u(t) - u(t)|\leq\varepsilon e^{-P(t)}\int_{t_0}^{t}e^{P(\tau)}d\tau
\end{equation}
Where $P(t)=\int_{t_0}^{t} p(\tau)d\tau$, and $\varepsilon$ is an upper bound on the residuals
\begin{equation}
    \label{eq:eb_res_bound}
    |u'(t)+(\lambda+i\omega)u(t)-f(t)|\leq \varepsilon \quad \forall t \in I
\end{equation}
where $I$ can be any of the forms $(t_0, t), (t_0, t], (t_0, \infty)$.

\subsection{Tight Bounds Computation}\label{sup:tight_bounds}
In \cref{sec:eb}, we described how to obtain a bound using the global maximum residual $\varepsilon$ in $I$. However, we can compute a tighter error bound by partitioning the domain $I=I_1 \uplus I_2 \uplus \dots \uplus I_n$.

%The tight bound for the first-order Linear ODE with constant coefficients turns out to be
%\begin{equation}
%    |\u(t) - u(t)|\leq e^{-\lambda t}\int_{t_0}^{t}|r_{\theta}(\tau)|e^{\lambda\tau}  d\tau
%\end{equation}
%Similarly, for the first-order ODE with nonconstant coefficients, the bound is
The tight bound for the first-order ODE with nonconstant coefficients turns out to be
\begin{equation}
    |\u(t) - u(t)|\leq e^{-P(t)}\int_{t_0}^{t}|r_{\theta}(\tau)|e^{P(\tau)}d\tau
\end{equation}
To use the partitions $I_k$, we define the maximum local residual
\begin{equation}
    \varepsilon_k := \max_{\tau \in I_k} |r_{\theta}(\tau)|
\end{equation}
and we compute the bounds as
\begin{equation}
\label{eq:bound_computation}
|\u(t) - u(t)| \leq \sum_{i=1}^n \varepsilon_i e^{-P(t)}\int_{\tau=t_{i-1}}^{\tau=t_i} e^{P(\tau)}d\tau
\end{equation}
where $t_k = \max I_k$ and $t_n = t$. \cref{alg:eb} in the Supplementary Material shows the implementation of tight bounds computation.

\section{NLM Details}\label{sup:nlm_details}
We now provide details on the NLM derivation presented in \cref{sec:post_aprox}.
\begin{align}
    \Sigma_{\text{post}} &= \left(\Phi_{\theta}^T\Sigma\Phi_{\theta} + \sigma_{\text{prior}}^{-2}I\right)^{-1}\label{eq:nlm_sigma_post} \\
    \mu_{\text{post}} &= \Sigma_{\text{post}}\left(\Sigma^{-1}\Phi_{\theta}\right)^Tu_{\theta_{\text{det}}}\label{eq:nlm_mean_post}
\end{align}
we have writen $\Phi_{\theta}^T$ and $\Sigma$ instead of  $\Phi_{\theta}^T(\x_{\mathcal{D}}, \lambda_{\mathcal{D}})$ and $\Sigma(\x_{\mathcal{D}}, \lambda_{\mathcal{D}})$ respectively, to simplify notation. In \cref{eq:nlm_sigma_post} and \cref{eq:nlm_mean_post}, 

$$\x_{\mathcal{D}}=[\x_1,...,\x_N]^T$$
$$\lambda_{\mathcal{D}}=[\lambda_1,...,\lambda_N]^T$$
$$\Sigma(\x_{\mathcal{D}}, \lambda_{\mathcal{D}})=\mathrm{diag}([\slike^2(\x_1, \lambda_1),...,\slike^2(\x_M, \lambda_M)])$$

\section{Computing Predictive Uncertainty}
The computation of predictive uncertainty is straightforward as it is the standard deviation of the posterior predictive distribution. We can obtain it by applying the law of total variance
\begin{align}
        \text{Var}(u | \x, \lambda, \mathcal{D}) &= \mathbb{E}_{\theta | \mathcal{D}}\left[\text{Var}(u|\x, \lambda, \theta)\right] + \text{Var}_{\theta | \mathcal{D}}\left[\mathbb{E}(u|\x, \lambda, \theta)\right]\\
        &= \mathbb{E}_{\theta | \mathcal{D}}\left[\slike^2(\x,\lambda)\right]+ \text{Var}_{\theta|\mathcal{D}}[u_{\theta}(\x,\lambda)]\nonumber\\
        &\approx \slike^2(\x,\lambda) + \frac{1}{M}\sum_{i=1}^M(u_{\theta_i}(\x,\lambda) - \overline{u(\x, \lambda)})^2\nonumber
\end{align}
where $\theta_i \sim p(\theta| \mathcal{D})$ and $\overline{u(\x, \lambda)} = \frac{1}{M}\sum_{i=1}^M u_{\theta_i}(\x,\lambda)$ is the sample mean of the network  outputs.

 The latter approximation applies to BBB and NUTS; for BBB, the samples are taken from the variational posterior, and for NUTS, we use the samples of the true posterior the method yields. We have the analytical expression for NLM as shown in \cref{eq:nlm_sigma_post}.

\section{Statistical Analysis}
It may also be the case that the observed values are a function $f$ of the solution. In that scenario, the dataset would become $\mathcal{O}^f = \{(\mathbf{x}_i, \mu_i^f, \sigma_i^f) \mid \mathbf{x}_i \in \Omega, \mu_i^f \in \mathbb{R}, \sigma_i^f \in \mathbb{R}^+ \}_{i=1}^{O^f}$, where $\mu_i^f$ and $\sigma_i^f$ are the mean and standard deviation of the observed function, respectively. Here $f$ is a function of the true solution and possibly other parameters $\lambda^f$, which can also be included in the set of parameters to be estimated. In a similar way to \cref{eq:marginalization}, we can compute the likelihood:
\begin{equation}
\label{eq:inverse2}
        p(\x_i, \mu_i, \sigma_i|\lambda, \lambda^f)=\int_{\mathcal{U}}p(\x_i, \mu_i^f, \sigma_i^f|f(u_{\lambda}(\x_i), \lambda^f))\cdot p(u_{\lambda}(\x_i)|\x_i, \lambda)du
\end{equation}

\section{Proofs}

\begin{proposition}
    \label{prop:bounds}
    Given a differential equation of the form $\diff{u}{t} = f(t)$ for some function $f$ and an initial condition $u(t_0) = u_0$. The error of an approximation $\hat u$ of $u$ is:
    \begin{equation}
        \hat u(\tau) - u(\tau) = \int_{t_0}^{\tau}R(t)dt
    \end{equation}
    if the approximation satisfies the ICs, where $R(t) := \diff{\hat u}{t} - f(t)$.
\end{proposition}
\begin{proof}
    Let us take some approximate solution to the DE
    \begin{equation*}
        \hat u(t) = u(t) + e_{u}(t)
    \end{equation*}
    where $\hat u(t)$ is the approximation, $u(t)$ is the true solution and $e_{u}(t)$ is an error term. We can do the same for the first derivative of the solution. Using Newton's notation:
    \begin{equation}
        \label{eq:p1_diff}
        \hat {\dot u}(t) = \dot u(t) + e_{\dot u}(t)
    \end{equation}
    Using the definition of $R(t)$
    \begin{equation*}
       R(t) = \dot u(t) + e_{\dot u}(t) - f(t)
    \end{equation*}
    Since $\dot u(t)$ is the derivative of the true solution, $\dot u(t)- f(t)=0$, thus $R(t) = e_{\dot u}(t)$. Substituting that in \cref{eq:p1_diff}, we get
    \begin{equation*}
        R(t) = \hat {\dot u}(t) - \dot u(t)
    \end{equation*}
    Integrating both sides with respect to $t$:
    \begin{align*}
        \int_{t_0}^{\tau} R(t)dt &= \int_{t_0}^{\tau} \hat {\dot u}(t) dt- \int_{t_0}^{\tau}\dot u(t)dt\\
        \int_{t_0}^{\tau} R(t)dt &= \hat u(\tau) - \hat u(t_0) +  u(t_0) - u(\tau)\\
        \int_{t_0}^{\tau} R(t)dt & + e_{u}(t_0) = \hat u(\tau) - u(\tau)
    \end{align*}
    If $u(t_0) = \hat u(t_0)$, which is the case for our PINNs because we enforce ICs, then
    \begin{equation*}
         \hat u(\tau) - u(\tau) = \int_{t_0}^{\tau} R(t)dt
    \end{equation*}    
\end{proof}

\section{Cosmology Models}\label{sup:cosmology}

\subsection{$\Lambda$CDM Error Bounds Computation}
\paragraph{Error Bounds Computation}
Since this model is a first-order linear system with non-constant coefficients, it is possible to compute its error bounds. To compute the error bounds with \cref{alg:eb} we need an expression for $e^{P(z)}$ and $\int_0^ze^{P(s)}ds$. In the $\Lambda$CDM model, we have $p(z) = 3/(1+z)$, thus
\begin{equation}
    P(z) = \int_0^z -\frac{3}{(1+s)}ds = -3\text{ln}(1+s)
\end{equation}
then
\begin{gather}
    e^{P(s)} = e^{-3\text{ln}(1+s)}\\
    \int_{0}^z e^{P(s)}ds = \int_{0}^z e^{-3\text{ln}(1+s)}ds = \frac{1}{2} - \frac{1}{2(1+z)^2}
\end{gather}
which are the inputs needed for \cref{alg:eb}.

\subsection{Parametric Dark Energy Error Bounds Computation} \label{sup:cpl_bounds}
Once again, we need to find an expression for the inputs to \cref{alg:eb}. CPL's equation is also a first-order linear with a non-constant coefficient equation, so the procedure is the same as for $\Lambda$CDM. First, we have
\begin{equation}
    p(z) = \frac{-3}{1+z}(1+w_0+\frac{w_1z}{1+z})
\end{equation}
\begin{align}
    P(z) &= \int_0^z \frac{-3}{1+s}(1+w_0+\frac{w_1s}{1+s})ds\\
    &= -3((w_0 + w_1+1)\text{ln}(z+1)+\frac{w_1}{z+1}-w_1)
\end{align}
we can then simplify the expression $e^{P(z)}$ as
\begin{align}
    e^{P(z)} &= e^{-3((w_0 + w_1+1)\text{ln}(z+1)+\frac{w_1}{z+1}-w_1)}\\
    &= -3((w_0 + w_1+1)\text{ln}(z+1)+\frac{w_1}{z+1}-w_1)\\
    &= (z+1)^{-3(w_0 + w_1 + 1)}e^{\frac{3w_1z}{z+1}}
\end{align}

Finally, the integral turns out to be
\begin{equation}
    \int_0^z e^{P(s)}ds = e^{3w_1}(s+1)^{1-3(w_0 +w_1+1)}E_{2-3(w_0+w_1+1)}\left(\frac{3w_1}{s+1}\right)\Bigg|_{s=0}^{s=z}
\end{equation}
Where $E_{n}(x)$ is the exponential integral defined as $E_{n}(x) = \int_1^{\infty}\frac{e^{-xt}}{t^n}dt$. We used computational tools to find the antiderivative of $e^{P(s)}$.

\subsection{Quintessence}
The Quintessence \citep{caldwell_cosmological_1998, armendariz-picon_dynamical_2000, copeland_exponential_1998, zlatev_quintessence_1999} system of equations is
\begin{align*}
    \frac{dx}{dN}&=-3x+\frac{\sqrt{6}}{2}\lambda y^2 + \frac{3}{2}x(1+x^2-y^2)\\
    \frac{dy}{dN}&=\frac{\sqrt{6}}{2}xy\lambda + \frac{3}{2}y(1+x^2-y^2)
\end{align*}
with the following initial conditions
\begin{align*}
    x_0&=0 \\
    y_0&=\sqrt{\frac{1-\Omega_{m,0}^{\Lambda}}{\Omega_{m,0}^{\Lambda}(1+z_0)^3+1-\Omega_{m,0}^{\Lambda}}}
\end{align*}

where $\lambda,\Omega_{m,0}^{\Lambda}$ are parameters. Given a solution to the system, the Hubble parameter is computed as    
\begin{align*} 
H(z)&=H_0^{\Lambda}\sqrt{\frac{\Omega_{m,0}^{\Lambda}(1+z)^3}{1-x^2-y^2}}
\end{align*}

Here $H_0^{\Lambda}$ is also a parameter.

\subsection{$f(R)$ Gravity}
The $f(R)$ gravity \citep{clifton_modified_2012, buchdahl_non-linear_1970} model, which we refer to as HS for the authors \cite{hu_models_2007} consists of a system of 5 equations.
\begin{align*}
    \frac{dx}{dz}&=\frac{1}{z+1}(-\Omega-2v+x+4y+xv+x^2) \\
    \frac{dy}{dz}&=\frac{-1}{z+1}(vx\Gamma-xy+4y-2yv) \\
    \frac{dv}{dz}&=\frac{-v}{z+1}(x\Gamma +4-2v) \\
    \frac{d\Omega}{dz}&=\frac{\Omega}{z+1}(-1+2v+x) \\
    \frac{dr}{dz}&=-\frac{r\Gamma x}{z+1}
\end{align*}

where 

\begin{align*}
    \Gamma(r) =\frac{(r+b)\left[(r+b)^2-2b\right]}{4br}
\end{align*}

The initial conditions are
\begin{align*}
    x_0 &=0 \\
    y_0 &= \frac{\Omega_{m,0}^{\Lambda}(1+z_0)^3+2(1-\Omega_{m,0}^{\Lambda})}{2\left[\Omega_{m,0}^{\Lambda}(1+z_0)^3+(1-\Omega_{m,0}^{\Lambda})\right]}\\
    v_0 &= \frac{\Omega_{m,0}^{\Lambda}(1+z_0)^3+4(1-\Omega_{m,0}^{\Lambda})}{2\left[\Omega_{m,0}^{\Lambda}(1+z_0)^3+(1-\Omega_{m,0}^{\Lambda})\right]}\\
    \Omega_0&=\frac{\Omega_{m,0}^{\Lambda}(1+z_0)^3}{\Omega_{m,0}^{\Lambda}(1+z_0)^3+(1-\Omega_{m,0}^{\Lambda})}\\
    r_0&=\frac{\Omega_{m,0}^{\Lambda}(1+z_0)^3+4(1-\Omega_{m,0}^{\Lambda})}{1-\Omega_{m,0}^{\Lambda}}
\end{align*}

where $b,\Omega_{m,0}^{\Lambda}$ are the system parameters. Given a solution to the system, the Hubble parameter is computed as

$$H(z)=H_0^{\Lambda}\sqrt{\frac{r}{2v}(1-\Omega_{m,0}^{\Lambda})}$$

Here $H_0^{\Lambda}$ is also a parameter.

\section{Implementation Details}\label{sec:imp_details}

As described in \cref{sec:experiments}, we followed the implementation from \citep{chantada_cosmological_2022} to train the FCNNs in the first step. Architecture and hyperparameter details are shown in \cref{tab:impl_details_fcnn,tab:impl_details_nlm,tab:impl_details_bbb,tab:impl_details_hmc}. NLM shares FCNN details since it builds on FCNN.

We used Pytorch \citep{paszke_pytorch_2019} for FCNNs, Pyro \citep{bingham_pyro_2018} for BNNs, and Neurodiffeq \citep{chen_neurodiffeq_2020} for aiding in PINNs training.

\begin{table}[htpb]
    \caption{Implementation Details For FCNN.}
    \label{tab:impl_details_fcnn}
    \centering
    \begin{tabularx}{\textwidth}{YYYYYYYY}
    \hline
    \hline
     Equation & Input Dim. & Output Dim. & Hidden Units & Activation & Iterations & Samples per Dim. in a Batch & Learning Rate \\
     \hline
    $\Lambda$CDM & 1 & 1 & (32, 32) & Tanh & 100,000 & 64 & 0.001 \\
    CPL & 1 & 2 & (32, 32) & Tanh & 100,000 & 128 & 0.001 \\
    Quintessence & 1 & 1 & (32, 32) & Tanh & 100,000 & 32 & 0.001 \\
    HS & 1 & 1 & (32, 32) & Tanh & 600,000 & 32 & 0.001 \\
    \hline
    \hline
    \end{tabularx}
\end{table}

\begin{table}[htpb]
    \caption{Implementation Details For NLM.}
    \label{tab:impl_details_nlm}
    \centering
    \begin{tabularx}{\textwidth}{YYY}
    \hline
    \hline
    Equation & Samples per Dim. & Likelihood Std. \\
    \hline
    $\Lambda$CDM & 100 & 0.1 \\
    Quintessence & 32 & 0.005 \\
    HS & 32 & 0.005 \\
    \hline
    \hline
    \end{tabularx}
\end{table}

\begin{table}
    \caption{Implementation Details For BBB.}
    \label{tab:impl_details_bbb}
    \centering
    \begin{tabularx}{\textwidth}{YYYYYYYYYY}
    \hline
    \hline
    Equation & Input Dim. & Output Dim. & Hidden Units & Activation & Iterations & Samples per Dim. & Prior Std. & Learning Rate & Likelihood Std. \\
    \hline
    $\Lambda$CDM & 1 & 1 & (32, 32) & Tanh & 20,000 & 64 & 1 & 0.001 & 0.1 \\
    CPL & 1 & 2 & (32, 32) & Tanh & 20,000 & 128 & 1 & 0.001 & 0.01 \\
    Quintessence & 1 & 1 & (32, 32) & Tanh & 20,000 & 32 & 1 & 0.001 & 0.005 \\
    HS & 1 & 1 & (32, 32) & Tanh & 20,000 & 32 & 1 & 0.001 & 0.005 \\
    \hline
    \hline
    \end{tabularx}
\end{table}

\begin{table}
    \centering
    \caption{Implementation Details For HMC.}
    \label{tab:impl_details_hmc}
    \begin{adjustbox}{max width=\linewidth}
    \begin{tabular}{ccccccccccc}
    \hline
    \hline
    Equation & Input Dim. & Output Dim. & Hidden Units & Activation & Posterior Samples & Tune Samples & Samples per Dim. & Prior Std. & Likelihood Std. \\
    \hline
$\Lambda$CDM & 1 & 1 & (32, 32) & Tanh & 10,000 & 1,000 & 32 & 1 & 0.1 \\
CPL & 1 & 2 & (32, 32) & Tanh & 10,000 & 1,000 & 128 & 1 & 0.01 \\
Quintessence & 1 & 1 & (32, 32) & Tanh & 10,000 & 1,000 & 32 & 1 & 0.005 \\
HS & 1 & 1 & (32, 32) & Tanh & 10,000 & 1,000 & 32 & 1 & 0.005 \\
    \hline
    \hline
    \end{tabular}
    \end{adjustbox}
\end{table}

\begin{table}[H]
\centering
\caption{Measurements Of The Hubble Parameter H Using The Cosmic Chronometers Technique}
\label{tab:z_H_values}
\begin{tabularx}{\columnwidth}{YYY}
$z$ & $H(z) \pm \sigma_H\left[\frac{\text{km/s}}{\text{Mpc}}\right]$ & Ref. \\
\hline
0.09 & 69 $\pm$ 12 & \multirow{9}{*}{\citep{simon_constraints_2005}} \\
0.17 & 83 $\pm$ 8 & \\
0.27 & 77 $\pm$ 14 & \\
0.4 & 95 $\pm$ 17 & \\
0.9 & 117 $\pm$ 23 & \\
1.3 & 168 $\pm$ 17 & \\
1.43 & 177 $\pm$ 18 & \\
1.53 & 140 $\pm$ 14 & \\
1.75 & 202 $\pm$ 40 & \\
\hline
0.48 & 97 $\pm$ 62 & \multirow{2}{*}{\citep{stern_cosmic_2010}} \\
0.88 & 90 $\pm$ 40 & \\
\hline
0.1791 & 75 $\pm$ 4 & \multirow{8}{*}{\citep{moresco_improved_2012}} \\
0.1993 & 75 $\pm$ 5 & \\
0.3519 & 83 $\pm$ 14 & \\
0.5929 & 104 $\pm$ 13 & \\
0.6797 & 92 $\pm$ 8 & \\
0.7812 & 105 $\pm$ 12 & \\
0.8754 & 125 $\pm$ 17 & \\
1.037 & 154 $\pm$ 20 & \\
\hline
0.07 & 69 $\pm$ 19.6 & \multirow{4}{*}{\citep{zhang_cong_four_2014}} \\
0.12 & 68.6 $\pm$ 26.2 & \\
0.2 & 72.9 $\pm$ 29.6 & \\
0.28 & 88.8 $\pm$ 36.6 & \\
\hline
1.363 & 160 $\pm$ 33.6 & \multirow{2}{*}{\citep{moresco_raising_2015}} \\
1.965 & 186.5 $\pm$ 50.4 & \\
\hline
0.3802 & 83 $\pm$ 13.5 & \multirow{5}{*}{\citep{moresco_6_2016}} \\
0.4004 & 77 $\pm$ 10.2 & \\
0.4247 & 87.1 $\pm$ 11.2 & \\
0.4497 & 92.8 $\pm$ 12.9 & \\
0.4783 & 80.9 $\pm$ 9 & \\
\bottomrule
\end{tabularx}
\end{table}

\section{Exploration of Solutions, Residuals and Errors}\label{sm:sol_dist_exp}
To analyze the behavior of residuals and errors in PINNs, we conducted a systematic data collection process. We trained N deterministic PINNs for I iterations to solve each of the cosmological models under study. Throughout training, we recorded solutions and residuals every 10 iterations, evaluating them on a fixed, equidistant set of T time steps. Detailed specifications for each cosmological model are provided in \cref{table:exploration_res_err}.

\begin{table}
    \centering
    \caption{Implementation Details of \cref{sm:sol_dist_exp}.}
    \label{table:exploration_res_err}
    \begin{adjustbox}{max width=\linewidth}
    \begin{tabular}{ccccc}
    \hline
    \hline
    Equation & Number of NN (N) & Iterations (I) & Time Steps (T) \\
    \hline
$\Lambda$CDM & 1000 & 1000 & 50 \\
CPL & 1000 & 1000 & 50 \\
Quintessence & 1500 & 1000 & 50 \\
HS & 200 & 5000 & 50 \\
    \hline
    \hline
    \end{tabular}
    \end{adjustbox}
\end{table}

\paragraph{Distribution of Solutions}
After collecting the data, we generated histograms of the solutions throughout training, as shown in \cref{fig:lcdm_sol_dist,fig:cpl_sol_dist,fig:quint_sol_dist_1,fig:quint_sol_dist_2,fig:hs_sol_dist_1,fig:hs_sol_dist_2,fig:hs_sol_dist_3,fig:hs_sol_dist_4,fig:hs_sol_dist_5}. For reference, we overlaid Gaussian distributions with the sample mean and variance in orange. The distribution of solutions varies across cosmological models, with $\Lambda$CDM and CPL being the closest to Gaussian, while HS exhibits noticeable bimodality.

\paragraph{Residuals and Errors Relationship}
Understanding the relationship between solution errors and residuals is crucial for assessing whether residuals can serve as a reliable source of calibrated uncertainties. Ideally, a perfect correlation between residuals and errors would indicate that residuals effectively capture uncertainty.

In \cref{fig:err_vs_res_lcdm,fig:err_vs_res_cpl,fig:err_vs_res_quint,fig:err_vs_res_hs}, we plot the absolute, normalized values of relative errors against their corresponding residuals for all cosmological models. The results indicate no clear correlation; in fact, the trend often shows higher errors associated with lower residuals. This behavior aligns with the implicit bias described by \cite{wang_respecting_2022}, where PINNs exhibit lower convergence rates near initial conditions.

\begin{figure}[H]
    \centering
 \includegraphics[width=\linewidth]{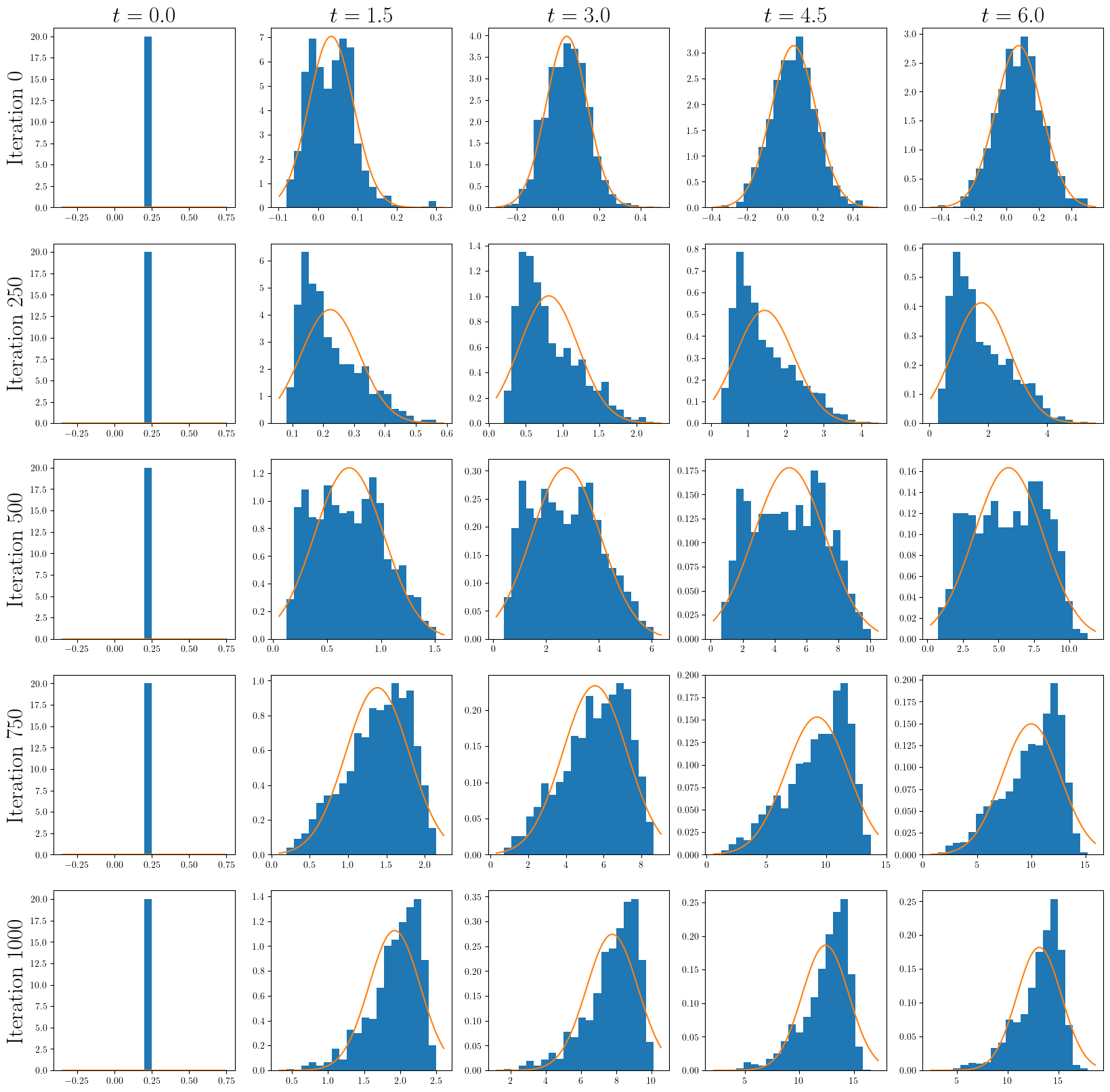}
 \caption{Distribution of $x_m(z)$ from $\Lambda$CDM. Samples were collected as described in \ref{sm:sol_dist_exp}. Orange lines show a Gaussian distribution with sample mean and variance.}
 \label{fig:lcdm_sol_dist}
\end{figure}

\begin{figure}[H]
    \centering
 \includegraphics[width=\linewidth]{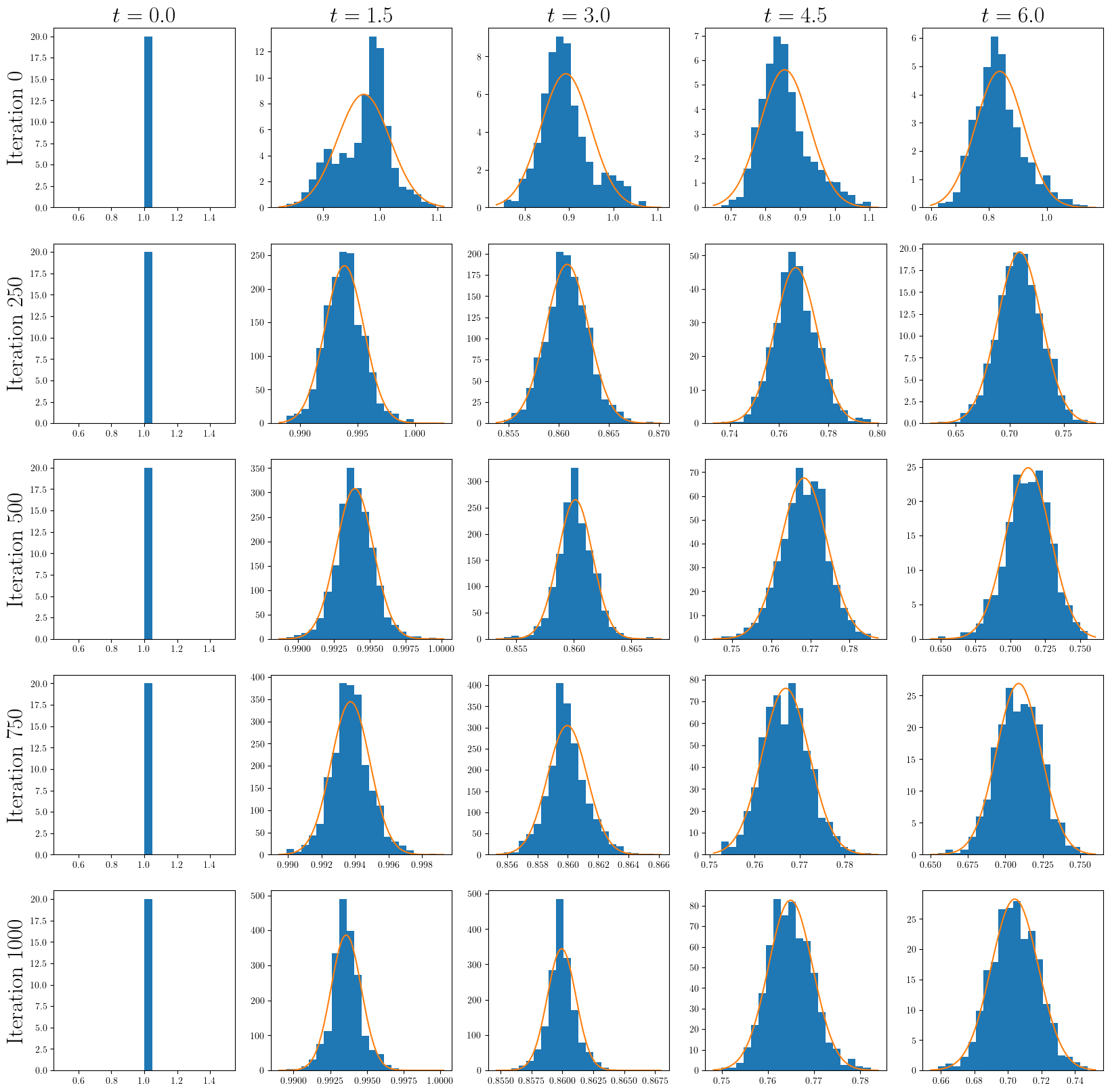}
 \caption{Distribution of $x_{\text{DE}}(z)$ from CPL. Samples were collected as described in \ref{sm:sol_dist_exp}. Orange lines show a Gaussian distribution with sample mean and variance.}
 \label{fig:cpl_sol_dist}
\end{figure}

\begin{figure}[H]
    \centering
 \includegraphics[width=\linewidth]{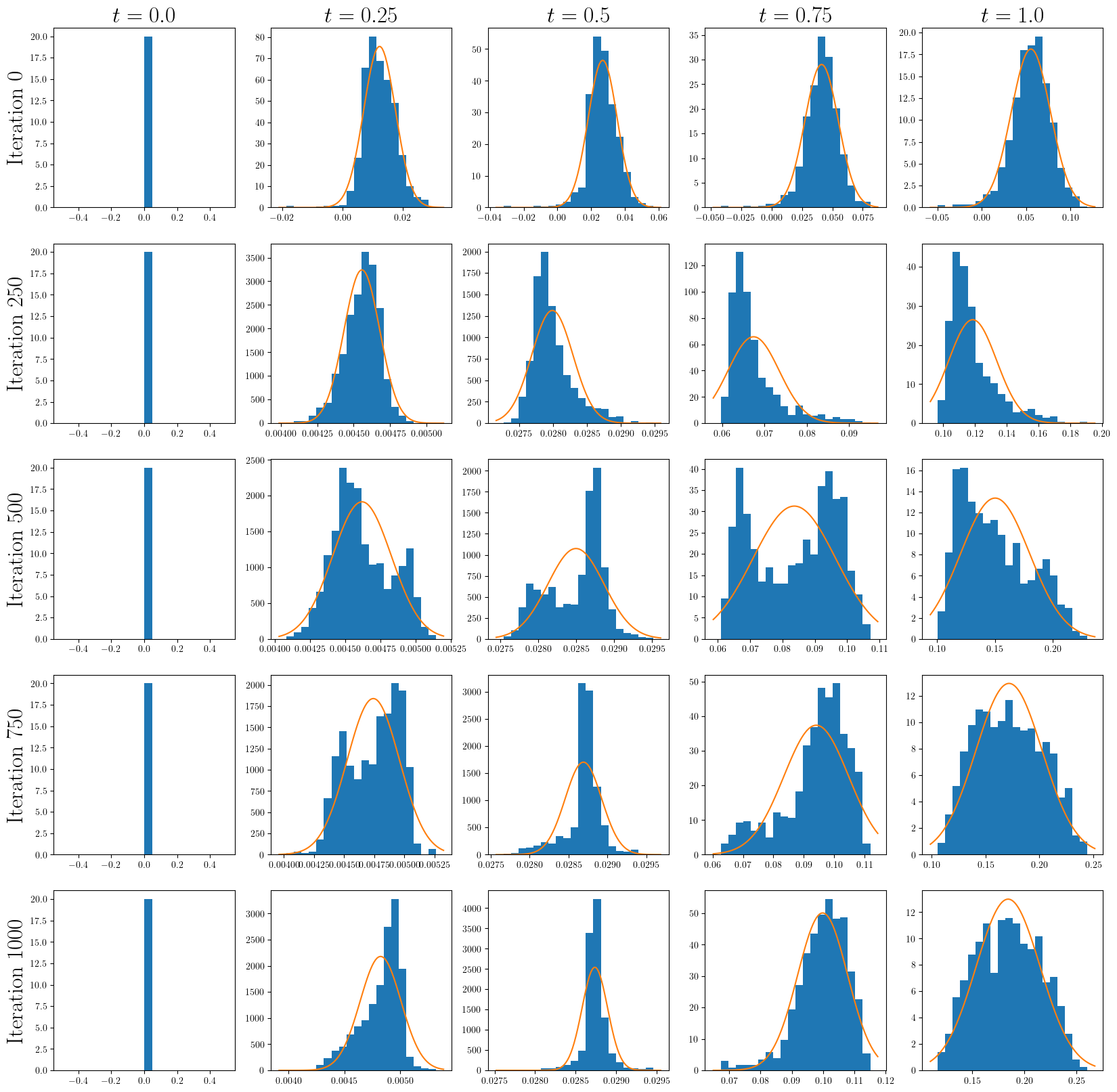}
 \caption{Distribution of $x(N)$ from Quintessence. Samples were collected as described in \ref{sm:sol_dist_exp}. Orange lines show a Gaussian distribution with sample mean and variance.}
 \label{fig:quint_sol_dist_1}
\end{figure}

\begin{figure}[H]
    \centering
 \includegraphics[width=\linewidth]{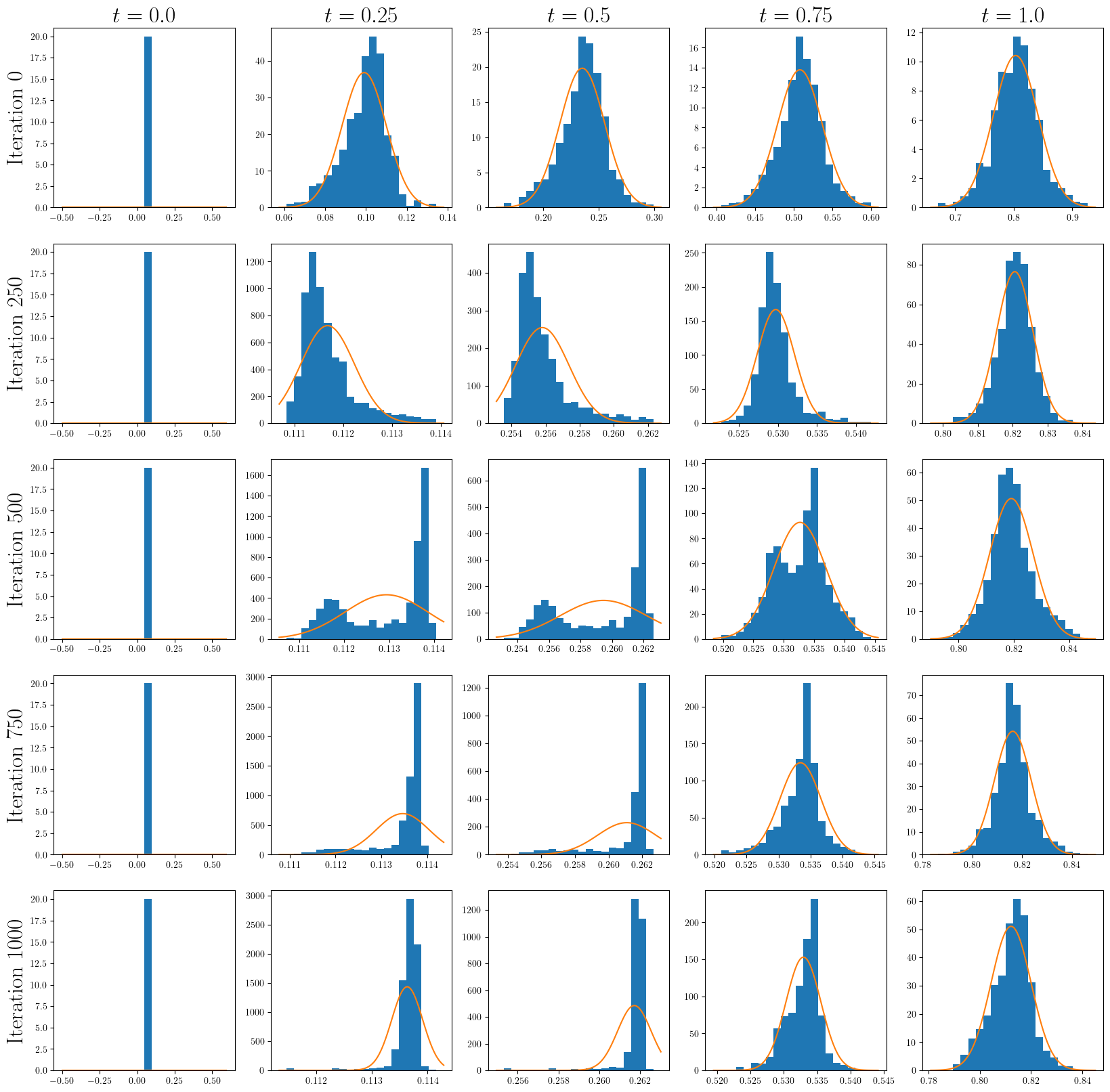}
 \caption{Distribution of $y(N)$ from Quintessence. Samples were collected as described in \ref{sm:sol_dist_exp}. Orange lines show a Gaussian distribution with sample mean and variance.}
 \label{fig:quint_sol_dist_2}
\end{figure}

\begin{figure}[H]
    \centering
 \includegraphics[width=\linewidth]{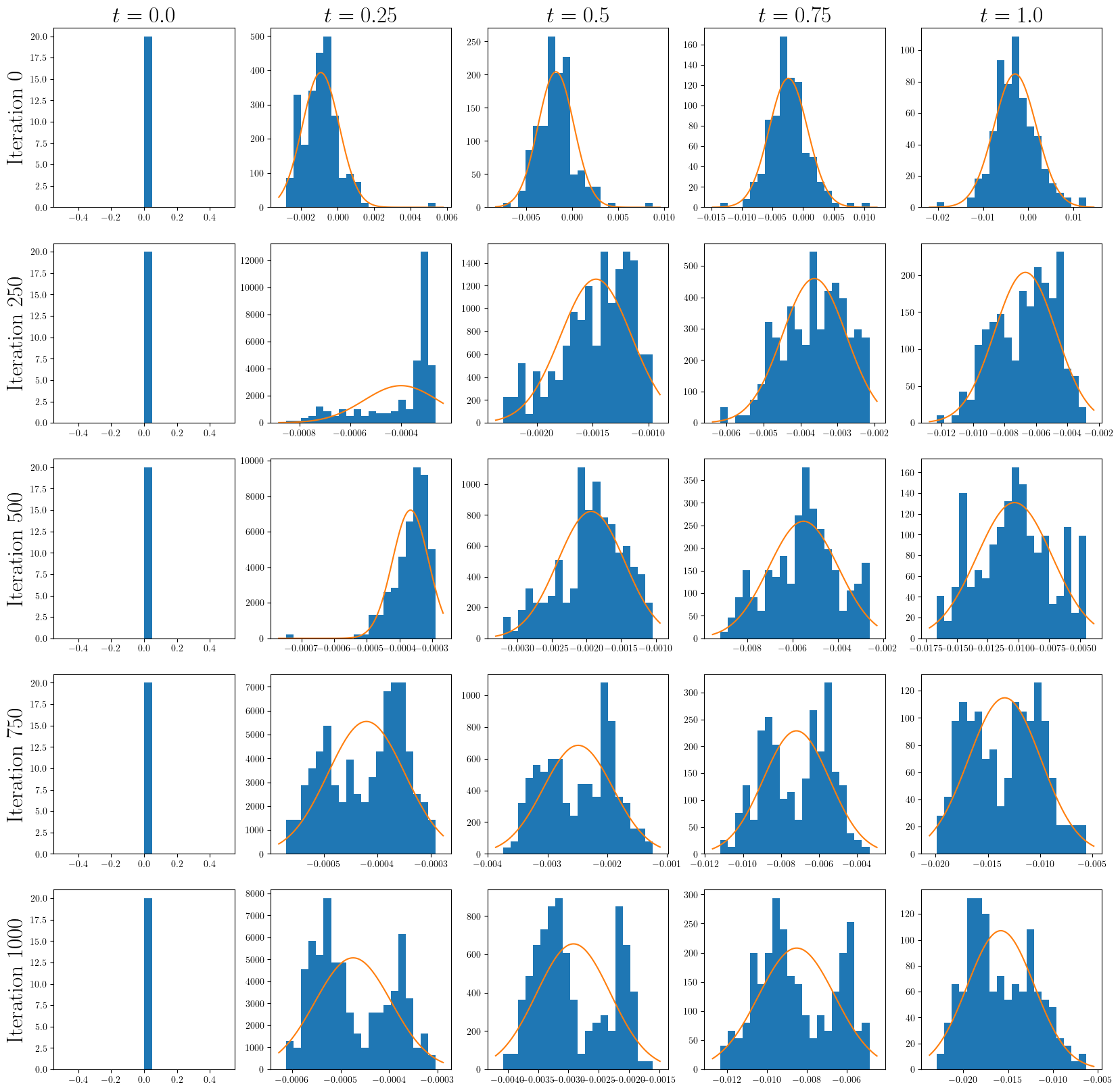}
 \caption{Distribution of $x(z)$ from HS. Samples were collected as described in \ref{sm:sol_dist_exp}. Orange lines show a Gaussian distribution with sample mean and variance.}
 \label{fig:hs_sol_dist_1}
\end{figure}

\begin{figure}[H]
    \centering
 \includegraphics[width=\linewidth]{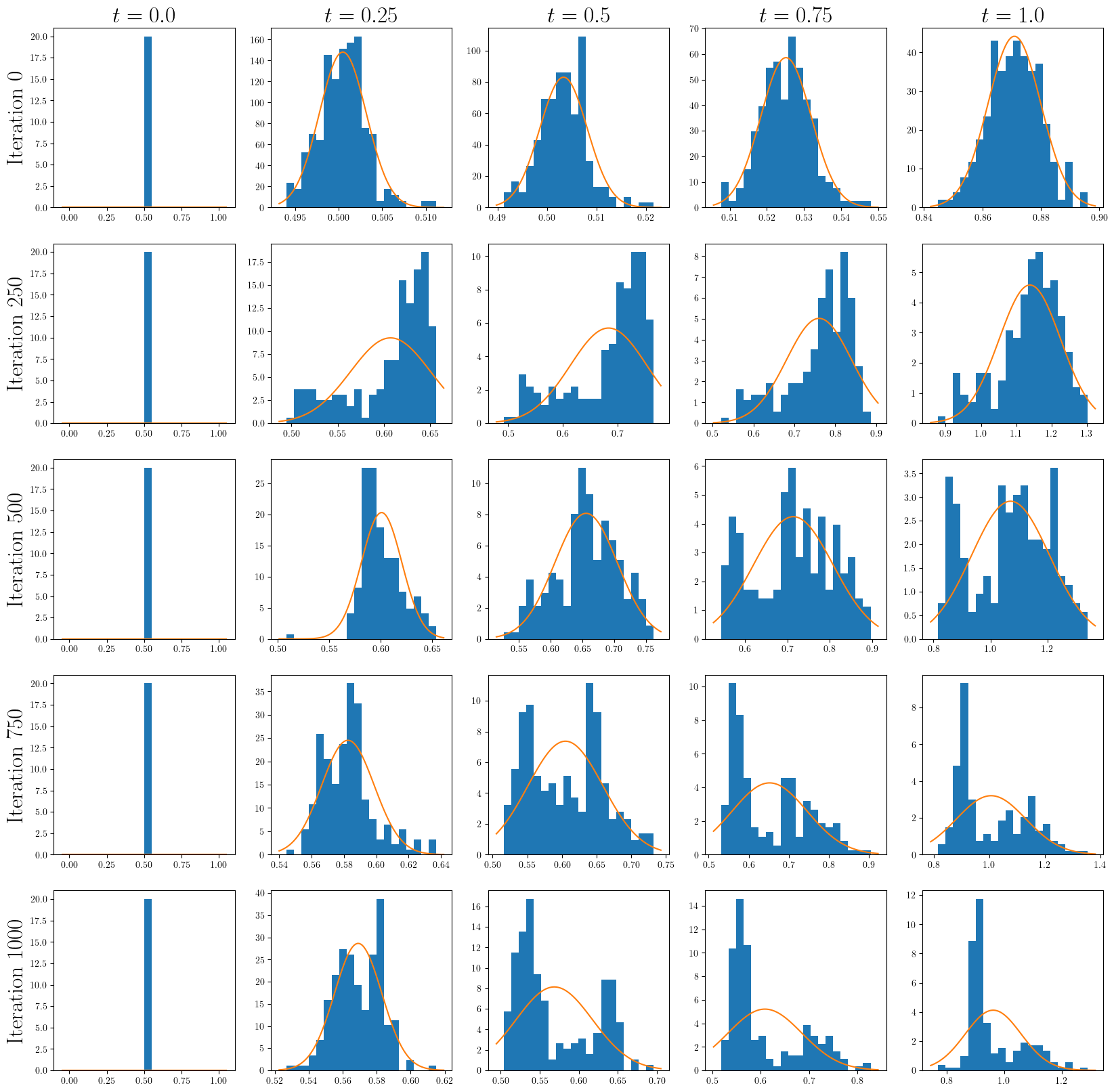}
 \caption{Distribution of $y(z)$ from HS. Samples were collected as described in \ref{sm:sol_dist_exp}. Orange lines show a Gaussian distribution with sample mean and variance.}
 \label{fig:hs_sol_dist_2}
\end{figure}

\begin{figure}[H]
    \centering
 \includegraphics[width=\linewidth]{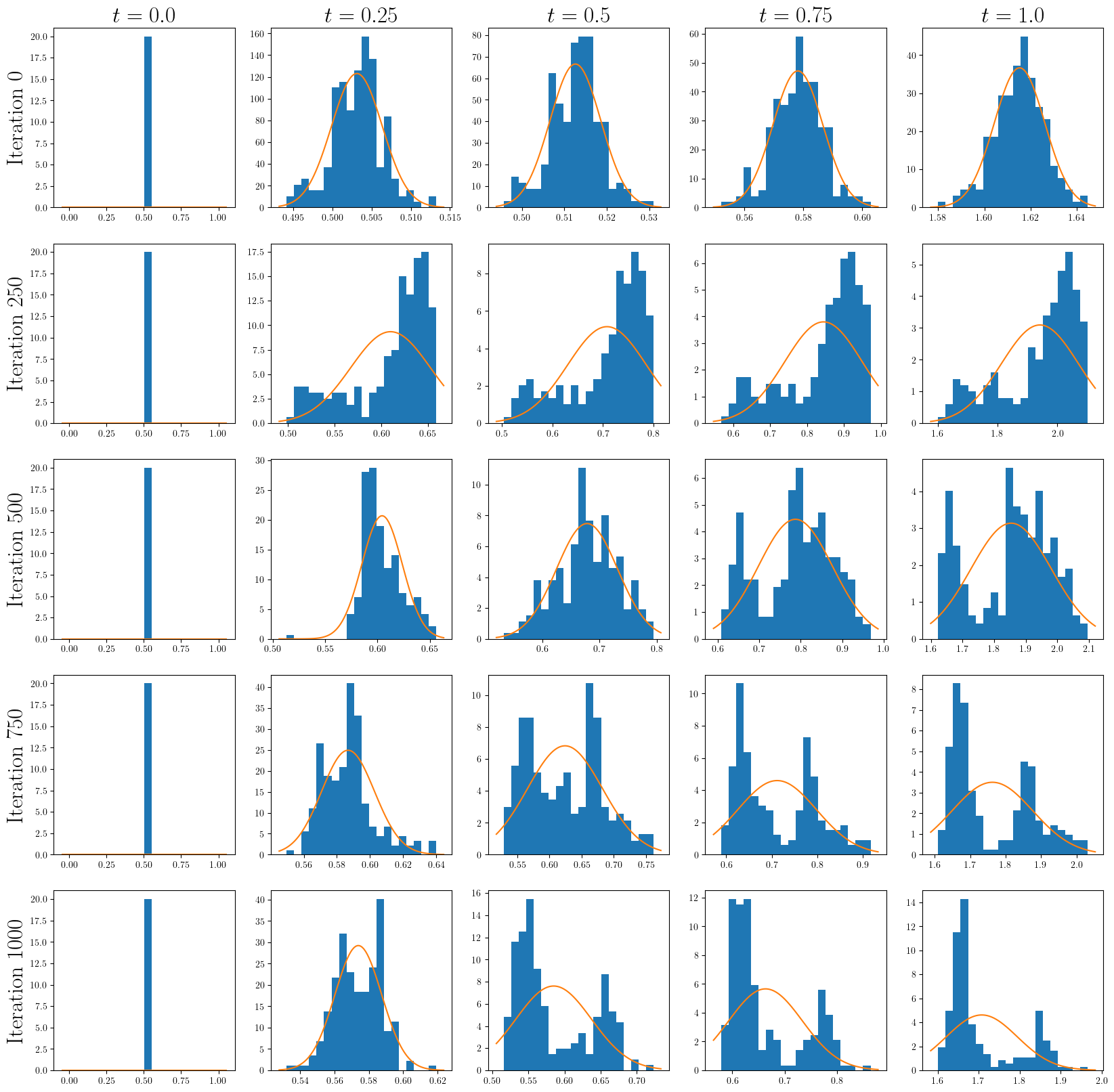}
 \caption{Distribution of $v(z)$ from HS. Samples were collected as described in \ref{sm:sol_dist_exp}. Orange lines show a Gaussian distribution with sample mean and variance.}
 \label{fig:hs_sol_dist_3}
\end{figure}

\begin{figure}[H]
    \centering
 \includegraphics[width=\linewidth]{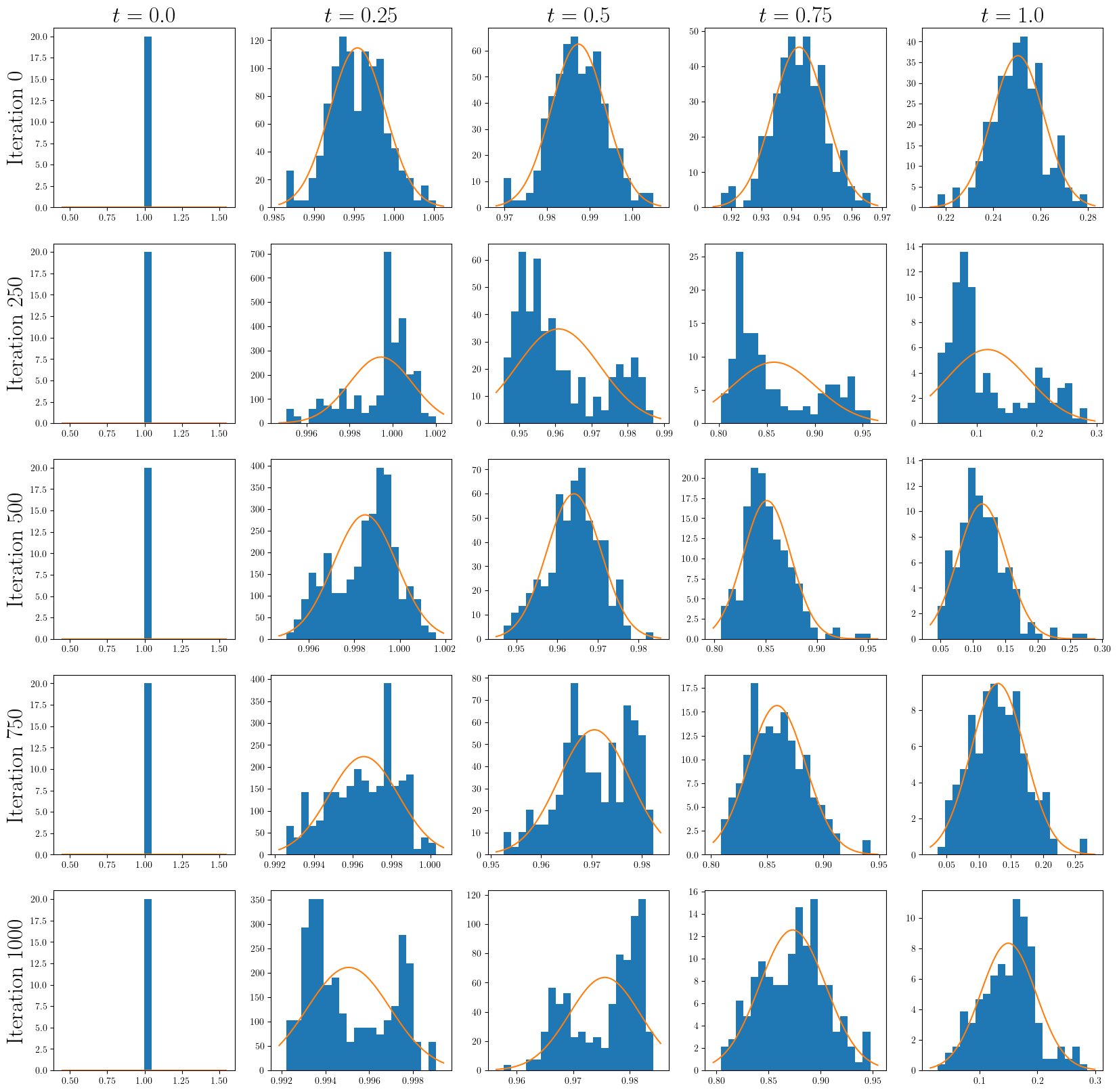}
 \caption{Distribution of $\Omega(z)$ from HS. Samples were collected as described in \ref{sm:sol_dist_exp}. Orange lines show a Gaussian distribution with sample mean and variance.}
 \label{fig:hs_sol_dist_4}
\end{figure}

\begin{figure}[H]
    \centering
 \includegraphics[width=\linewidth]{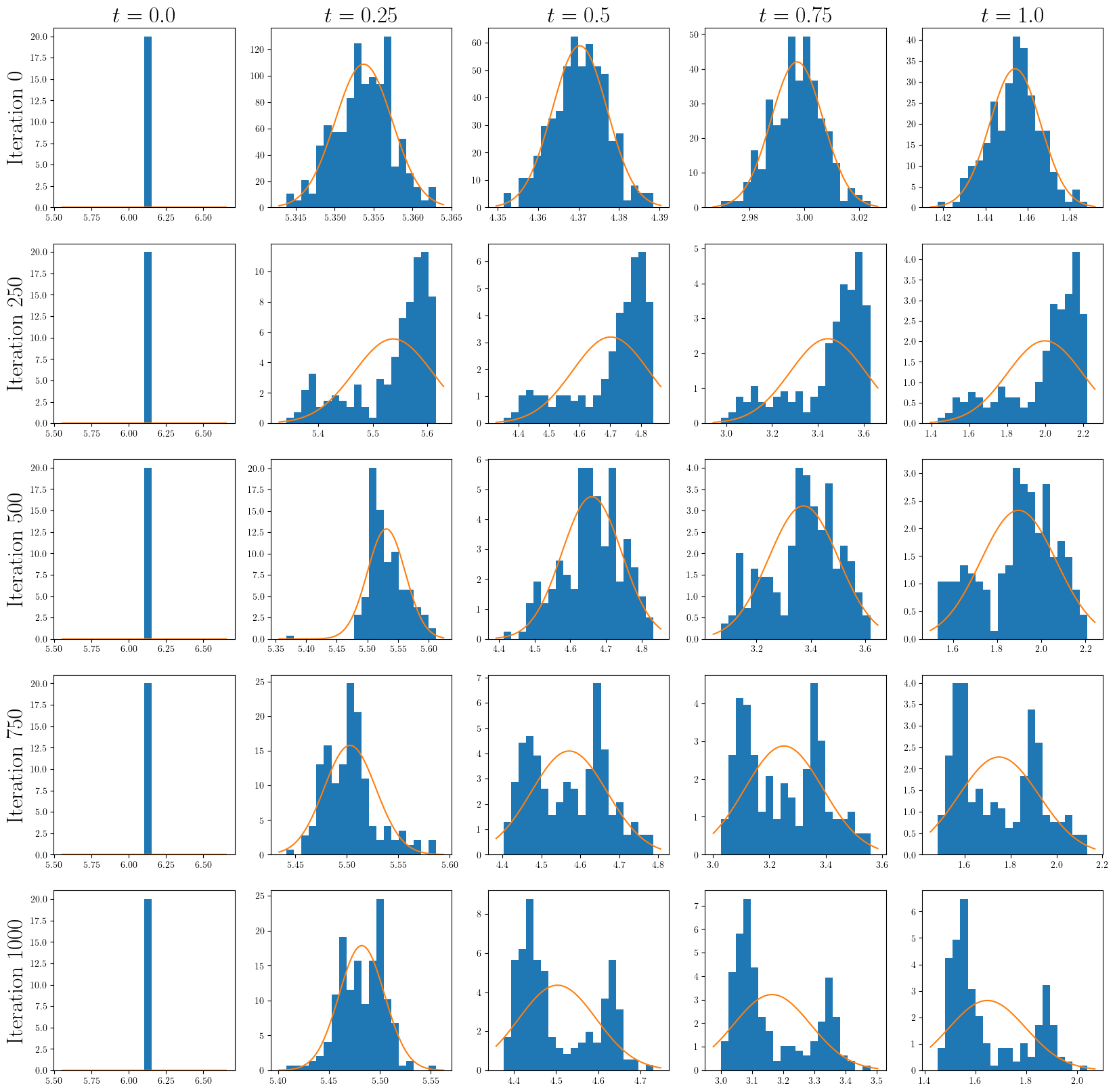}
 \caption{Distribution of $r(z)$ from HS. Samples were collected as described in \ref{sm:sol_dist_exp}. Orange lines show a Gaussian distribution with sample mean and variance.}
 \label{fig:hs_sol_dist_5}
\end{figure}

\begin{figure}[H]
    \centering
 \includegraphics[width=\linewidth]{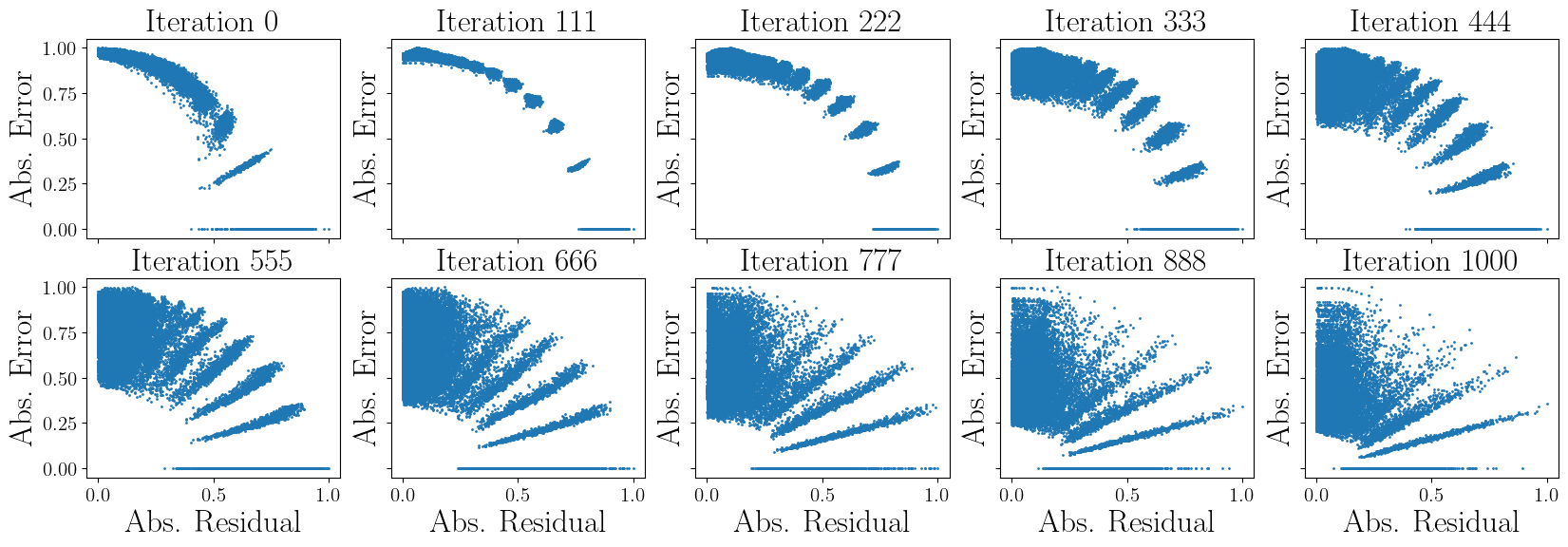}
 \caption{Relationship Between Solution Errors and Residuals from $\Lambda$CDM. Samples were collected as described in \ref{sm:sol_dist_exp}.}
 \label{fig:err_vs_res_lcdm}
\end{figure}

\begin{figure}[H]
    \centering
 \includegraphics[width=\linewidth]{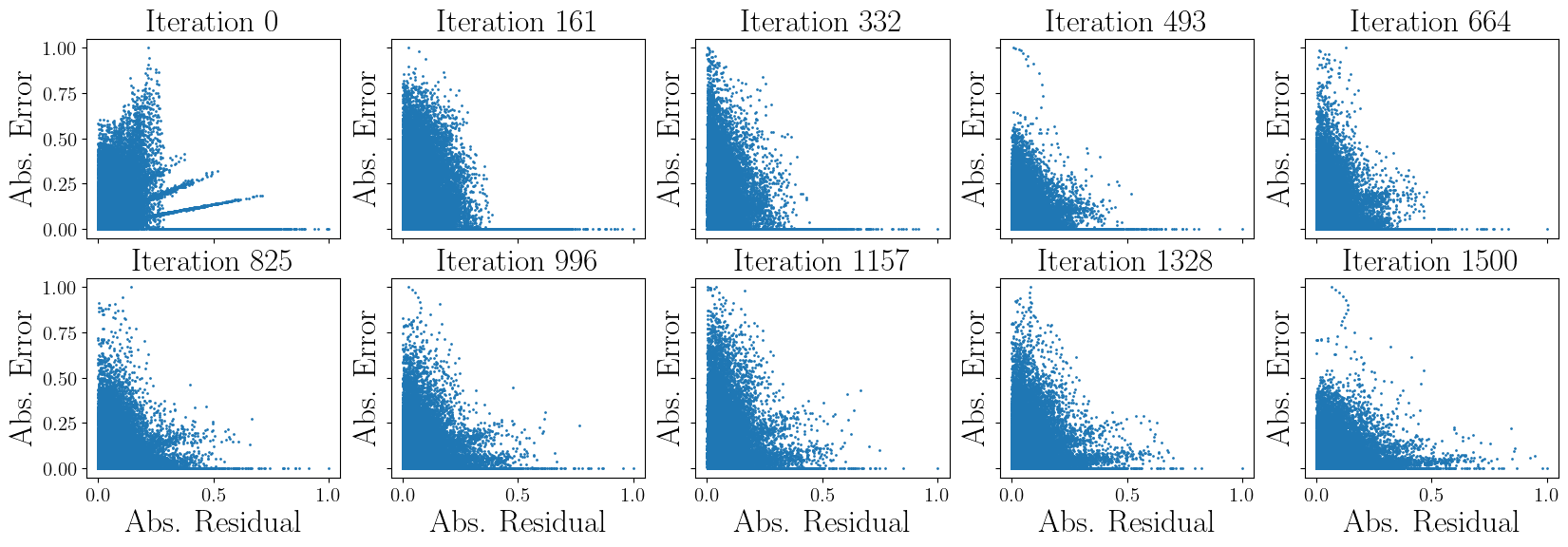}
 \caption{Relationship Between Solution Errors and Residuals from CPL. Samples were collected as described in \ref{sm:sol_dist_exp}.}
 \label{fig:err_vs_res_cpl}
\end{figure}

\begin{figure}[H]
    \centering
 \includegraphics[width=\linewidth]{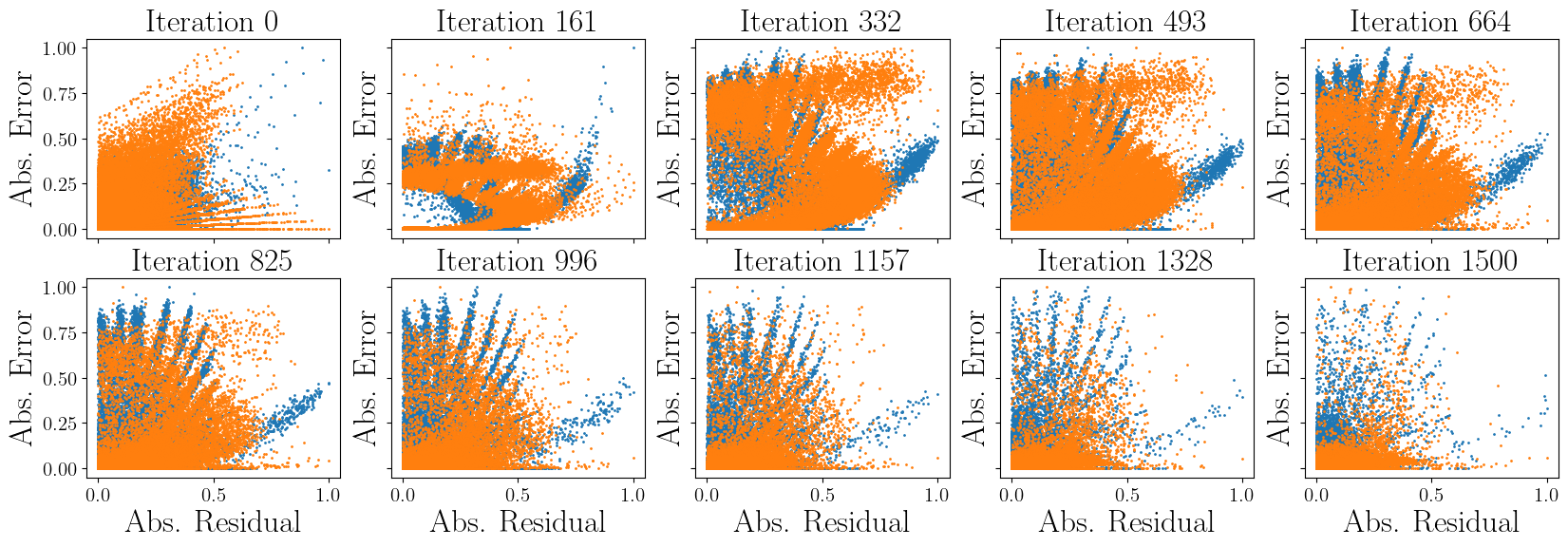}
 \caption{Relationship Between Solution Errors and Residuals from Quintessence. Samples were collected as described in \ref{sm:sol_dist_exp}. Colors differentiate variables in the system.}
 \label{fig:err_vs_res_quint}
\end{figure}

\begin{figure}[H]
    \centering
 \includegraphics[width=\linewidth]{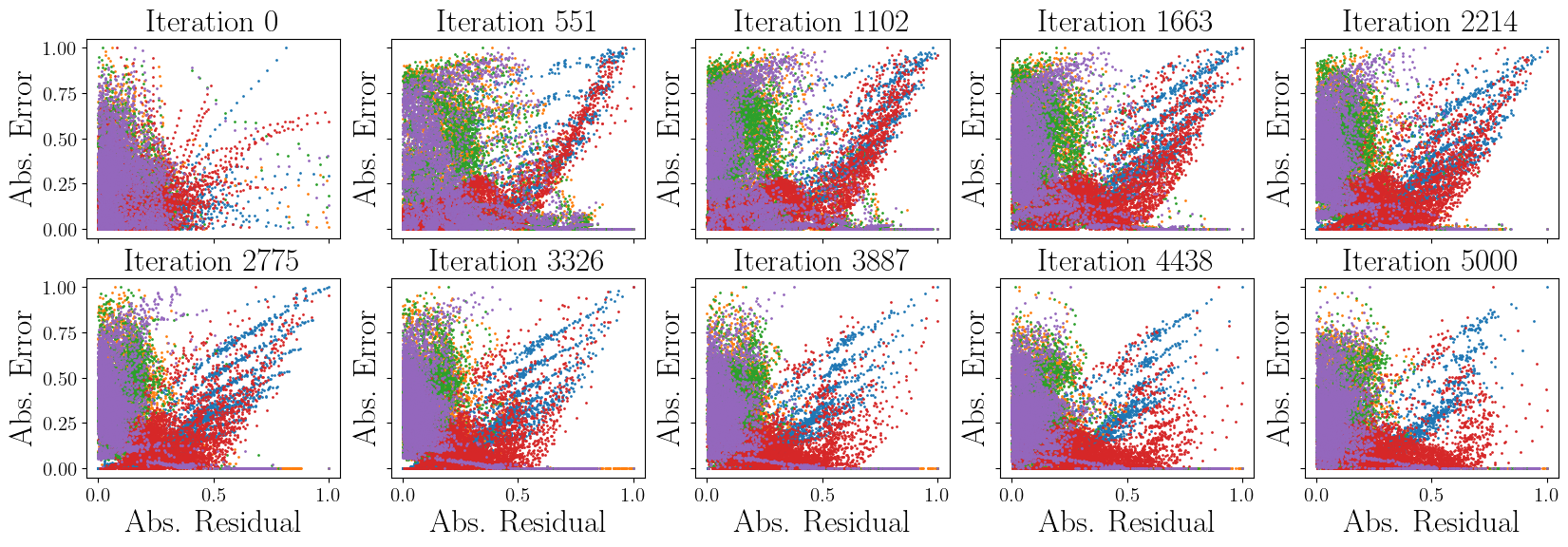}
 \caption{Relationship Between Solution Errors and Residuals from HS. Samples were collected as described in \ref{sm:sol_dist_exp}. Colors differentiate variables in the system.}
 \label{fig:err_vs_res_hs}
\end{figure}

\section{Aditional Results}\label{sec:add_results}

\begin{table}[h]
    \caption{Parameters Mean And Standard Deviation Estimation Of All Cosmology Models.}
    \label{tab:inverse}
    \centerline{
    \begin{adjustbox}{max width=\linewidth}
    \begin{tabular}{c c c c c c c c}
\hline
\hline
Equation & Method & $w_0$ & $w_1$ & $\lambda$ & $b$ & $\Omega_{m,0}$ & $H_0$ \\ 
\hline
 \multirow{9}{*}{\rotatebox[origin=c]{90}{$\Lambda$CDM}} & FCNN & - & - & - & - & 0.31 $\pm$ 0.05 & 68.5 $\pm$ 2.63 \\
 \cline{2-8}
 & BBB & - & - & - & - & 0.1 $\pm$ 0.01 & 79.95 $\pm$ 0.75 \\
 & HMC & - & - & - & - & 0.34 $\pm$ 0.01 & 68.23 $\pm$ 0.4 \\
 \cline{2-8}
 & NLM + 2S & - & - & - & - & 0.33 $\pm$ 0.01 & 67.5 $\pm$ 0.51 \\
 & BBB + 2S & - & - & - & - & 0.33 $\pm$ 0.01 & 67.42 $\pm$ 0.68 \\
 & HMC + 2S & - & - & - & - & 0.32 $\pm$ 0.01 & 68.05 $\pm$ 0.4 \\
 \cline{2-8}
 & NLM + 2S + EB & - & - & - & - & 0.32 $\pm$ 0.01 & 68.14 $\pm$ 0.36 \\
 & BBB + 2S + EB & - & - & - & - & 0.31 $\pm$ 0.01 & 68.42 $\pm$ 0.43 \\
 & HMC + 2S + EB & - & - & - & - & 0.32 $\pm$ 0.01 & 68.16 $\pm$ 0.37 \\
 \hline
 \multirow{9}{*}{\rotatebox[origin=c]{90}{CPL}} & FCNN & -1.03 $\pm$ 0.3 & -2.49 $\pm$ 2.55 & - & - & 0.34 $\pm$ 0.12 & 65.76 $\pm$ 7.61 \\
 \cline{2-8}
 & BBB & -1.0 $\pm$ 0.25 & -2.53 $\pm$ 3.06 & - & - & 0.32 $\pm$ 0.13 & 66.08 $\pm$ 6.47 \\
 & HMC & -1.05 $\pm$ 0.3 & -2.46 $\pm$ 2.88 & - & - & 0.32 $\pm$ 0.14 & 65.8 $\pm$ 7.06 \\
 \cline{2-8}
 & NLM + 2S & -1.06 $\pm$ 0.32 & -2.91 $\pm$ 2.76 & - & - & 0.3 $\pm$ 0.16 & 64.33 $\pm$ 7.58 \\
 & BBB + 2S & -1.01 $\pm$ 0.3 & -2.82 $\pm$ 2.48 & - & - & 0.31 $\pm$ 0.16 & 65.01 $\pm$ 6.63 \\
 & HMC + 2S & -1.07 $\pm$ 0.33 & -3.02 $\pm$ 2.45 & - & - & 0.34 $\pm$ 0.15 & 66.74 $\pm$ 7.27 \\
 \cline{2-8}
 & NLM + 2S + EB & -0.91 $\pm$ 0.29 & -2.61 $\pm$ 2.63 & - & - & 0.44 $\pm$ 0.14 & 63.53 $\pm$ 6.46 \\
 & BBB + 2S + EB & -1.06 $\pm$ 0.25 & -2.24 $\pm$ 2.85 & - & - & 0.32 $\pm$ 0.11 & 66.41 $\pm$ 6.7 \\
 & HMC + 2S + EB & -1.12 $\pm$ 0.28 & -2.57 $\pm$ 3.18 & - & - & 0.32 $\pm$ 0.14 & 67.65 $\pm$ 6.62 \\
 \hline
 \multirow{4}{*}{\rotatebox[origin=c]{90}{Quint.}} & FCNN & - & - & 1.24 $\pm$ 1.0 & - & 0.27 $\pm$ 0.07 & 67.36 $\pm$ 7.14 \\
 \cline{2-8}
 & BBB  & - & - & 0.88 $\pm$ 1.16 & - & 0.29 $\pm$ 0.07 & 66.33 $\pm$ 5.28 \\
 & HMC  & - & - & 0.92 $\pm$ 1.15 & - & 0.3 $\pm$ 0.06 & 67.96 $\pm$ 4.43 \\
 \cline{2-8}
 & NLM + 2S  & - & - & 1.36 $\pm$ 0.83 & - & 0.32 $\pm$ 0.1 & 64.5 $\pm$ 4.95 \\
 & BBB + 2S  & - & - & 1.13 $\pm$ 1.19 & - & 0.3 $\pm$ 0.05 & 68.05 $\pm$ 5.46 \\
 & HMC + 2S  & - & - & 0.8 $\pm$ 1.14 & - & 0.3 $\pm$ 0.06 & 65.7 $\pm$ 5.62 \\
\hline
 \multirow{4}{*}{\rotatebox[origin=c]{90}{HS}} & FCNN & - & - & - & 2.53 $\pm$ 1.87 & 0.28 $\pm$ 0.07 & 69.03 $\pm$ 7.94 \\
 \cline{2-8}
 & BBB  & - & - & - & 1.82 $\pm$ 1.95 & 0.3 $\pm$ 0.06 & 66.75 $\pm$ 6.75 \\
 & HMC  & - & - & - & 2.06 $\pm$ 1.96 & 0.27 $\pm$ 0.08 & 66.73 $\pm$ 6.5 \\
 \cline{2-8}
 & NLM + 2S  & - & - & - & 2.34 $\pm$ 1.84 & 0.28 $\pm$ 0.12 & 57.44 $\pm$ 8.81 \\
 & BBB + 2S  & - & - & - & 1.98 $\pm$ 1.87 & 0.28 $\pm$ 0.07 & 72.78 $\pm$ 10.01 \\
 & HMC + 2S  & - & - & - & 2.29 $\pm$ 1.67 & 0.28 $\pm$ 0.09 & 67.82 $\pm$ 10.38 \\
 \hline
 \hline
\end{tabular}
        \end{adjustbox}
    }
\end{table}

\begin{table}[]
    \centering
    \caption{Smallest Number Of Sigmas Within Which The Results From \citep{dagostino_measurements_2020,motta_taxonomy_2021,akrami_landscape_2019} Fall.}
    %\begin{adjustbox}{max width=\linewidth}
        \begin{tabular}{c c c c c c c c}
\hline
\hline
Equation & Method & $w_0$ & $w_1$ & $\lambda$ & $b$ & $\Omega_{m,0}$ & $H_0$ \\ 
\hline
 \multirow{9}{*}{\rotatebox[origin=c]{90}{$\Lambda$CDM}} & FCNN & - & - & - & - & 1$\sigma$ & 1$\sigma$ \\
 \cline{2-8}
 & BBB & - & - & - & - & 19$\sigma$ & 15$\sigma$ \\
 & HMC & - & - & - & - & 5$\sigma$ & 3$\sigma$ \\
 \cline{2-8}
 & NLM & - & - & - & - & 4$\sigma$ & 4$\sigma$ \\
 & BBB & - & - & - & - & 3$\sigma$ & 3$\sigma$ \\
 & HMC & - & - & - & - & 4$\sigma$ & 3$\sigma$ \\
 \cline{2-8}
 & NLM + EB & - & - & - & - & 4$\sigma$ & 3$\sigma$ \\
 & BBB + EB & - & - & - & - & 3$\sigma$ & 2$\sigma$ \\
 & HMC + EB & - & - & - & - & 4$\sigma$ & 3$\sigma$ \\
 \hline
 \multirow{9}{*}{\rotatebox[origin=c]{90}{CPL}} & FCNN & 2$\sigma$ & 1$\sigma$ & - & - & 1$\sigma$ & 1$\sigma$ \\
 \cline{2-8}
 & BBB & 3$\sigma$ & 1$\sigma$ & - & - & 1$\sigma$ & 2$\sigma$ \\
 & HMC & 2$\sigma$ & 1$\sigma$ & - & - & 1$\sigma$ & 2$\sigma$ \\
 \cline{2-8}
 & NLM & 2$\sigma$ & 1$\sigma$ & - & - & 1$\sigma$ & 2$\sigma$ \\
 & BBB & 2$\sigma$ & 2$\sigma$ & - & - & 1$\sigma$ & 2$\sigma$ \\
 & HMC & 2$\sigma$ & 2$\sigma$ & - & - & 1$\sigma$ & 1$\sigma$ \\
 \cline{2-8}
 & NLM + EB & 3$\sigma$ & 1$\sigma$ & - & - & 2$\sigma$ & 2$\sigma$ \\
 & BBB + EB & 2$\sigma$ & 1$\sigma$ & - & - & 1$\sigma$ & 1$\sigma$ \\
 & HMC + EB & 2$\sigma$ & 1$\sigma$ & - & - & 1$\sigma$ & 1$\sigma$ \\
 \hline
 \multirow{6}{*}{\rotatebox[origin=c]{90}{Quint.}} & FCNN & - & - & 2$\sigma$ & - & 1$\sigma$ & - \\
 \cline{2-8}
 & BBB  & - & - & 1$\sigma$ & - & 1$\sigma$ & - \\
 & HMC  & - & - & 1$\sigma$ & - & 1$\sigma$ & - \\
 \cline{2-8}
 & NLM  & - & - & 2$\sigma$ & - & 1$\sigma$ & - \\
 & BBB  & - & - & 1$\sigma$ & - & 1$\sigma$ & - \\
 & HMC  & - & - & 1$\sigma$ & - & 1$\sigma$ & - \\
\hline
 \multirow{6}{*}{\rotatebox[origin=c]{90}{HS}} & FCNN & - & - & - & 2$\sigma$ & 1$\sigma$ & 1$\sigma$ \\
 \cline{2-8}
 & BBB  & - & - & - & 1$\sigma$ & 1$\sigma$ & 1$\sigma$ \\
 & HMC  & - & - & - & 1$\sigma$ & 1$\sigma$ & 1$\sigma$ \\
 \cline{2-8}
 & NLM  & - & - & - & 2$\sigma$ & 1$\sigma$ & 2$\sigma$ \\
 & BBB  & - & - & - & 1$\sigma$ & 1$\sigma$ & 1$\sigma$ \\
 & HMC  & - & - & - & 2$\sigma$ & 1$\sigma$ & 1$\sigma$ \\
 \hline
 \hline
\end{tabular}
    %\end{adjustbox}
    \label{tab:sigmas}
\end{table}

\begin{figure}[H]
    \includegraphics[width=\linewidth]{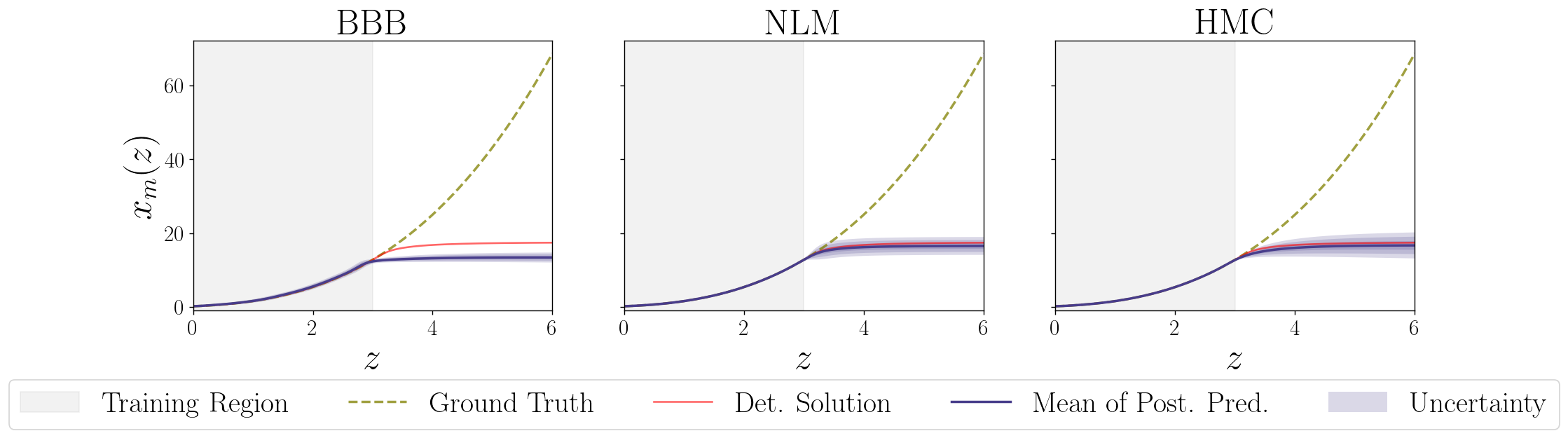}
    \caption{$\Lambda$CDM Bayesian Solutions. The Analytic Solution Is Presented In Dotted Lines.}
    \label{fig:lcdm_forward}
\end{figure}
\begin{figure}[H]
    \includegraphics[width=\linewidth]{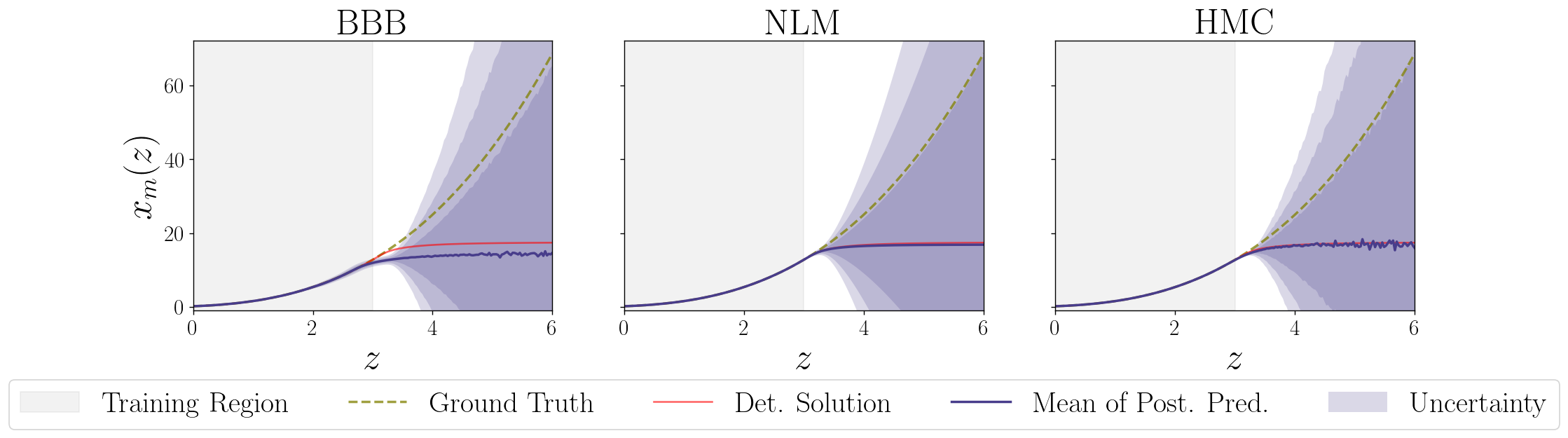}
    \caption{$\Lambda$CDM Bayesian Solutions With Error Bounds. The Analytic Solution Is Presented In Dotted Lines.}
    \label{fig:lcdm_forward_eb}
\end{figure}

\begin{figure}[H]
    \includegraphics[width=\linewidth]{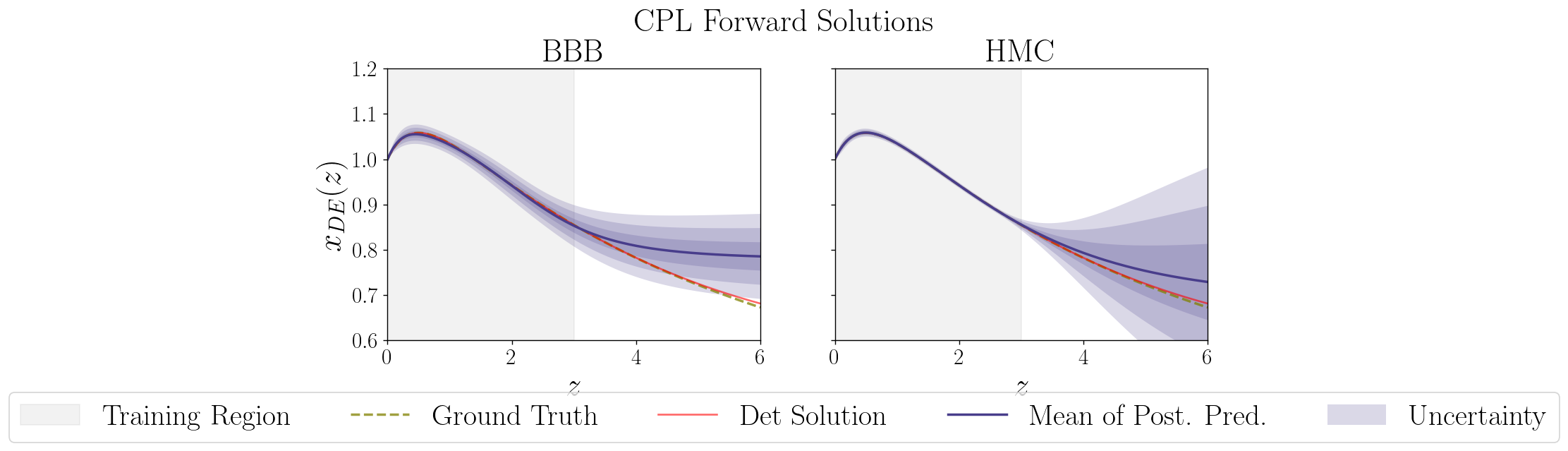}
    \caption{CPL Bayesian Solutions. The Analytic Solution Is Presented In Dotted Lines.}
    \label{fig:cpl_forward}
\end{figure}
\begin{figure}[H]
    \includegraphics[width=\linewidth]{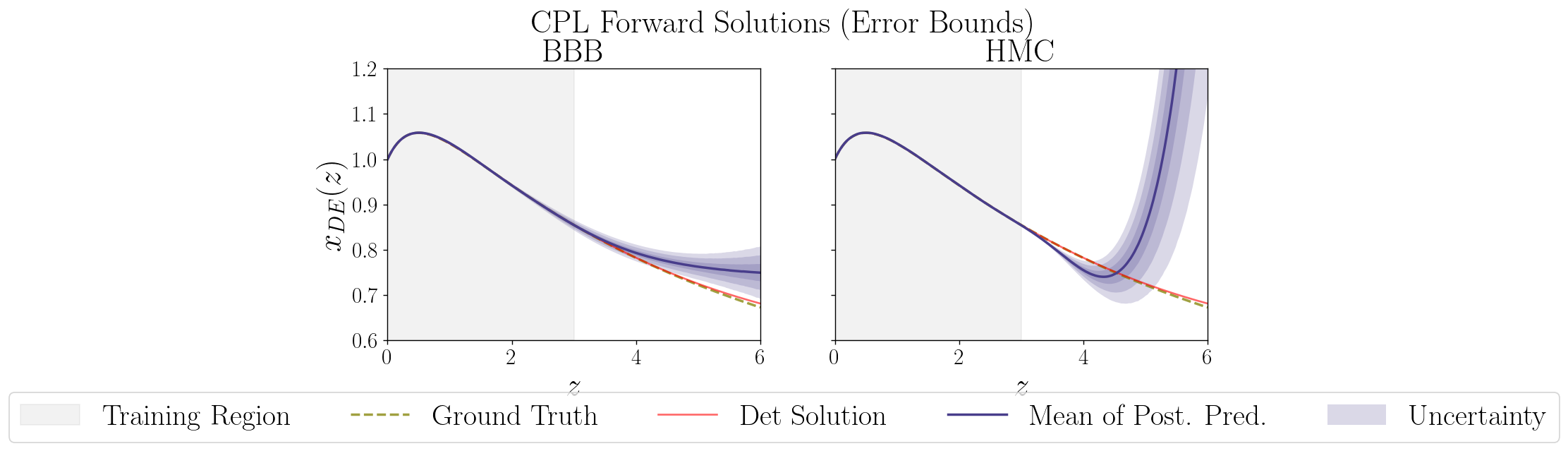}
    \caption{CPL Bayesian Solutions With Error Bounds. The Analytic Solution Is Presented In Dotted Lines.}
    \label{fig:cpl_forward_eb}
\end{figure}

\begin{figure}[H]
    \includegraphics[width=\linewidth]{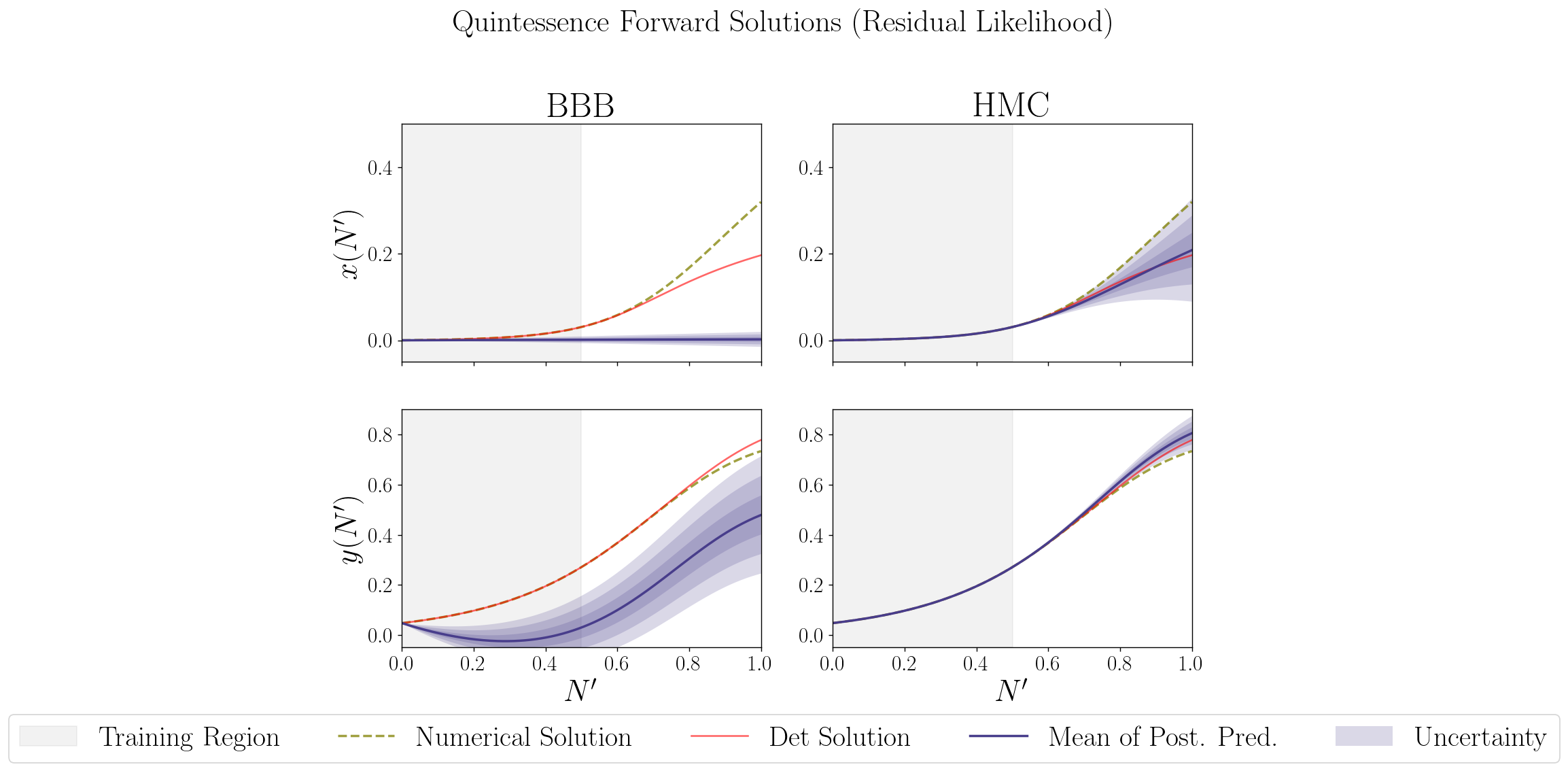}
    \caption{Quintessence Bayesian Solutions Residual Likelihood. The Numerical Solution Is Presented In Dotted Lines.}
    \label{fig:quint_forward_res}
\end{figure}

\begin{figure}[H]
    \includegraphics[width=\linewidth]{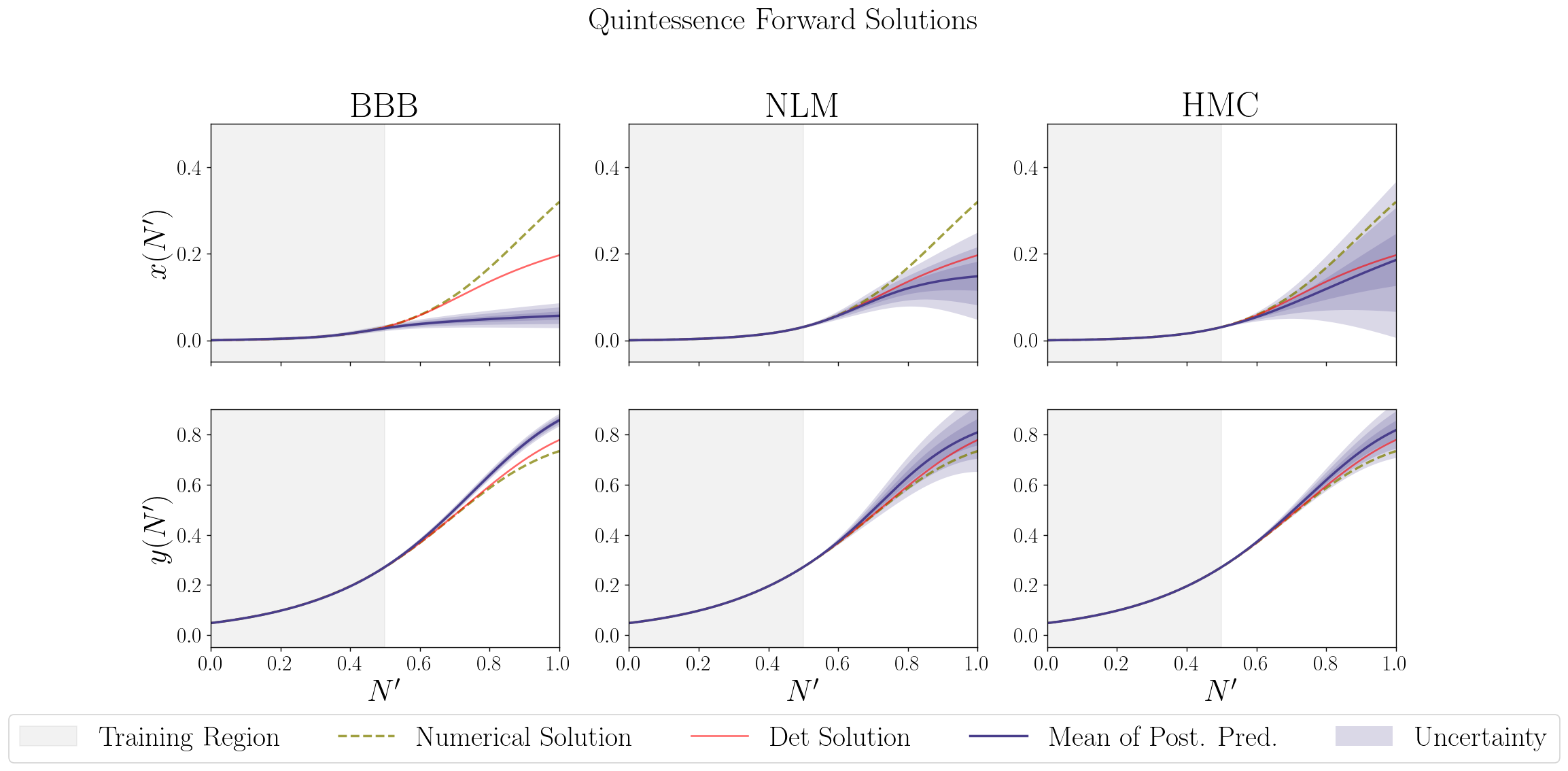}
    \caption{Quintessence Bayesian Solutions. The Numerical Solution Is Presented In Dotted Lines.}
    \label{fig:quint_forward}
\end{figure}

\begin{figure}[H]
    \includegraphics[width=\linewidth]{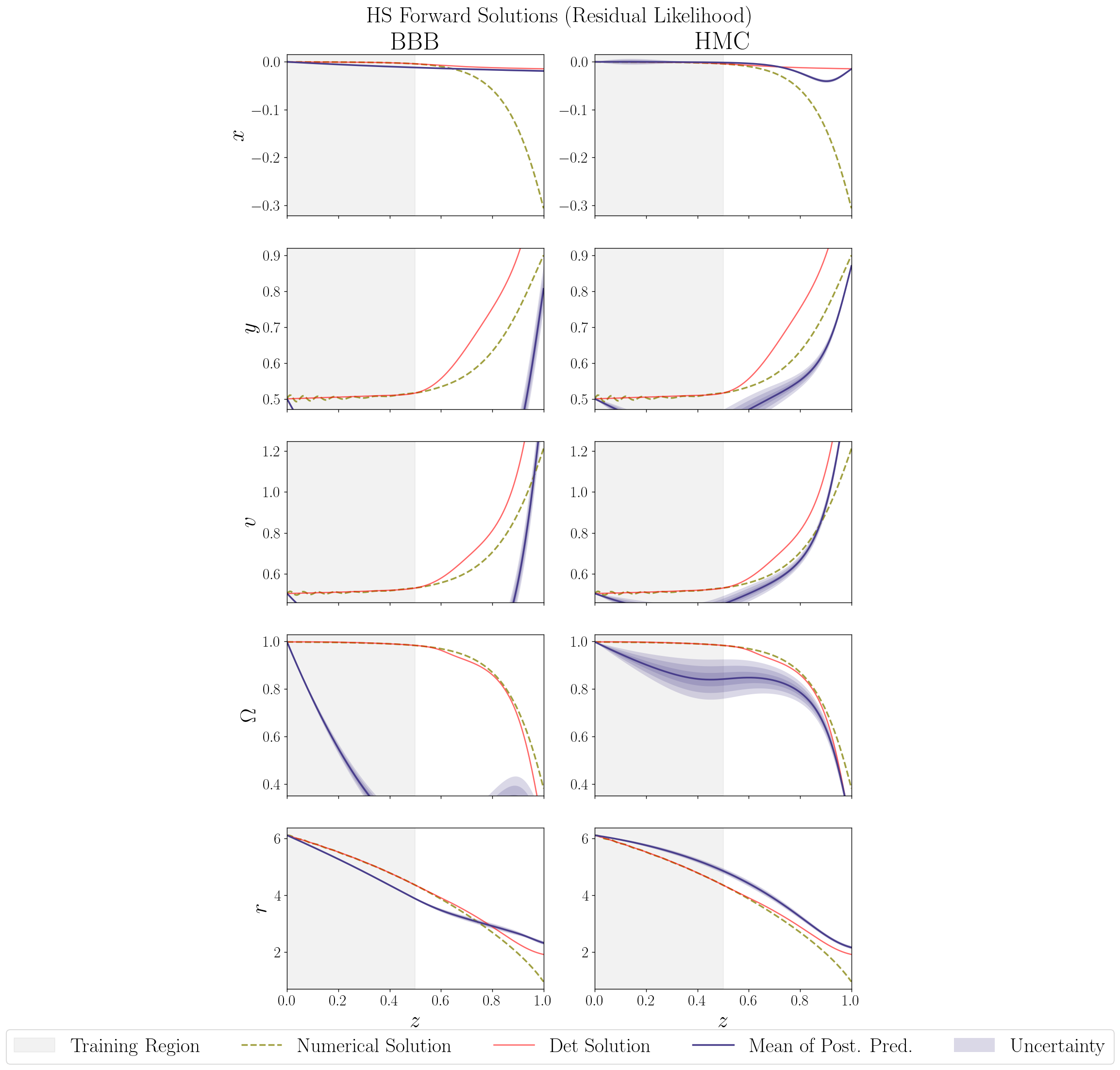}
    \caption{HS Bayesian solutions Residual Likelihood. The Numerical Solution Is Presented In Dotted Lines.}
    \label{fig:hs_forward}
\end{figure}

\begin{figure}[H]
    \includegraphics[width=\linewidth]{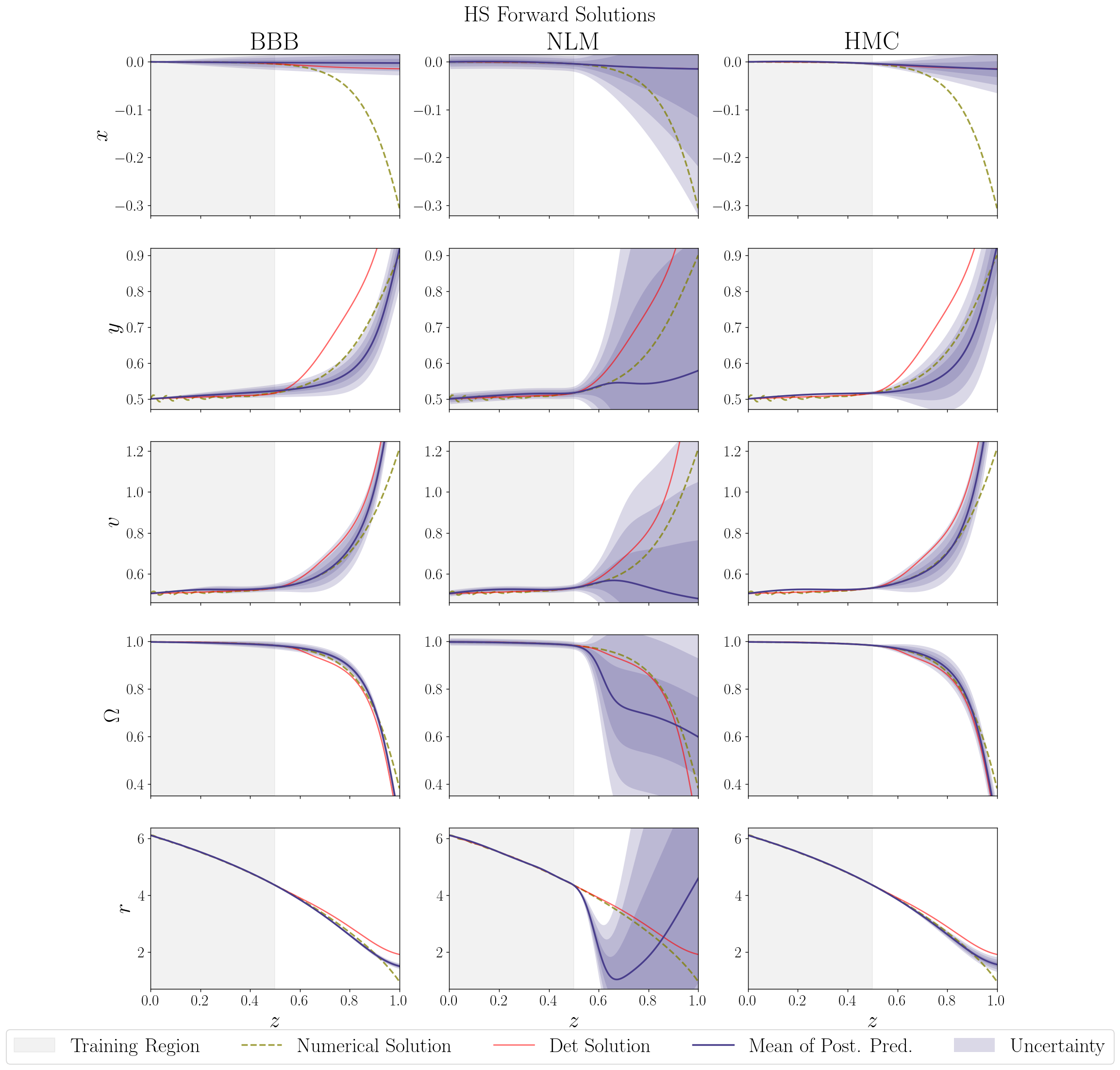}
    \caption{HS Bayesian solutions. The Numerical Solution Is Presented In Dotted Lines.}
    \label{fig:hs_forward}
\end{figure}

\begin{figure}[H]
    \includegraphics[width=\linewidth]{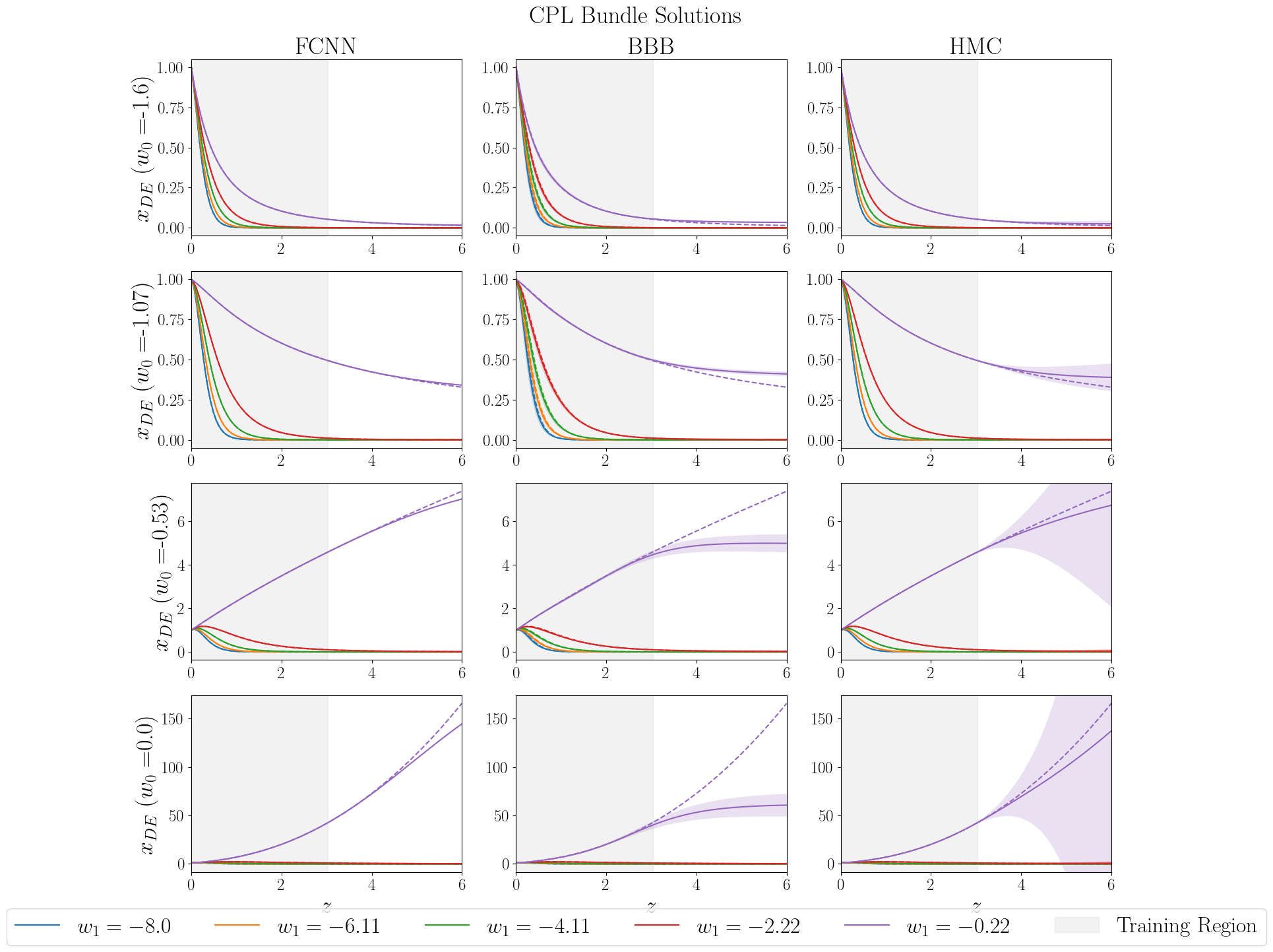}
    \caption{Examples Of CPL Bayesian Solutions Obtained Using The Bundle Solver. Analytic Solutions Are Presented In Dotted Lines.}
    \label{fig:cpl_examples}
\end{figure}

\begin{figure}[H]
    \includegraphics[width=\linewidth]{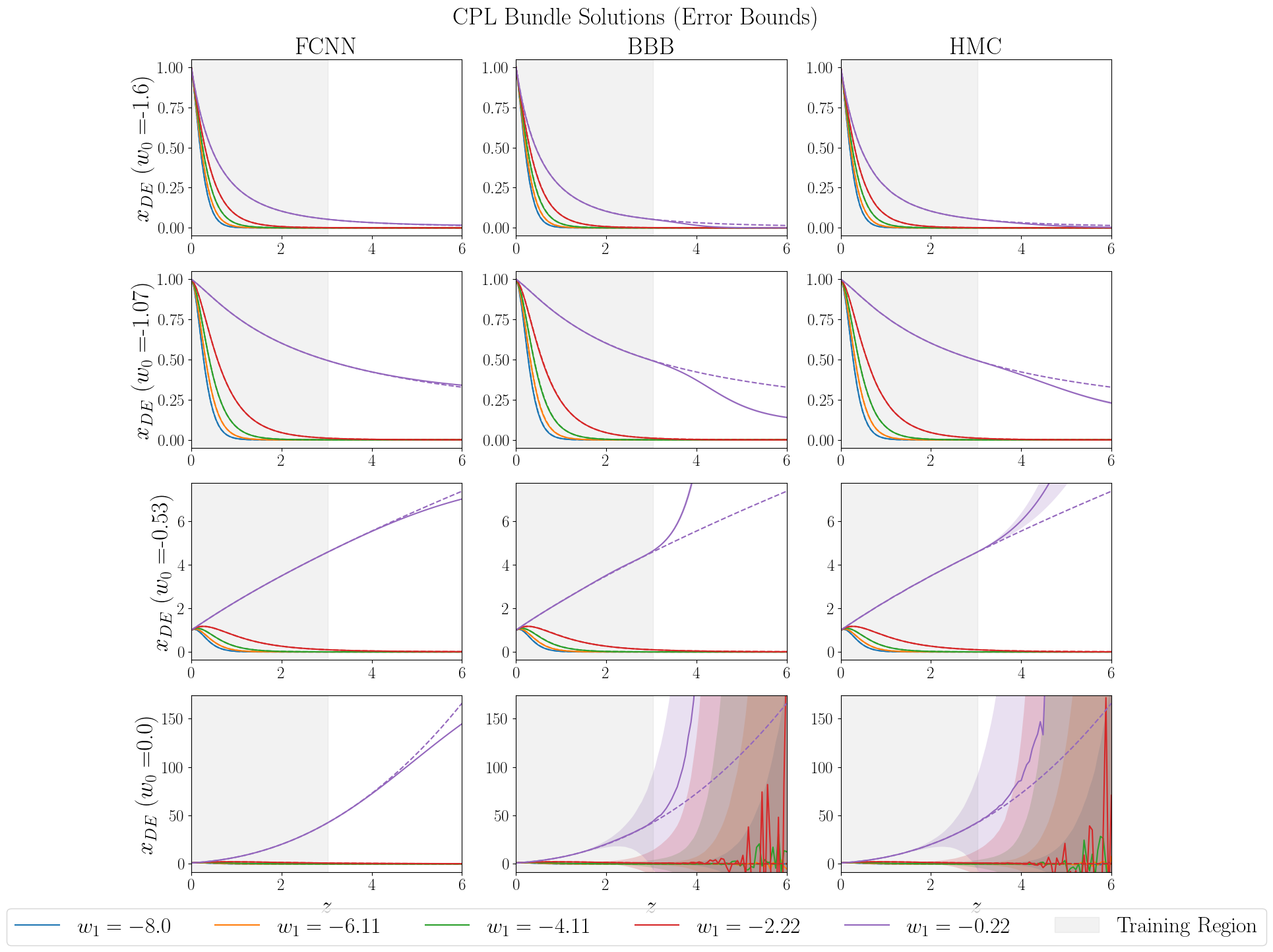}
    \caption{Examples Of CPL Bayesian Solutions Obtained Using The Bundle Solver With Error Bounds. Analytic Solutions Are Presented In Dotted Lines.}
    \label{fig:cpl_examples_eb}
\end{figure}

\begin{figure}[H]
    \includegraphics[width=\linewidth]{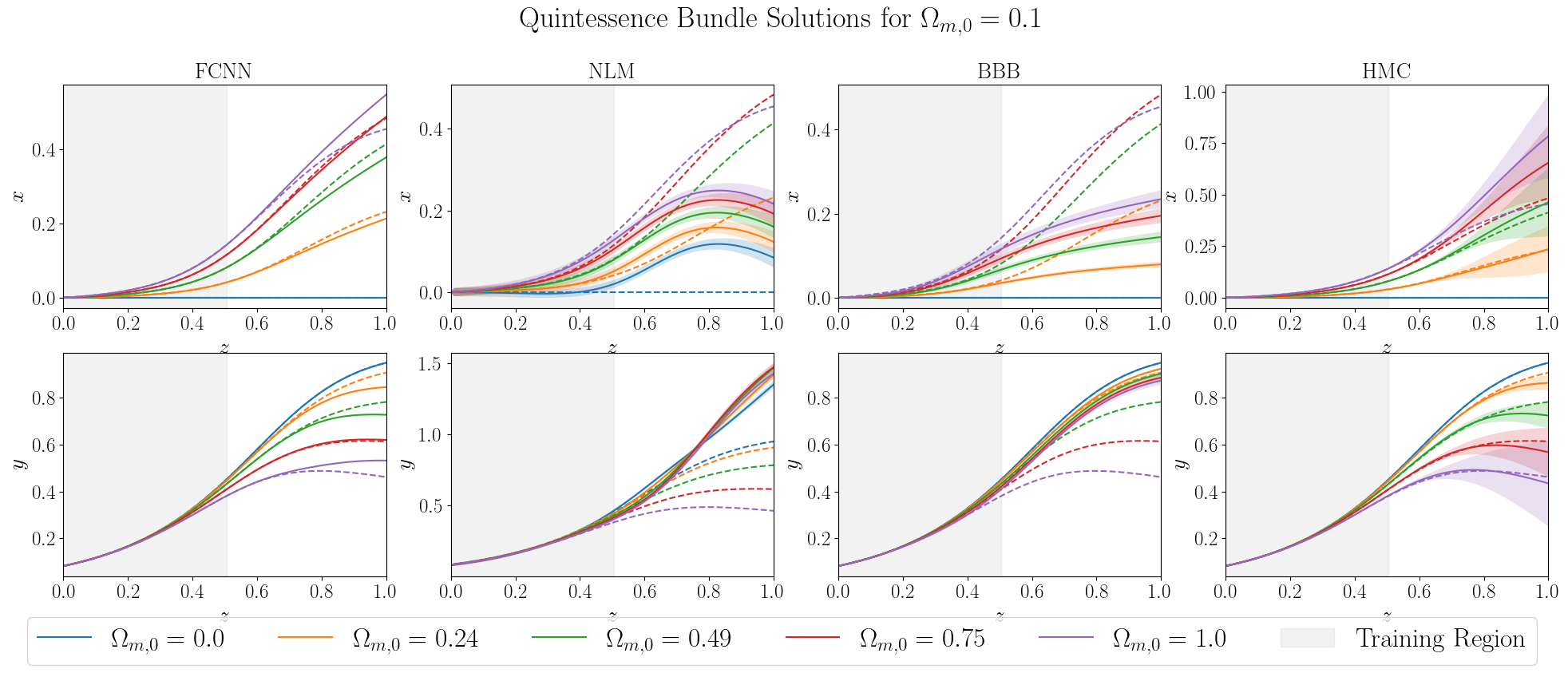}
    \caption{Examples of Quintessence Bayesian Solutions Obtained Using The Bundle Solver For The Parameter Value $\Omega_{m,0}=0.1$. Numerical Solutions Are Presented In Dotted Lines.}
    \label{fig:quint_examples_0}
\end{figure}
\begin{figure}[H]
    \includegraphics[width=\linewidth]{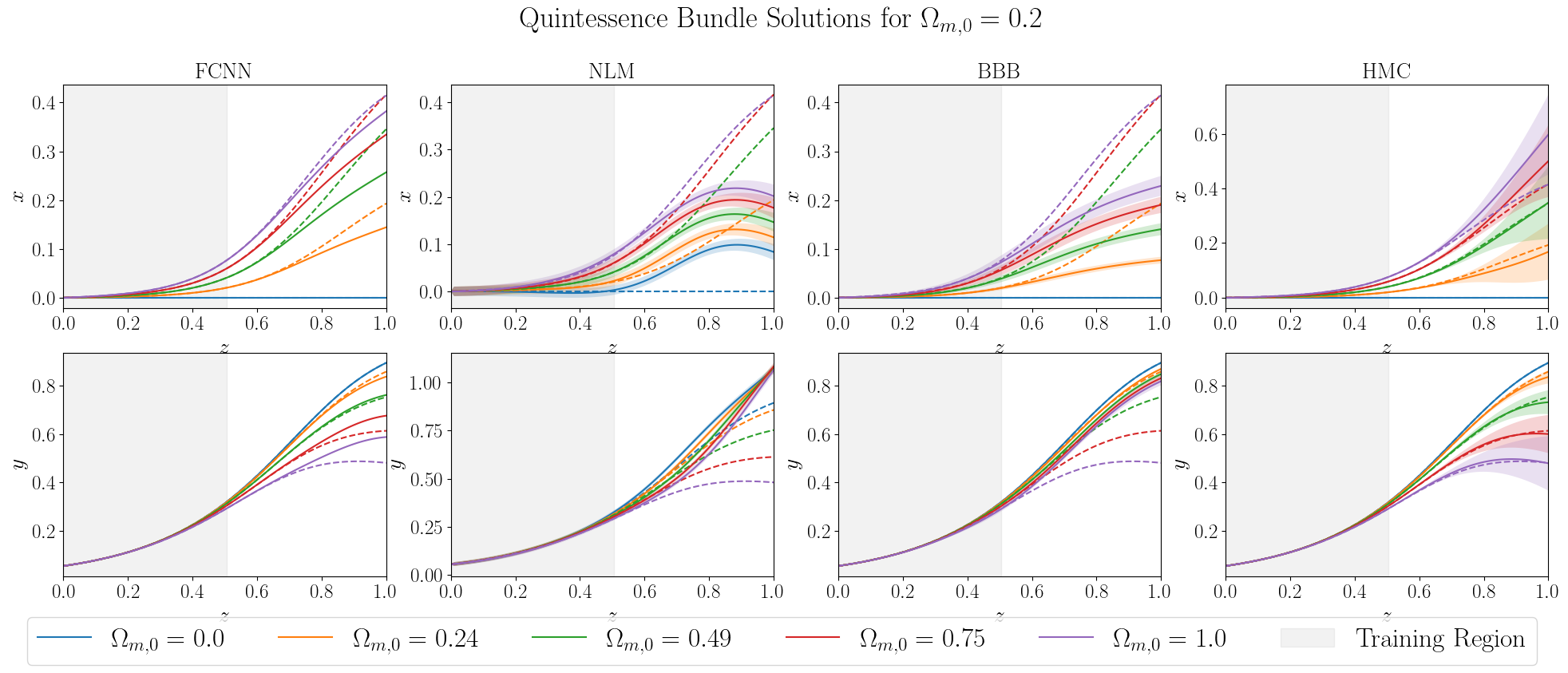}
    \caption{Examples of Quintessence Bayesian Solutions Obtained Using The Bundle Solver For The Parameter Value $\Omega_{m,0}=0.2$. Numerical Solutions Are Presented In Dotted Lines.}
    \label{fig:quint_examples_1}
\end{figure}
\begin{figure}[H]
    \includegraphics[width=\linewidth]{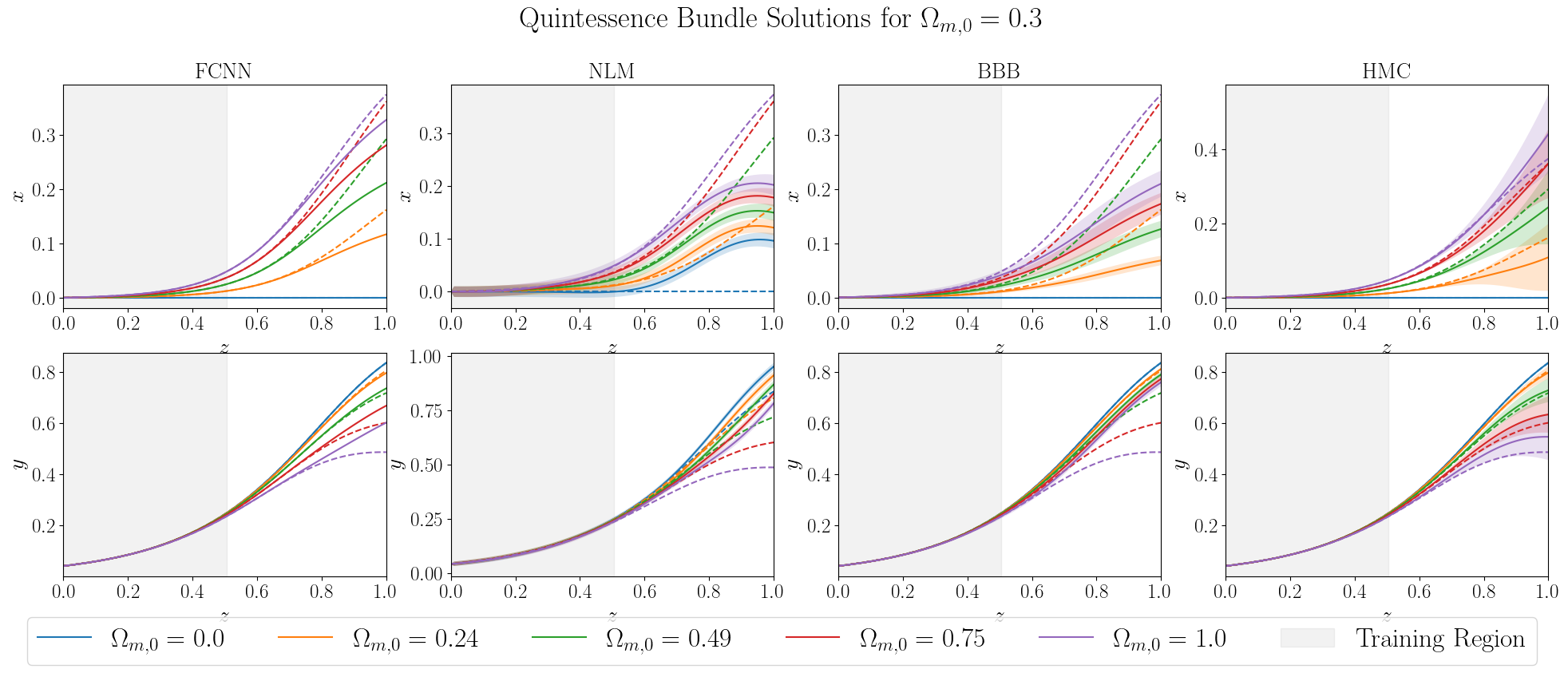}
    \caption{Examples of Quintessence Bayesian Solutions Obtained Using The Bundle Solver For The Parameter Value $\Omega_{m,0}=0.3$. Numerical Solutions Are Presented In Dotted Lines.}
    \label{fig:quint_examples_2}
\end{figure}
\begin{figure}[H]
    \includegraphics[width=\linewidth]{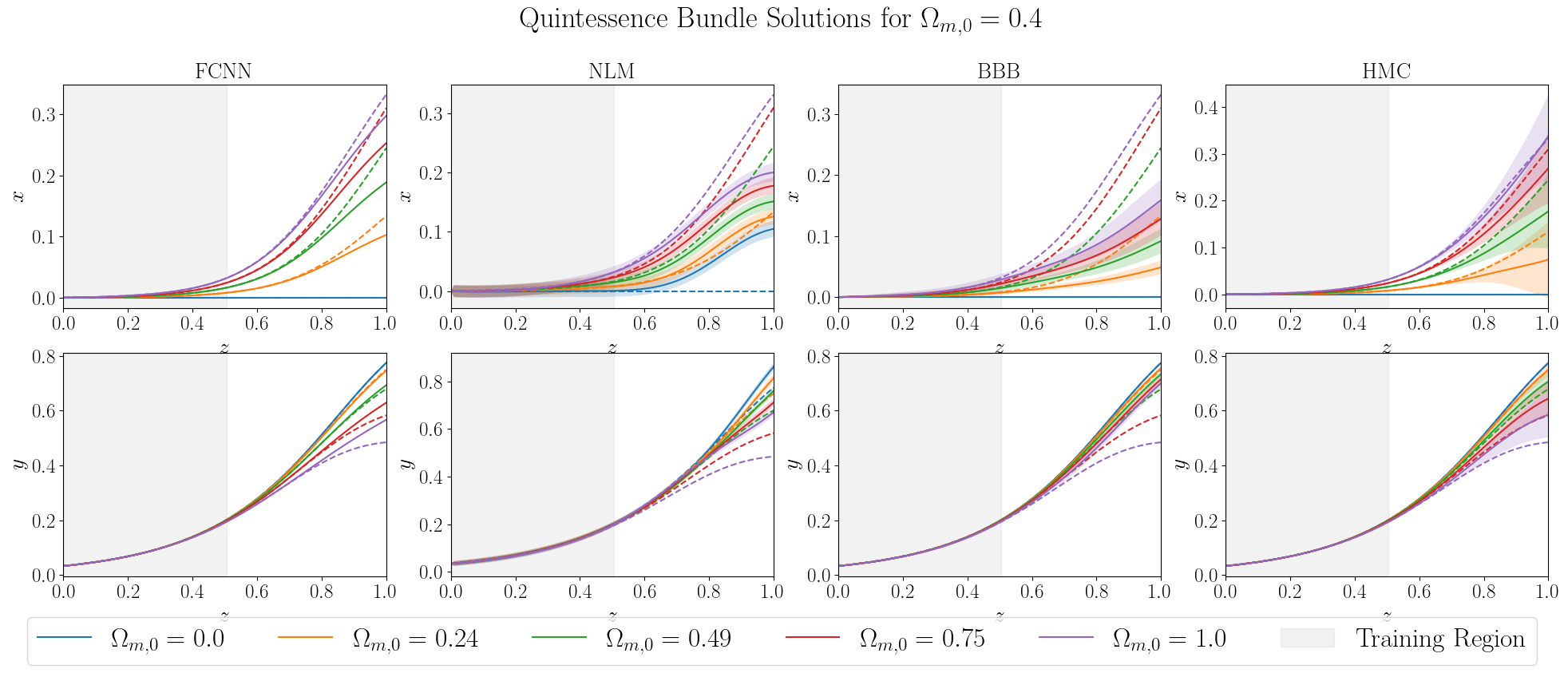}
    \caption{Examples of Quintessence Bayesian Solutions Obtained Using The Bundle Solver For The Parameter Value $\Omega_{m,0}=0.4$. Numerical Solutions Are Presented In Dotted Lines.}
    \label{fig:quint_examples_3}
\end{figure}

\begin{figure}[H]
    \includegraphics[width=\linewidth]{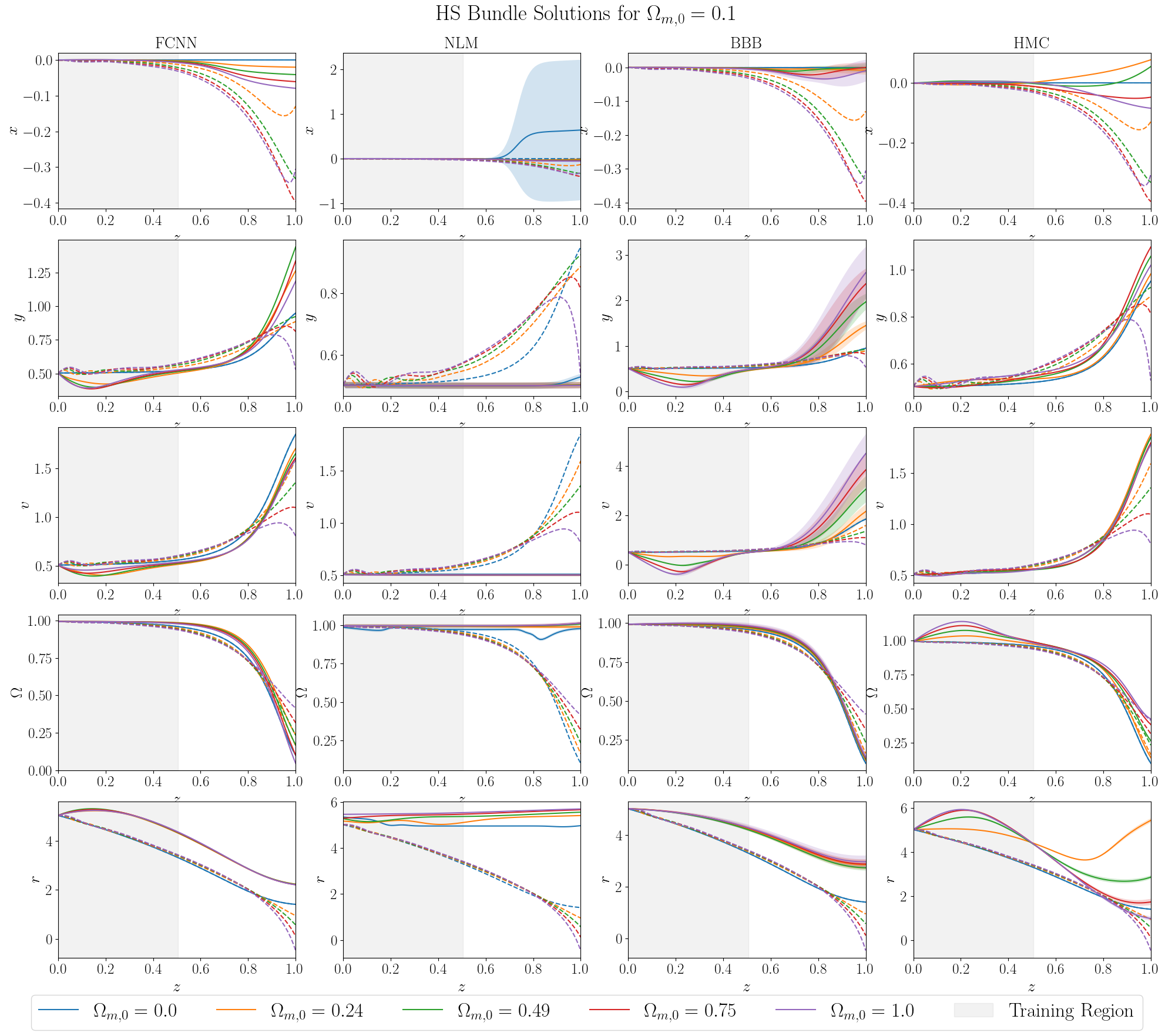}
    \caption{Examples Of HS Bayesian Solutions Obtained Using The Bundle Solver For The Parameter Value $\Omega_{m,0}=0.1$. Numerical Solutions Are Presented In Dotted Lines.}
    \label{fig:hs_examples_0}
\end{figure}
\begin{figure}[H]
    \includegraphics[width=\linewidth]{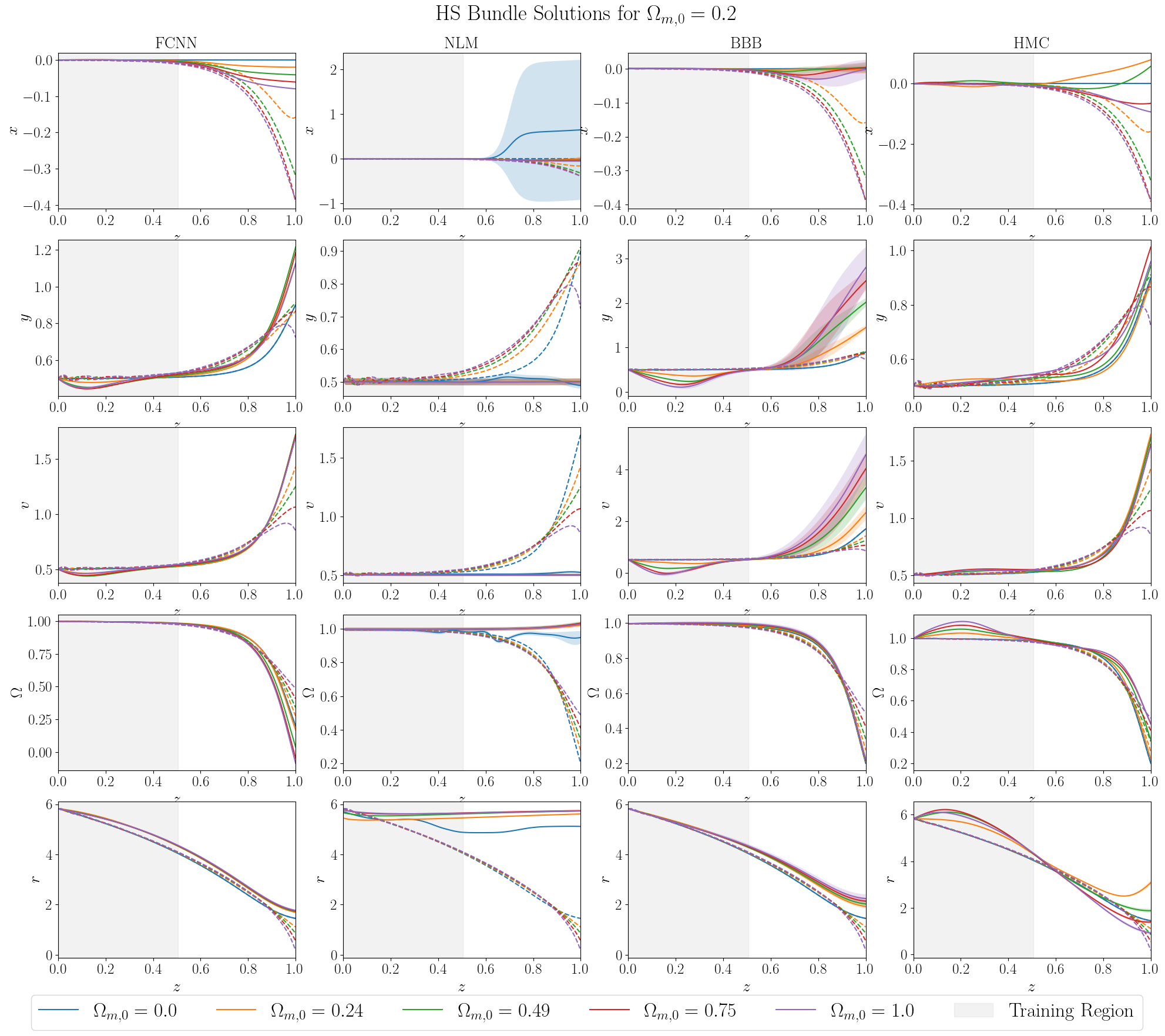}
    \caption{Examples Of HS Bayesian Solutions Obtained Using The Bundle Solver For The Parameter Value $\Omega_{m,0}=0.2$. Numerical Solutions Are Presented In Dotted Lines.}
    \label{fig:hs_examples_1}
\end{figure}
\begin{figure}[H]
    \includegraphics[width=\linewidth]{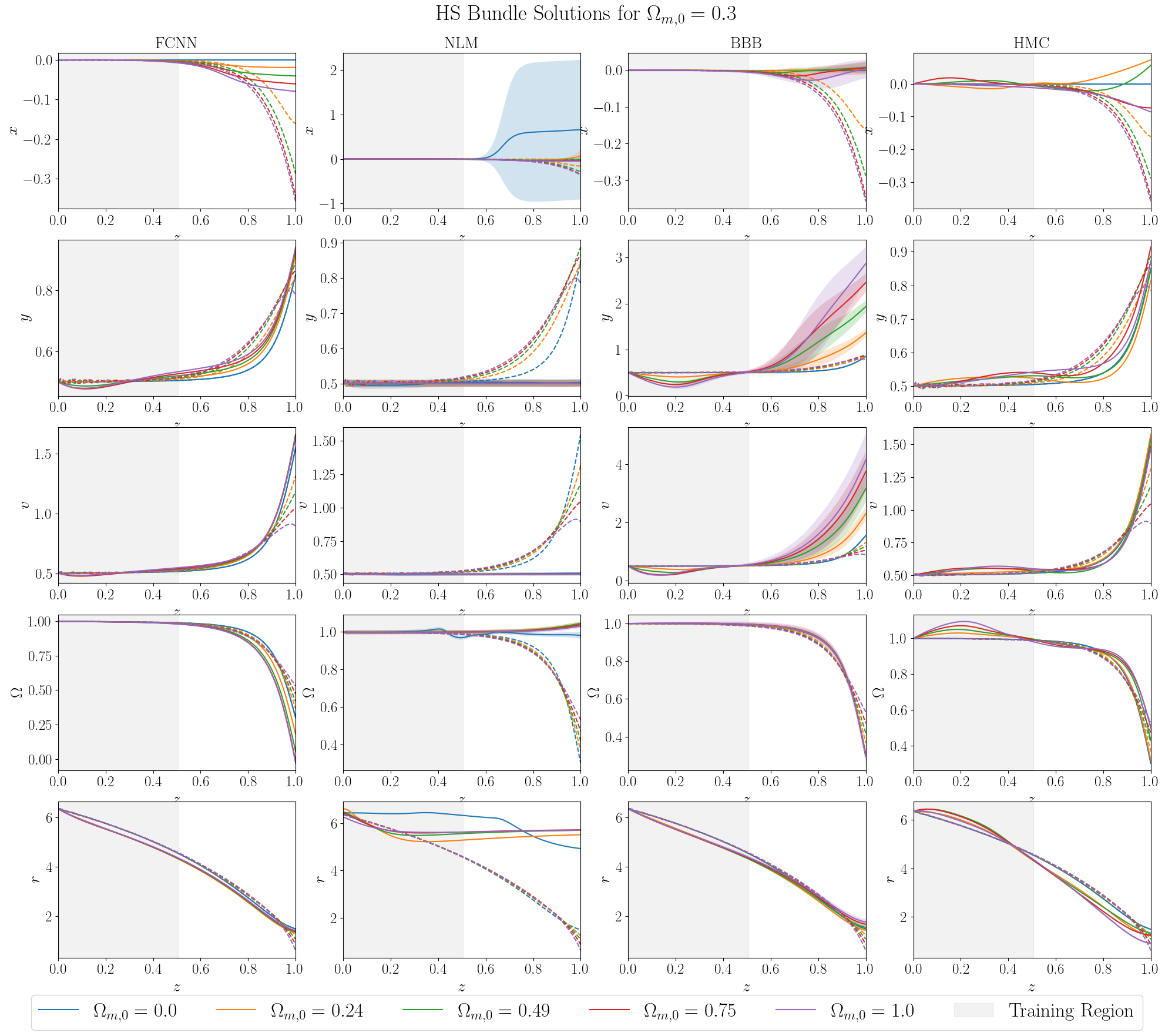}
    \caption{Examples Of HS Bayesian Solutions Obtained Using The Bundle Solver For The Parameter Value $\Omega_{m,0}=0.3$. Numerical Solutions Are Presented In Dotted Lines.}
    \label{fig:hs_examples_2}
\end{figure}
\begin{figure}[H]
    \includegraphics[width=\linewidth]{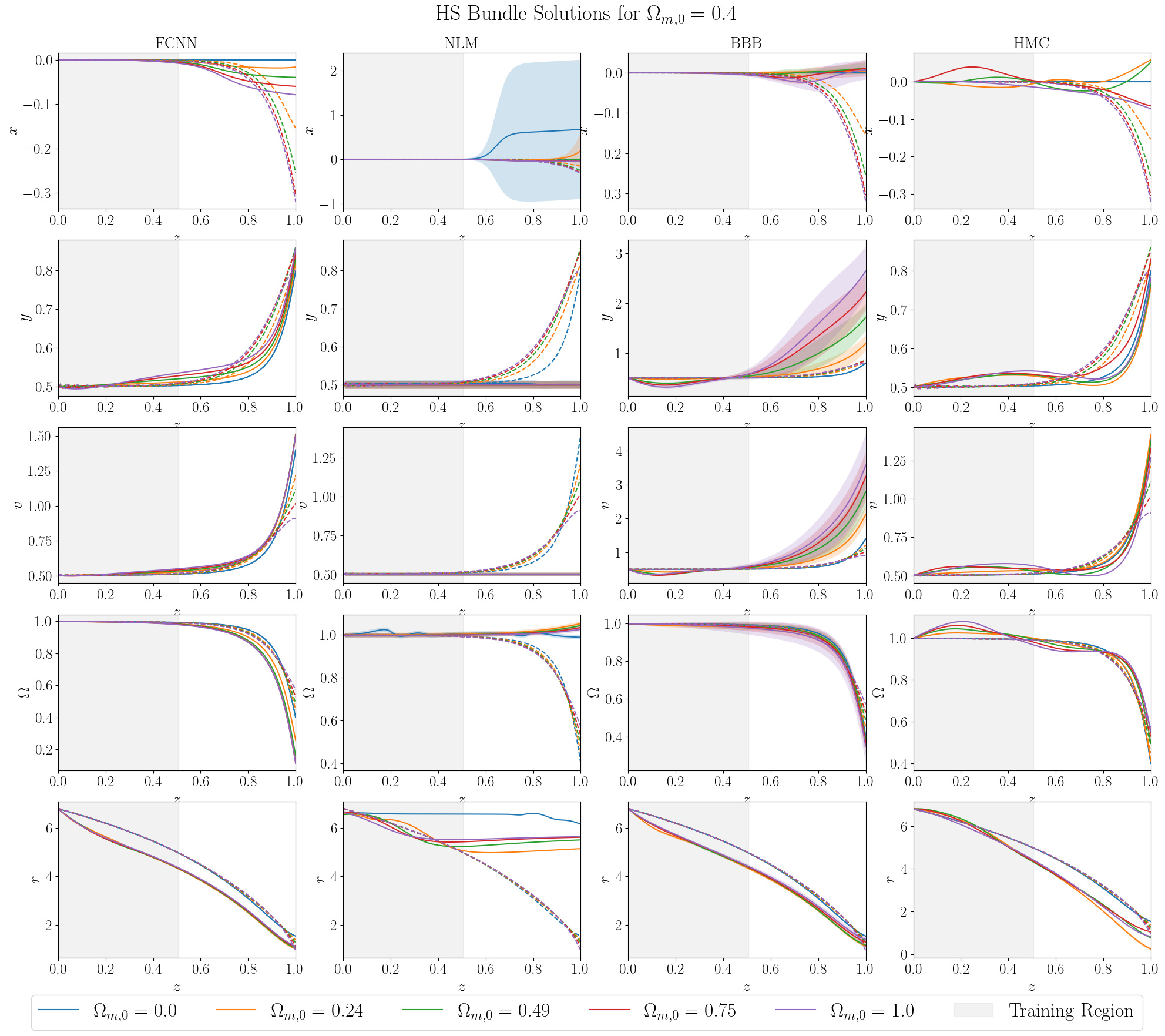}
    \caption{Examples Of HS Bayesian Solutions Obtained Using The Bundle Solver For The Parameter Value $\Omega_{m,0}=0.4$. Numerical Solutions Are Presented In Dotted Lines.}
    \label{fig:hs_examples_3}
\end{figure}

\renewcommand{\arraystretch}{1.2}

\begin{table}[]
    \centering
    \caption{Error Quantiles Of The NN Solutions Relative To The Analytical Solution For $\Lambda$CDM And CPL, And The Numerical Solution For Quintessence And HS.}
    \begin{adjustbox}{max width=\columnwidth}
    \begin{tabular}{c c | c c c c c c c c c c | c c c c c c c c c c}
    \hline
    \hline
    Equation & Method &  \multicolumn{10}{c}{Forward} & \multicolumn{10}{c}{Bundle} \\ \hline
     && Q10 & Q20 & Q30 & Q40 & Q50 & Q60 & Q70 & Q80 & Q90 & Q100 & Q10 & Q20 & Q30 & Q40 & Q50 & Q60 & Q70 & Q80 & Q90 & Q100\\
     \hline
     \multirow{7}{*}{\rotatebox[origin=c]{90}{$\Lambda$CDM}} & FCNN & 0.0 & 0.0 & 0.0 & 0.0 & 0.0 & 0.162 & 0.394 & 0.558 & 0.669 & 0.746 & 0.0 & 0.0 & 0.0 & 0.001 & 0.001 & 0.017 & 0.248 & 0.48 & 0.647 & 0.819 \\
     \cline{2-22}
     & NLM & 0.0 & 0.0 & 0.001 & 0.001 & 0.007 & 0.184 & 0.418 & 0.577 & 0.685 & 0.759 & 0.001 & 0.002 & 0.003 & 0.006 & 0.018 & 0.058 & 0.32 & 0.554 & 0.701 & 0.843 \\
     & BBB & 0.0 & 0.0 & 0.0 & 0.001 & 0.004 & 0.162 & 0.397 & 0.56 & 0.671 & 0.748 & 0.0 & 0.0 & 0.001 & 0.001 & 0.002 & 0.017 & 0.236 & 0.47 & 0.639 & 0.814 \\
     & HMC & 0.001 & 0.001 & 0.002 & 0.004 & 0.019 & 0.228 & 0.443 & 0.593 & 0.696 & 0.767 & 0.001 & 0.001 & 0.002 & 0.003 & 0.007 & 0.051 & 0.212 & 0.415 & 0.597 & 0.799 \\
     \cline{2-22}
     & NLM + EB & 0.0 & 0.0 & 0.0 & 0.0 & 0.0 & 0.17 & 0.406 & 0.569 & 0.679 & 0.754 & 0.0 & 0.0 & 0.001 & 0.001 & 0.002 & 0.017 & 0.184 & 0.381 & 0.537 & 0.717 \\
     & BBB + EB & 0.0 & 0.0 & 0.0 & 0.0 & 0.0 & 0.156 & 0.391 & 0.559 & 0.666 & 0.75 & 0.001 & 0.002 & 0.004 & 0.005 & 0.007 & 0.025 & 0.254 & 0.49 & 0.657 & 0.835 \\
     & HMC + EB & 0.0 & 0.0 & 0.0 & 0.0 & 0.004 & 0.202 & 0.432 & 0.58 & 0.69 & 0.765 & 0.0 & 0.0 & 0.001 & 0.001 & 0.002 & 0.033 & 0.152 & 0.353 & 0.541 & 0.78 \\
     \hline
     \multirow{5}{*}{\rotatebox[origin=c]{90}{CPL}} & FCNN & 0.0 & 0.0 & 0.0 & 0.0 & 0.0 & 0.0 & 0.001 & 0.003 & 0.006 & 0.013 & 0.0 & 0.0 & 0.0 & 0.001 & 0.001 & 0.017 & 0.248 & 0.48 & 0.647 & 0.819\\
     \cline{2-22}
     & BBB & 0.0 & 0.0 & 0.0 & 0.0 & 0.0 & 0.0 & 0.001 & 0.003 & 0.006 & 0.011 & 0.0 & 0.001 & 0.002 & 0.006 & 0.073 & 0.302 & 0.684 & 2.283 & 98.738 & 9.53e+08 \\
     & HMC & 0.0 & 0.0 & 0.0 & 0.0 & 0.001 & 0.008 & 0.022 & 0.041 & 0.066 & 0.095 & 0.001 & 0.002 & 0.004 & 0.007 & 0.033 & 0.13 & 0.363 & 1.711 & 100.003 & 8.31e+08 \\
     \cline{2-22}
     & BBB + EB & 0.0 & 0.0 & 0.0 & 0.0 & 0.0 & 0.0 & 0.0 & 0.001 & 0.004 & 0.009 & 0.002 & 0.004 & 0.007 & 0.014 & 0.034 & 0.239 & 0.743 & 2.333 & 21.77 & 9.01e+08 \\
     & HMC + EB & 0.0 & 0.0 & 0.0 & 0.0 & 0.0 & 0.01 & 0.038 & 0.078 & 0.124 & 0.173 & 0.002 & 0.004 & 0.006 & 0.01 & 0.021 & 0.305 & 0.999 & 28.567 & 6.67e+05 & 3.88e+26 \\
     \hline
     \multirow{4}{*}{\rotatebox[origin=c]{90}{Quint.}} & FCNN & 0.0 & 0.0 & 0.0 & 0.0 & 0.001 & 0.006 & 0.039 & 0.103 & 0.171 & 0.223 & 0.001 & 0.001 & 0.002 & 0.004 & 0.007 & 0.012 & 0.021 & 0.039 & 0.077 & 0.26 \\
     \cline{2-22}
     & NLM & 0.003 & 0.006 & 0.009 & 0.016 & 0.034 & 0.058 & 0.094 & 0.16 & 0.237 & 0.288 & 0.014 & 0.027 & 0.042 & 0.065 & 0.096 & 0.145 & 0.211 & 0.325 & 0.708 & 164.163 \\
     & BBB & 0.001 & 0.002 & 0.003 & 0.005 & 0.012 & 0.027 & 0.067 & 0.129 & 0.197 & 0.918 & 0.001 & 0.003 & 0.005 & 0.008 & 0.013 & 0.02 & 0.034 & 0.057 & 0.104 & 24.699 \\
     & HMC & 0.012 & 0.02 & 0.026 & 0.043 & 0.065 & 0.135 & 0.218 & 0.29 & 0.348 & 1.31 & 0.001 & 0.003 & 0.005 & 0.008 & 0.015 & 0.026 & 0.044 & 0.071 & 0.119 & 18.676 \\
     \hline
     \multirow{4}{*}{\rotatebox[origin=c]{90}{HS}} & FCNN & 0.0 & 0.0 & 0.0 & 0.0 & 0.0 & 0.162 & 0.394 & 0.558 & 0.669 & 0.746 & 0.0 & 0.0 & 0.0 & 0.001 & 0.001 & 0.017 & 0.248 & 0.48 & 0.647 & 0.819 \\
     \cline{2-22}
     & NLM & 0.003 & 0.007 & 0.01 & 0.015 & 0.026 & 0.905 & 8.203 & 23.657 & 40.486 & 147.148 & 0.078 & 0.135 & 0.182 & 0.213 & 0.247 & 0.635 & 3.403 & 16.619 & 88.025 & 1.33e+09 \\
     & BBB & 0.002 & 0.003 & 0.003 & 0.005 & 0.008 & 0.022 & 0.031 & 0.042 & 0.119 & 0.831 & 0.039 & 0.089 & 0.143 & 0.23 & 0.429 & 0.715 & 1.434 & 3.463 & 9.211 & 7.52e+05 \\
     & HMC & 0.121 & 0.305 & 0.521 & 0.674 & 0.742 & 1.308 & 2.71 & 7.066 & 16.963 & 59.125 & 0.067 & 0.123 & 0.169 & 0.217 & 0.294 & 0.368 & 0.629 & 2.291 & 8.942 & 9.06e+05 \\
    \end{tabular}
    \end{adjustbox}
    \label{tab:test_quantiles}
\end{table}

\begin{table}[]
    \centering
    \caption{Metrics of the NN Forward Solutions.}
    \begin{adjustbox}{max width=\columnwidth}
    \begin{tabular}{c c | c c c c c c c c c c c c c c}
\hline
\hline
 Equation & Method &  MRE & Mean Residual & Miscal. Area & RMS Cal. & MA Cal. & Sharpness & NLL & CRPS & Check & Interval & Acc. MAE & Acc. RMSE & Acc. MDAE & Acc. MARPD \\
 \hline
 \multirow{9}{*}{\rotatebox[origin=c]{90}{$\Lambda$CDM}} & FCNN & $\mathbf{0.216}$ & $\mathbf{0.0}$ & - & - & - & - & - & - & - & - & - & - & - & - \\
 \cline{2-16}
 & BBB & 0.955 & 0.154 & 0.491 & 0.566 & 0.486 & $\mathbf{0.101}$ & 27843.2 & 20.019 & 10.01 & 208.426 & 20.073 & 28.172 & 12.83 & 187.161 \\
 & HMC & 0.361 & 3.127 & 0.491 & 0.566 & 0.486 & 0.516 & 319.717 & 11.214 & 5.612 & 113.07 & 11.453 & 19.039 & 2.687 & 49.016 \\
 \cline{2-16}
 & NLM + 2S& 0.226 & 6.544 & 0.224 & 0.258 & 0.222 & 0.546 & 247.862 & 9.743 & 4.876 & 98.293 & 9.939 & 17.917 & 0.023 & 32.164 \\
 & BBB + 2S & 0.29 & 3.807 & 0.185 & 0.222 & 0.183 & 0.3 & 1100.191 & 11.292 & 5.648 & 116.159 & 11.41 & 19.662 & 0.459 & 41.833 \\
 & HMC + 2S & 0.23 & 3.842 & 0.203 & 0.235 & 0.201 & 0.634 & 135.058 & 9.739 & 4.875 & 97.661 & 9.977 & 17.878 & 0.075 & 32.568 \\
 \cline{2-16}
 & NLM + 2S + EB & 0.221 & 6.623 & 0.148 & 0.186 & 0.147 & 17.23 & $\mathbf{-0.472}$ & $\mathbf{5.948}$ & $\mathbf{3.004}$ & $\mathbf{25.502}$ & 9.792 & 17.737 & $\mathbf{0.005}$ & $\mathbf{31.357}$ \\
 & BBB + 2S + EB & 0.266 & 3.69 & 0.114 & $\mathbf{0.137}$ & 0.113 & 17.265 & 2.216 & 6.958 & 3.513 & 30.458 & 11.066 & 19.178 & 0.855 & 38.671 \\
 & HMC + 2S + EB & 0.222 & 3.841 & $\mathbf{0.111}$ & 0.138 & $\mathbf{0.11}$ & 17.251 & -0.208 & 5.953 & 3.006 & 25.627 & $\mathbf{9.788}$ & $\mathbf{17.707}$ & 0.055 & 31.387 \\
 \hline

 \multirow{7}{*}{\rotatebox[origin=c]{90}{CPL}} & FCNN & 0.002 & $\mathbf{0.0}$ & - & - & - & - & - & - & - & - & - & - & - & - \\
 \cline{2-16}
 & BBB & 0.008 & 0.014 & 0.359 & 0.394 & 0.355 & 0.023 & -3.196 & 0.006 & 0.003 & 0.036 & 0.006 & 0.011 & 0.001 & 0.766 \\
 & HMC & 0.009 & 0.014 & 0.268 & 0.288 & 0.266 & 0.016 & -4.458 & 0.004 & 0.002 & 0.02 & 0.007 & 0.012 & $\mathbf{0.0}$ & 0.905 \\
 \cline{2-16}
 & BBB + 2S & 0.036 & 0.026 & $\mathbf{0.103}$ & $\mathbf{0.114}$ & $\mathbf{0.102}$ & 0.019 & -2.284 & 0.021 & 0.01 & 0.123 & 0.026 & 0.043 & 0.004 & 3.444 \\
 & HMC + 2S & 0.016 & 0.019 & 0.267 & 0.3 & 0.265 & 0.034 & -3.87 & 0.008 & 0.004 & 0.041 & 0.012 & 0.02 & 0.001 & 1.592 \\
 \cline{2-16}
 & BBB + 2S + EB & 0.021 & 0.015 & 0.115 & 0.137 & 0.114 & $\mathbf{0.007}$ & -2.376 & 0.013 & 0.006 & 0.093 & 0.015 & 0.027 & 0.001 & 1.993 \\
 & HMC + 2S + EB & 0.158 & 0.257 & 0.389 & 0.443 & 0.385 & 0.07 & 5.18 & 0.093 & 0.047 & 0.7 & 0.11 & 0.292 & $\mathbf{0.0}$ & 10.464 \\
 \hline

 \multirow{6}{*}{\rotatebox[origin=c]{90}{Quintessence}} & FCNN & $\mathbf{0.044}$ & $\mathbf{0.054}$ & - & - & - & - & - & - & - & - & - & - & - & - \\
 \cline{2-16}
 & BBB & 0.838 & 0.109 & 0.482 & 0.553 & 0.477 & 0.025 & 163.756 & 0.127 & 0.064 & 1.113 & 0.14 & 0.171 & 0.135 & nan \\
 & HMC & 0.06 & 0.07 & 0.176 & 0.204 & 0.174 & 0.01 & -4.034 & $\mathbf{0.013}$ & $\mathbf{0.006}$ & 0.08 & $\mathbf{0.015}$ & $\mathbf{0.03}$ & $\mathbf{0.0}$ & nan \\
 \cline{2-16}
 & NLM + 2S & 0.073 & 0.257 & 0.19 & 0.23 & 0.188 & 0.015 & -3.317 & 0.017 & 0.008 & 0.109 & 0.021 & 0.041 & $\mathbf{0.0}$ & $\mathbf{9.054}$ \\
 & BBB + 2S & 0.201 & 0.147 & 0.128 & 0.181 & 0.126 & $\mathbf{0.004}$ & 40.068 & 0.036 & 0.018 & 0.352 & 0.038 & 0.068 & 0.002 & nan \\
 & HMC + 2S & 0.072 & 0.084 & $\mathbf{0.124}$ & $\mathbf{0.142}$ & $\mathbf{0.123}$ & 0.016 & $\mathbf{-4.948}$ & 0.014 & 0.007 & $\mathbf{0.068}$ & 0.019 & 0.036 & $\mathbf{0.0}$ & nan \\
 \hline

 \multirow{4}{*}{\rotatebox[origin=c]{90}{HS}} & FCNN & $\mathbf{0.216}$ & $\mathbf{0.0}$ & - & - & - & - & - & - & - & - & - & - & - & - \\
 \cline{2-16}
 & BBB & 1.224 & 0.215 & 0.487 & 0.56 & 0.482 & 0.013 & 2094.883 & 0.324 & 0.162 & 3.284 & 0.33 & 0.378 & 0.347 & nan \\
 & HMC & 0.268 & 0.128 & 0.44 & 0.512 & 0.436 & 0.013 & 723.293 & 0.138 & 0.069 & 1.326 & 0.144 & 0.169 & 0.148 & nan \\
 \cline{2-16}
 & NLM + 2S & 0.353 & 0.324 & $\mathbf{0.152}$ & $\mathbf{0.184}$ & $\mathbf{0.151}$ & 0.409 & $\mathbf{-1.279}$ & 0.149 & 0.075 & 0.862 & 0.197 & 0.341 & 0.012 & $\mathbf{35.335}$ \\
 & BBB + 2S & 0.158 & 0.326 & 0.219 & 0.253 & 0.216 & $\mathbf{0.012}$ & 59.506 & 0.028 & 0.014 & 0.239 & 0.031 & 0.068 & 0.007 & nan \\
 & HMC + 2S & 0.262 & 0.212 & 0.265 & 0.304 & 0.263 & 0.021 & 90.284 & $\mathbf{0.024}$ & $\mathbf{0.012}$ & $\mathbf{0.182}$ & $\mathbf{0.028}$ & $\mathbf{0.066}$ & $\mathbf{0.006}$ & nan \\
\end{tabular}

    \end{adjustbox}
    \label{tab:metrics_forward}
\end{table}

\begin{table}[]
    \centering
    \caption{Metrics of the NN Bundle Solutions.}
    \begin{adjustbox}{max width=\columnwidth}
    \begin{tabular}{c c | c c c c c c c c c c c c c c}
\hline
\hline
 Equation & Method &  MRE & Mean Residual & Miscal. Area & RMS Cal. & MA Cal. & Sharpness & NLL & CRPS & Check & Interval & Acc. MAE & Acc. RMSE & Acc. MDAE & Acc. MARPD \\
 \hline
 \multirow{9}{*}{\rotatebox[origin=c]{90}{$\Lambda$CDM}} & FCNN & 0.178 & $\mathbf{0.0}$ & - & - & - & - & - & - & - & - & - & - & - & - \\
 \cline{2-16}
 & BBB & 0.94 & 0.199 & 0.491 & 0.566 & 0.486 & $\mathbf{0.076}$ & 67151.127 & 24.692 & 12.347 & 257.601 & 24.73 & 36.873 & 14.224 & 181.346 \\
 & HMC & 0.196 & 5.169 & 0.4 & 0.455 & 0.396 & 0.822 & 139.806 & 10.309 & 5.161 & 102.634 & 10.61 & 22.142 & 0.752 & 25.336 \\
 \cline{2-16}
 & NLM + 2S& 0.206 & 8.278 & 0.183 & 0.211 & 0.182 & 0.45 & 604.068 & 12.215 & 6.111 & 124.943 & 12.369 & 25.205 & 0.032 & 29.855 \\
 & BBB + 2S & 0.204 & 5.386 & 0.143 & 0.173 & 0.141 & 0.487 & 935.006 & 11.829 & 5.918 & 121.119 & 11.965 & 24.362 & 0.092 & 28.573 \\
 & HMC + 2S & 0.171 & 5.446 & 0.123 & 0.148 & 0.122 & 1.4 & 39.06 & 10.346 & 5.182 & 100.22 & 10.809 & 22.873 & $\mathbf{0.011}$ & 23.62 \\
 \cline{2-16}
 & NLM + 2S + EB & $\mathbf{0.145}$ & 9.714 & 0.063 & 0.074 & 0.063 & 25.004 & $\mathbf{-0.532}$ & $\mathbf{5.74}$ & $\mathbf{2.898}$ & $\mathbf{26.943}$ & $\mathbf{9.357}$ & $\mathbf{20.003}$ & 0.015 & $\mathbf{19.212}$ \\
 & BBB + 2S + EB & 0.198 & 5.17 & $\mathbf{0.05}$ & $\mathbf{0.058}$ & $\mathbf{0.05}$ & 25.008 & 1.647 & 7.004 & 3.536 & 31.722 & 11.421 & 23.284 & 0.294 & 26.712 \\
 & HMC + 2S + EB & 0.168 & 5.104 & 0.098 & 0.122 & 0.097 & 25.033 & -0.243 & 6.36 & 3.212 & 28.814 & 10.597 & 22.323 & 0.029 & 23.026 \\
 \hline

 \multirow{7}{*}{\rotatebox[origin=c]{90}{CPL}} & FCNN & $\mathbf{0.062}$ & $\mathbf{0.001}$ & - & - & - & - & - & - & - & - & - & - & - & - \\
 \cline{2-16}
 & BBB & 0.683 & 1410.627 & 0.2 & 0.222 & 0.198 & 5777.129 & -2.261 & 3545.896 & 1777.486 & 33236.555 & 3770.007 & 76032.127 & 0.005 & 28.948 \\
 & HMC & 0.362 & 1722.466 & 0.317 & 0.341 & 0.314 & 1.05e+05 & $\mathbf{-6.292}$ & $\mathbf{1324.918}$ & $\mathbf{668.959}$ & $\mathbf{7053.056}$ & $\mathbf{1969.836}$ & $\mathbf{41625.811}$ & $\mathbf{0.0}$ & $\mathbf{13.055}$ \\
 \cline{2-16}
 & BBB + 2S & 2.751 & 818.487 & 0.15 & $\mathbf{0.164}$ & 0.149 & 2121.18 & 8.885 & 5054.062 & 2529.02 & 51204.587 & 5148.735 & 96419.681 & 0.004 & 37.724 \\
 & HMC + 2S & 15.533 & 2213.885 & 0.255 & 0.272 & 0.252 & 2.92e+05 & -5.211 & 2891.45 & 1459.777 & 18215.13 & 2540.452 & 45826.617 & $\mathbf{0.0}$ & 31.464 \\
 \cline{2-16}
 & BBB + 2S + EB & 2.09e+05 & 9.58e+08 & 0.177 & 0.212 & 0.175 & 2.85e+09 & 26693.418 & 2.68e+08 & 1.35e+08 & 2.10e+09 & 3.09e+08 & 1.23e+10 & 0.003 & 61.241 \\
 & HMC + 2S + EB & 2.35e+05 & 7.49e+05 & $\mathbf{0.145}$ & 0.165 & $\mathbf{0.143}$ & 2.89e+06 & 90.949 & 2.74e+05 & 1.38e+05 & 2.06e+06 & 3.20e+05 & 1.03e+07 & 0.001 & 47.367 \\
 \hline

 \multirow{6}{*}{\rotatebox[origin=c]{90}{Quintessence}} & FCNN & 0.024 & 0.027 & - & - & - & - & - & - & - & - & - & - & - & - \\
 \cline{2-16}
 & BBB & 0.128 & 0.079 & 0.148 & 0.193 & 0.147 & $\mathbf{0.004}$ & 7.61e+09 & 0.017 & 0.009 & 0.159 & 0.019 & 0.041 & 0.002 & nan \\
 & HMC & $\mathbf{0.009}$ & $\mathbf{0.019}$ & 0.147 & 0.165 & 0.146 & 0.007 & 7.61e+09 & $\mathbf{0.002}$ & $\mathbf{0.001}$ & $\mathbf{0.009}$ & $\mathbf{0.002}$ & $\mathbf{0.007}$ & $\mathbf{0.0}$ & nan \\
 \cline{2-16}
 & NLM + 2S & 0.39 & 0.33 & 0.136 & 0.158 & 0.135 & 0.006 & $\mathbf{51.905}$ & 0.039 & 0.019 & 0.375 & 0.04 & 0.099 & 0.002 & $\mathbf{24.652}$ \\
 & BBB + 2S & 0.159 & 0.12 & 0.119 & 0.165 & 0.118 & $\mathbf{0.004}$ & 7.61e+09 & 0.027 & 0.014 & 0.265 & 0.029 & 0.064 & 0.002 & nan \\
 & HMC + 2S & 0.041 & 0.037 & $\mathbf{0.048}$ & $\mathbf{0.054}$ & $\mathbf{0.047}$ & 0.014 & 8.50e+09 & 0.005 & 0.003 & 0.026 & 0.007 & 0.02 & $\mathbf{0.0}$ & nan \\
 \hline

 \multirow{4}{*}{\rotatebox[origin=c]{90}{HS}} & FCNN & 0.178 & $\mathbf{0.0}$ & - & - & - & - & - & - & - & - & - & - & - & - \\
 \cline{2-16}
 & BBB & 0.663 & 0.136 & 0.449 & 0.518 & 0.445 & 0.015 & 5.90e+12 & 0.264 & 0.132 & 2.649 & 0.27 & 0.343 & 0.227 & nan \\
 & HMC & 0.408 & 0.152 & 0.455 & 0.52 & 0.45 & 0.041 & 5.90e+12 & $\mathbf{0.126}$ & $\mathbf{0.064}$ & $\mathbf{1.003}$ & 0.145 & $\mathbf{0.193}$ & 0.125 & nan \\
 \cline{2-16}
 & NLM + 2S & 3740.057 & 2.27e+06 & $\mathbf{0.315}$ & $\mathbf{0.363}$ & $\mathbf{0.312}$ & 0.027 & $\mathbf{20707.105}$ & 0.38 & 0.19 & 3.926 & 0.384 & 0.551 & 0.264 & $\mathbf{41.024}$ \\
 & BBB + 2S & 0.417 & 0.361 & 0.396 & 0.454 & 0.392 & 0.049 & 5.90e+12 & 0.2 & 0.101 & 1.824 & 0.216 & 0.349 & 0.113 & nan \\
 & HMC + 2S & 2.739 & 3.456 & 0.486 & 0.561 & 0.482 & $\mathbf{0.003}$ & 5.98e+12 & 0.136 & 0.068 & 1.401 & $\mathbf{0.137}$ & 0.201 & $\mathbf{0.096}$ & nan \\
\end{tabular}

    \end{adjustbox}
    \label{tab:metrics_bundle}
\end{table}

\end{document}